%% file: main.tex
\documentclass[letterpaper]{article} % DO NOT CHANGE THIS
\usepackage{aaai25}  % DO NOT CHANGE THIS
\usepackage{times}  % DO NOT CHANGE THIS
\usepackage{helvet}  % DO NOT CHANGE THIS
\usepackage{courier}  % DO NOT CHANGE THIS
\usepackage[hyphens]{url}  % DO NOT CHANGE THIS
\usepackage{graphicx} % DO NOT CHANGE THIS
\usepackage{natbib}  % DO NOT CHANGE THIS AND DO NOT ADD ANY OPTIONS TO IT
\usepackage{caption} % DO NOT CHANGE THIS AND DO NOT ADD ANY OPTIONS TO IT
\usepackage{algorithm}
\usepackage{algorithmic}

\usepackage{newfloat}
\usepackage{listings}
\DeclareCaptionStyle{ruled}{labelfont=normalfont,labelsep=colon,strut=off} % DO NOT CHANGE THIS
\floatstyle{ruled}
\newfloat{listing}{tb}{lst}{}
\floatname{listing}{Listing}
\title{On Corruption-Robustness in {\em Performative} Reinforcement Learning}
\author{
    Vasilis Pollatos\textsuperscript{\rm 1}\thanks{This work was done as part of a research immersion lab at the Max Planck Institute for Software Systems.},
    Debmalya Mandal\textsuperscript{\rm 2}\thanks{The work was done while at the Max Planck Institute for Software Systems.}, 
    Goran Radanovic\textsuperscript{\rm 3}
}
\affiliations{
    \textsuperscript{\rm 1} Archimedes/Athena RC, Greece, 
    \textsuperscript{\rm 2}University of Warwick, 
    \textsuperscript{\rm 3}Max Planck Institute for Software Systems\\
    \textsuperscript{\rm 1}v.pollatos@athenarc.gr, \textsuperscript{\rm 2}debmalya.mandal@warwick.ac.uk, \textsuperscript{\rm 3}gradanovic@mpi-sws.org
}
\usepackage{booktabs} % for professional tables
\input{./macros}

\begin{document}

\maketitle

%\newtoggle{longversion}
%\settoggle{longversion}{true}
%\settoggle{longversion}{false}

%%%%%%%%%%%%%%%%%%%%%%%%%%%%%%%%%%%%% INTRODUCTION
\input{0_abstract}

%%%%%%%%%%%%%%%%%%%%%%%%%%%%%%%%%%%%% INTRODUCTION
\input{1_introduction}

%%%%%%%%%%%%%%%%%%%%%%%%%%%%%%%%%%%%% RELATED WORK
\input{2_related_work}

%%%%%%%%%%%%%%%%%%%%%%%%%%%%%%%%%%%%% PROBLEM FORMULATION
\input{3_preliminaries}

%%%%%%%%%%%%%%%%%%%%%%%%%%%%%%%%%%%%% MAIN
\input{./4_corruption_robust_prl}

\input{6_experiments}

%%%%%%%%%%%%%%%%%%%%%%%%%%%%%%%%%%%%% CONCLUSION
\input{7_conclusion}

%%%%%%%%%%%%%%%%%%%%%%%%%%%%%%%%%%%%% ACKNWLEDGEMENTS
\input{8_acknowledgement}

%%%%%%%%%%%%%%%%%%%%%%%%%%%%%%%%%%%%% BIBLIOGRAPHY
\bibliography{main}

%%%%%%%%%%%%%%%%%%%%%%%%%%%%%%%%%%%%% CHECKLIST
% \input{./9_reproducibility}

%%%%%%%%%%%%%%%%%%%%%%%%%%%%%%%%%%%%% APPENDIX
% \iftoggle{longversion}{
% \clearpage
% \onecolumn
\appendix 
{\allowdisplaybreaks
\input{./10.0_appendix_main.tex}
}
% }
% {}

\end{document}

%% file: macros.tex
\usepackage{xcolor}

\usepackage{subcaption}
\usepackage{multicol}

\usepackage{dsfont}

\usepackage{amsmath,amssymb,xfrac,amsthm}
\usepackage{mathtools}
\newtheorem{lemma}{Lemma}
\newtheorem{theorem}{Theorem}
\newtheorem*{theorem*}{Theorem}

\newtheorem{corollary}{Corollary}

\newtheorem{definition}{Definition}
\newtheorem{assumption}{Assumption}
\newtheorem{remark}{Remark}

\newcommand{\expctu}[2]{\mathbb{E}_{#1}\left[#2\right]}

\newcommand{\ind}[1]{\mathds{1}\left[#1\right]}

\newcommand{\norm}[1]{\left\lVert#1\right\rVert}

\DeclareMathOperator*{\argmin}{arg\,min}
\DeclareMathOperator*{\argmax}{arg\,max}

\newcommand{\eps}{\epsilon}

%% file: 0_abstract.tex
% !TEX root =  main.tex
%%%%%%%%%%%%%%%%%%%%%%%%%%%%%%%%%%%%%
%%%%%%%%%%%%%%%%%%%%%%%%%%%%%%%%%%%%%
\begin{abstract}
In {\em performative} Reinforcement Learning (RL), an agent faces a policy-dependent environment: the reward and transition functions depend on the agent's policy. Prior work on performative RL has studied the convergence of repeated retraining approaches to a {\em performatively} stable policy. In the finite sample regime, these approaches repeatedly solve for a saddle point of a convex-concave objective, which estimates the Lagrangian of a regularized version of the reinforcement learning problem. In this paper, we aim to extend such repeated retraining approaches, enabling them to operate under corrupted data. More specifically, we consider Huber's $\epsilon$-contamination model, where an $\epsilon$ fraction of data points is corrupted by arbitrary adversarial noise. We propose a repeated retraining approach based on convex-concave optimization under corrupted gradients and a novel problem-specific robust mean estimator for the gradients. We prove that our approach exhibits last-iterate convergence to an approximately stable policy, with the approximation error %proportional to 
linear in $\sqrt{\epsilon}$. We experimentally demonstrate the importance of accounting for corruption in performative RL. %reinforcement learning.

\end{abstract}

%% file: 1_introduction.tex
% !TEX root =  main.tex
%%%%%%%%%%%%%%%%%%%%%%%%%%%%%%%%%%%%%
%%%%%%%%%%%%%%%%%%%%%%%%%%%%%%%%%%%%%
\section{Introduction}\label{sec.introduction}

In {\em performative reinforcement learning}~\cite{mandal2023performative, rankperformative}, the learner operates in a policy-dependent environment, where the reward and transition functions are influenced by the learner's policy.
A canonical example of such a setting is an RL system involving human users, such as RL-based recommender systems or chatbots: the RL policy affects user preferences, which in turn alters user engagement and behavior, thereby modifying the RL environment.

For an offline learner, {\em performativity} represents a specific type of distribution shift. Similar to standard reinforcement learning (RL), where the learner's policy influences the data it encounters due to the sequential decision-making process, performativity introduces an additional layer of complexity. The learner not only affects the observed data but also alters the underlying data generation process through its policy choices. This dual influence complicates the learning process, impacting the learner's ability to generalize effectively.

%For an offline learner, performativity is a form of distribution shift. As in standard RL, the learner's policy affects the data it observes due to the sequential decision making nature of the problem. However, in addition, the learner also modifies the underlying data generation model through the choice of its policy. Intuitively, this makes the learning problem challenging as this property affects generalization. 

%To address this challenge, previous work imposes sensitivity assumptions, often formalized as Lipschitz conditions, which limit the extent to which the learner's policy can influence the reward and transition functions of the underlying environment. These conditions enable convergence guarantees to a performatively stable policy during repeated retraining. This process involves optimizing a regularized RL objective after each deployment round to update the current policy. In the finite sample regime, the retraining step seeks a saddle point of a convex-concave objective, which approximates the Lagrangian of the regularized reinforcement learning problem.

To alleviate this challenge prior work places sensitivity assumptions, formalized as Lipschitz conditions, which limit the extent the learner's policy can influence the reward and transition functions of the underlying environment. These conditions enable convergence guarantees to a performatively stable policy during repeated retraining. This process involves optimizing a regularized RL objective after each deployment round to update the current policy. In the finite data regime, the retraining step solves for a saddle point of a convex-concave objective, which approximates the Lagrangian of the regularized reinforcement learning problem. 

%To alleviate this challenge prior work places sensitivity assumptions, formalized as Lipschitz conditions, which limit the extent the learner's policy can influence the reward and transition functions of the underlying environment. Under these conditions, one can establish converge guarantees to a {\em performatively} stable policy for repeated retraining, which optimizes a regularized RL objective after each deployment round to update the current policy. In the finite sample regime, the retraining step solves for a saddle point of a convex-concave objective, which estimates the Lagrangian of a regularized version of the reinforcement learning problem. 

In this paper, we additionally consider a third type of distribution shift, specifically induced by an adversary capable of corrupting training data. This scenario is particularly relevant in practical settings where performativity plays a role. For example, recommendation systems are vulnerable to Sybil attacks, while chatbots can be susceptible to poisoning poisoning attacks (e.g., see the Tay.ai incident \citep{neff2016talking}). 
%
%Another example comes from autonomous vehicles (AV), where poisoning attacks in the training of AVs has been shown to be a security threat~\cite{jiang2020poisoning}.
%
Therefore, it is crucial to explore strategies for achieving corruption-robustness in performative environments.

%Another example comes from autonomous vehicles (AV). Even if we ignore the multiagent aspect of these learning algorithms, a deployed AV might change how the pedestrians, and other cars behave, and the resulting environment might be quite different from what the designers of the AV had in mind [arXiv:1701.07790]. Corruption robustness in this setting is motivated, as poissoning attacks in the training of AVs has been shown to be a threat [arXiv:2011.04065, Jiang, Wenbo, et al. "Poisoning and evasion attacks against deep learning algorithms in autonomous vehicles. IEEE transactions on vehicular technology].

%In this paper, we additionally consider a third type of distribution shift, caused by an adversary that can corrupt training data. This scenario is highly relevant for practical scenarios where perfomativity is present. For example, recommendation systems can experience Sybil attacks, whereas chatbots can be susceptible to poisoning attacks (e.g., see the Tay.ai incident \citep{neff2016talking}). Hence, it is important to understand how to enable corruption-robustness in performative environments. 
 
Our goal is to extend algorithmic and convergence results from prior work on perfomrative RL by allowing an $\epsilon$ fraction of data points in training data to be corrupted by adversarial noise---commonly referred to as the {\em Huber $\epsilon$-contamination model}. We recognize two main challenges: % in doing so:
\begin{itemize}
    \item First, this involves solving a convex-concave optimization problem where access to a given objective function is corrupted. For instance, 
    gradient-based methods, which are commonly used for convex-concave optimization, rely on first-order gradient oracles to access the objective. 
    However, existing guarantees for these methods typically assume that the outputs of gradient oracles are clean and not adversarially corrupted.  
    \item Second, generic robust estimators, such as those used for estimating gradients, do not account for the problem structure, often making them impractical in challenging high-dimensional settings.
\end{itemize}

\textbf{Contributions:} Our work aims to resolve these challenges within the performative RL framework. Our contributions are as follow:  
\begin{itemize}
    \item We propose a robust version of Optimistic Optimistic Follow the Regularized Leader (OFTRL) for optimizing convex-concave objectives under corruption  and provide theoretical analysis of thereof. This analysis shows that the robust OFTRL achieves optimal convergence rates and information-theoretically optimal terminal error. These results are of an independent interest and complement the concurrent work on gradient-based algorithms for convex-concave optimization under corrupted gradients. See the related work section for more details. 
    \item We extend the existing algorithmic approach to performative reinforcement learning, making it robust against corruption. More specifically, we propose a novel repeated retraining approach based on robust OFTRL and a novel coordinate-wise robust mean estimator for estimating the gradients. The robust mean estimator is particularly suited for performative RL. We theoretically analyze our repeated retraining approach, showing that it exhibits last-iterate convergence to an approximately stable policy, with the approximation error that scales linearly with the square root of the corruption level $\epsilon$. 
    \item Using a simulation-based experimental testbed, we showcase the importance of accounting for corrupted gradients, and the efficacy of our approach. 
\end{itemize}

%% file: 2_related_work.tex
% !TEX root =  main.tex
%%%%%%%%%%%%%%%%%%%%%%%%%%%%%%%%%%%%%
%%%%%%%%%%%%%%%%%%%%%%%%%%%%%%%%%%%%%
\section{Related Work}\label{sec.related_work}

We recognize several lines of works related to this paper: {\em performative prediction and RL}, {\em corruption-robust offline RL}, and {\em convex-concave optimization}, and {\em convex optimization under corrupted gradients}. 

\textbf{Performative Prediction and RL}: Performative prediction models data distribution shifts due to the model deployment~\cite{perdomo2020performative}. Much prior work has studied convergence properties of algorithms to different solution concepts, including performative stability~\cite{perdomo2020performative, mendler2020stochastic} and performative optimality~\cite{izzo2021learn,miller2021outside}. In recent years, other variants of the canonical setting have been proposed, including multi-player variants~\cite{narang2023multiplayer,piliouras2023multi}, variants that consider more nuanced state-dependent distribution shifts~\cite{brown2022performative, li2022state}, or variants that introduce constraints~\cite{yan2024zero} or bilevel optimization~\cite{lu2023bilevel}. 
%
%\textbf{Performative RL}: 
%
The closest to our setting is the work of 
\citep{mandal2023performative}, who introduced the concept of performativity in reinforcement learning. 
The performative RL framework relates to Stackelberg stochastic games~\citep{letchford2012computing,zhong2021can}, in which a principal agent commits to a policy to which follower agents respond—making the principal's effective environment performative. However, the performative RL framework abstracts away from game-theoretic considerations, as it does not model performativity through game-theoretic agents.
\cite{rankperformative} extends the performative RL framework by considering gradual environment shifts, akin to those considered in the performative precision settings of~\cite{brown2022performative, li2022state}.
%
%and also adapted the minimax optimization problem~\citep{ZHH+22} for the offline setting of performative RL. 
%\cite{rankperformative} extends this framework by considering state-dependent distribution shifts akin to those on prior work on performative prediciton in supervised learning. 
%In this work, we introduce corruption-robustness to performative RL.
Our paper contributes to the extensive literature on performative prediction~\cite{hardt2023performative} by introducing corruption-robustness to performative RL.

\textbf{Corruption-Robust Offline RL}: 
From a technical perspective, our results build on the analysis~\cite{mandal2023performative}, who adapted the minimax optimization problem~\citep{ZHH+22} for the offline setting of performative RL. Hence, our work relates to the vast literature on corruption robustness in offline RL~\citep{zhang2022corruption, wu2022copa, ye2024corruption, nika2024corruption, mandal2024corruption}. 
However, these works do not utilize corruption-robust minimax optimization in their frameworks, nor do their underlying framework model performativity effects.  
An additional discussion of how the bounds in some of these works compares to ours is provided in the convergence analysis of our approach. 

%\textbf{Convex-Concave Optimization}: Historically, the earliest method for solving convex-concave optimization was the \emph{fictitious play} algorithm~\citep{Robinson51}, who studied bilinear optimization problems.
%i.e. $\min_{x \in \Delta_m} \max_{y \in \Delta_n} x^\top A y$. Subsequently, this problem has been tackled with the extra-gradient method~\citep{Korpelevich76, Tseng95}. For the general convex-concave optimization problem, several variants of Gradient Ascent Descent-based algorithms have been proposed and analyzed \cite{Nemirovski04, Nesterov07, Tseng08, NO09, MOP19, MOP20}. Several papers have also studied convex-concave optimization under additional structural assumptions. For strongly convex-concave minimax optimization, \citep{Tseng95, NS06} were the first ones to derive a gradient complexity. Their bounds were subsequently improved by ~\citep{GBG+18, MOP20, ADS+19}, eventually matching with the lower bound established by \citep{IAG+19, ZHZ19}. There is also significant research on nonconvex-concave minimax optimization~\cite{LJJ20, LJJ20b}, but their coverage is out of the scope of this paper.

%\textbf{minimax Gradient methods under inexact/corrupted gradients}:
\textbf{Convex-Concave Optimization}: % with Inexact Gradients}:
Our approach relies on convex-concave optimization. There has been a vast literature on this topic, from the \emph{fictitious play} algorithm~\citep{Robinson51} and the extra-gradient method~\citep{Korpelevich76, Tseng95} for solving bilinear optimization problems, to Gradient Ascent Descent-based algorithms for the general convex-concave optimization problem~\cite{Nemirovski04, Nesterov07, Tseng08, NO09, MOP19, MOP20}. Most relevant to our setting are works on convex-concave optimization that assume inexact oracles.
\citep{juditsky2011solving, huang2022new} consider biased stochastic gradient oracles. In \citep{beznosikov2020gradient},  a zeroth-order biased oracle with stochastic and bounded deterministic noise is assumed and \cite{dvinskikh2022gradient} consider a zeroth-order oracle corrupted by adversarial noise. 
%In Table \ref{tab:tabcontr} we compare our results with these works.
 Our work differs from these results as our setting has a bounded gradient corruption (bias) and we provide guarantees for the actual error of our algorithm instead of the expected.
Arguably, the closest related work on convex-concave optimization to ours is the concurrent work of  \cite{zhang2024optimal}. They provide convergence guarantees under adversarial  noise for smooth convex-concave functions, but achieve this through a different algorithm. Hence, our results on convex-concave optimization  complement theirs. Our information-theoretic lower bound on the duality gap has a different flavor than the lower bounds in \cite{zhang2024optimal}; the latter focuses on notions related to algorithmic {\em reproducibility}. Hence, our lower bound and the analysis we used for deriving it is novel. 
 
\textbf{Convex Optimization under Corrupted Gradients}: In convex optimization gradient corruption has been more extensively studied \citep{polyak1987introduction, d2008smooth, devolder2013exactness, devolder2014first}. From this line of research most relevant are the works of \cite{https://doi.org/10.1111/rssb.12364} and \cite{pmlr-v139-wang21r}, who study corruption in gradient descent due to poisoning attacks with applications to machine learning, as well as works of~\cite{ahn2022reproducibility}, that study reproducibility in optimization in the face of inexact gradients.

%% file: 3_preliminaries.tex
% !TEX root =  main.tex
%%%%%%%%%%%%%%%%%%%%%%%%%%%%%%%%%%%%%
%%%%%%%%%%%%%%%%%%%%%%%%%%%%%%%%%%%%%

\section{Preliminaries}\label{sec.preliminaries}

In this section, we provide the necessary background. Our main approach builds on convex-concave optimization under corrupted gradients. 
More specifically, we 
consider smooth convex-concave objectives. We propose a robust version of  Optimistic Follow the Regularized Leader (OFTRL), and provide a convergence guarantee for it. In particular, we show that it exhibits $O(1/T)$ convergence rate for the duality gap. This result and the analysis are of  independent interest, complementing prior work that studied convex-concave analysis under gradient corruption. 
% (see Remark \ref{rm.comparison_to_prior_work}). 

%In this section, we provide the necessary background.and introduce the required notation. We focus on two problem settings of interest: convex-concave optimization and performative reinforcement learning. 

\textbf{Notation.} In the following sections, we use $\|\cdot\|$ to denote the $L_2$ norm $\|\cdot\|_2$, $[N]$ to denote the set $\{1,2,...N\}$ and $\Pi_{\mathcal{X}}(x)$ to denote the projection of a vector $x$ on a set $\mathcal{X}$.

\subsection{Convex-concave optimization}

We consider the following constrained minimax problem 
\begin{equation} \label{minmax}
\min\limits_{x\in \mathcal{X}}\max\limits_{y\in \mathcal{Y}}  f(x,y),
\end{equation} 
where $f$ is a convex-concave function,  and $\mathcal{X}$ and $\mathcal{Y}$ are convex bounded domains. We denote the upper bounds to the radius of domains $\mathcal{X}$ and $\mathcal{Y}$ by $D_X$ and $D_Y$ respectively, i.e., $\max\limits_{x\in\mathcal{X}} \|x\| \leq D_X$ and  $\max\limits_{y\in\mathcal{Y}} \|y\| \leq D_Y$. We assume that gradients of this function are bounded, i.e., $\max\limits_{x\in\mathcal{X},y\in\mathcal{Y}}\|\nabla_x f(x,y)\| \leq G_X$ and $\max\limits_{x\in\mathcal{X},y\in\mathcal{Y}}\|\nabla_y f(x,y)\| \leq G_Y$, for some constants $G_X, G_Y$. These assumptions are satisfied in the objective function that we will consider in the sections on perfomrative RL.

\subsection{Contamination Model}

We are interested in designing a robust optimization method for finding a saddle point of $f$ which only has access to $f$ through noisy (first order) gradient oracles with bounded noise norm. We assume we are given a sampling procedure that generates unbiased gradient samples an $1-\epsilon$ fraction of the time and adversarial samples (potentially with unbounded corruption) otherwise. This is a strong contamination model known as Huber contamination. 
%Prior work has shown that there exist robust estimators (e.g. \cite{diakonikolas2020outlier}) that can filter the samples and guarantee that the gradient estimation has a bounded error w.r.t. the true gradient and this error scales with $\epsilon$ instead of the (potentially unbounded) initial corruption.
We aim to design an optimization method that is robust in the sense that: a) its convergence guarantees have an information theoretically optimal dependence on the noise norm of the inexact gradient oracle, matching our lower bound in \ref{thr:lb} and b) it deploys robust gradient estimators that trim the unbounded corruption down to a bounded error.

\subsection{Smooth Convex-Concave Objectives}

The convex-concave objective that we will study in the sections on performative RL satisfies smoothness. Hence, in this section, we consider a class of convex-concave functions $f$ that are smooth. In particular, we assume that for all $x\in \mathcal{X}$ and $y \in \mathcal{Y}$ functions $\nabla_x f(\cdot,y), \nabla_x f(x,\cdot), \nabla_y f(\cdot,y)$ and $ \nabla_y f(x,\cdot)$ have Lipschitz constants $L_{XX}, L_{XY}, L_{XY}$ and $L_{YY}$ respectively, w.r.t. $\ell_2$ norm. Similarly to the setting of exact gradient oracles, these smoothness conditions enable us to achieve a better convergence rate. 

\textbf{Robust OFTRL.} To find a saddle point of $f$, we propose a robust version of gradient-based Optimistic Follow the Regularized Leader, as defined in Algorithm \ref{alg:aoftrl}. The algorithm follows standard OFTRL steps (e.g., see \citep{orabona2019modern}).
It alternates between updating $x_t$ and $y_t$, in optimizing for each a regularized objective. Regularizer $\psi_X$ (resp. $\psi_Y$) is   $\lambda_X$ (resp. $\lambda_Y$) strongly convex and bounded over $\mathcal X$ (resp. $\mathcal Y$). For instance $\psi_X(x)$ could be equal to $\frac{\lambda_X}{2}\|x\|^2$ and $\psi_Y(y)$ could be equal to $\frac{\lambda_Y}{2}\|y\|^2$. Importantly, the algorithm utilizes robust gradient estimates in line 7. While adversarial samples can potentially have unbounded corruption, prior work has shown that there exist robust estimators for the Huber $\epsilon$-contamination model (e.g. \cite{diakonikolas2020outlier}) that can filter the samples and guarantee that the gradient estimation has a bounded error w.r.t. the true gradient and this error scales with $\epsilon$.
%instead of the (potentially unbounded) initial corruption.
For the analysis in the next subsection, 
we will assume black box access to a (robust) gradient estimation oracle upon query on some point $(x_t, y_t)$.
%we will assume such robust gradient estimators are given. 
%
For the results on performative RL, we provide a problem-specific robust estimator. 

\begin{algorithm}%[H]
\begin{algorithmic}[1]
%\caption{Robust alternating Optimistic FTRL}\label{alg:aoftrl}
\caption{Robust OFTRL}\label{alg:aoftrl}
\STATE Initialize $\alpha,b,c>0$
\STATE$(\lambda_X,\lambda_Y)  \gets 3 (L_{X X}+L_{X Y} \alpha+b,\ L_{Y Y}+L_{X Y} / \alpha+c)$ 
% \KwResult{$y = x^n$}
\STATE$(g_{X,0},\ g_{Y,0})\gets(\textbf{0},\textbf{0})$
\FOR{$t=1,...T$}
  \STATE $x_t\gets\argmin\limits_{x\in\mathcal{X}}\  \psi_X(x)+\langle g_{X,t-1},x\rangle+\sum\limits_{i=1}^{t-1}\langle g_{X,i},x\rangle$ 
  \STATE $y_t\gets\argmin\limits_{y\in\mathcal{Y}}\  \psi_Y(y)+\langle g_{Y,t-1},y\rangle+\sum\limits_{i=1}^{t-1}\langle g_{Y,i},y\rangle$ 
  \STATE Calculate robust estimations $g_{X,t}$ and $g_{Y,t}$ of
   $\nabla_x f(x_t,y_t)$ and $-\nabla_y f(x_t,y_t)$ respectively 
 \ENDFOR
 \STATE $(\bar{x},\ \bar{y}) \gets (\frac{1}{T}\sum\limits_{t=1}^T x_t,\ \ \frac{1}{T}\sum\limits_{t=1}^T y_t)$ 
 \STATE \textbf{Return} $\bar{x}, \bar{y}$ 
 \end{algorithmic}
\end{algorithm}  

\subsection{Analysis of Robust OFTRL}

Next, we analyze the convergence guarantees of Robust OFTRL. 
%\footnote{The algorithm and its analysis are based on results in \citep{orabona2019modern} that assume exact gradients}
Denoting the errors of robust gradient estimation 
%as in Theorem \ref{thr:gda}, i.e., 
as $\zeta_t^X= g_{X,t}-\nabla_x f(x_t,y_t)$ and $\zeta_t^Y= g_{Y,t}-\nabla_y f(x_t,y_t)$, we obtain the following result. 
% We also assume that domain $\mathcal{X}$ has radius $D_X$ and $\mathcal{Y}$ has radius $D_Y$.\\ \\
\begin{theorem} \label{thr:aoftrl}
The output $(\bar{x},\bar{y})$ of Algorithm \ref{alg:aoftrl} satisfies for all $ x\in\mathcal{X}$ and $ y\in\mathcal{Y}$  : %\\
%\goran{Maybe use $\Delta$ for the duality gap and combine terms so that we only use 2 lines, not 3.}
\begin{align*}
    &f\left(\bar{x}, y\right) -f\left(x, \bar{y}\right) \leq 
    \frac{\psi_X (x)+\psi_Y (y)}{T}+ \frac{\| \nabla_x f(x_1,y_1)  \|^2}{\lambda_X \cdot T} \\
    &+\frac{\|\nabla_y f(x_1,y_1)\|^2}{\lambda_Y \cdot T} +\frac{6D_X}{T}\sum_{t=1}^{T} \|\zeta_{t}^X\| + \frac{6D_Y}{T}\sum_{t=1}^{T} \|\zeta_{t}^Y\|.
\end{align*}
% \begin{align*}
%     f\left(\bar{x}, y\right)-f\left(x, \bar{y}\right) \leq 
%     &\frac{\psi_X (x)+\psi_Y (y)}{T}+\\
%     &\frac{\frac{\| \nabla_x f(x_1,y_1)  \|^2}{\lambda_X }+\frac{\|\nabla_y f(x_1,y_1)\|^2}{\lambda_Y }}{T} + \\
%     &\frac{6D_X}{T}\sum_{t=1}^{T} \|\zeta_{t}^X\| + \frac{6D_Y}{T}\sum_{t=1}^{T} \|\zeta_{t}^Y\|
% \end{align*}

\end{theorem}
The proof of this theorem is based on the results from \citep{orabona2019modern}, which consider the exact gradients.
We see that the algorithm converges to an approximate saddle point with the $\frac{1}{T}$ convergence rate and the approximation error, i.e., asymptotic duality gap, dependent on the errors of robust gradient estimation. 
The $\frac{1}{T}$
convergence rate is optimal for smooth convex-concave problems, as shown in prior work~\citep{ouyang2021lower}.
%The dependence of the error on the noise is also optimal, as we show in our next result. %\ref{sec.lower_bound}
%
Next, we show that the dependence of the asymptotic duality gap on the gradient noise and domain radius 
% in Theorem \ref{thr:aoftrl}
is information theoretically optimal. 

Let $\mathcal{A}$ be some deterministic algorithm that estimates saddle points. 
%In our setting 
$\mathcal{A}$ has only access to a noisy gradient oracle of $f$ that can be called $T$ times. At timestep $t$ the algorithm chooses a point $(x_t,y_t)$ and the oracle returns  $g_x(x_t,y_t)=\nabla_x f(x_t,y_t)+\zeta^X_t$ and   $g_y(x_t,y_t)=\nabla_y f(x_t,y_t)+\zeta^Y_t$. The noise is bounded as follows: $\|\zeta^X_t\|\leq Z_X\ \text{ and }\ \|\zeta^Y_t\|\leq Z_Y\ \forall t \in[T]$. The algorithm has knowledge of the constants $Z_X$ and $Z_Y$ but does not know the exact noise values. In this setting we can derive the following lower bound: 
\begin{theorem}\label{thr:lb}
    Consider a deterministic algorithm $\mathcal{A}$ that estimates saddle points of convex concave functions $f(x,y)$ over the domain $\|x\|\leq D_X$, $\|y\|\leq D_Y$, where $x$ and $y$ are $d$-dimensional vectors, using $T$ adaptive queries on noisy gradient oracles with $\|\zeta^X_t\|\leq Z_X$ and  $\|\zeta^Y_t\| \leq Z_Y$  for all $t\in[T]$ and $Z_Y\leq D/2$, $Z_X\leq D/2$, where $D=min\{D_X,D_Y\}$. For any such $\mathcal{A}$:
    \begin{itemize}
    \item[ ] There exists a convex concave (bilinear) function $f(x,y)$ and a noise sequence realisation, such that $\mathcal{A}$ returns a point $(x_0,y_0)$ that has distance at least $\frac{Z_X+Z_Y}{\sqrt{2}}$ from any saddle point of $f$ and duality gap $f(x_0,y)-f(x,y_0)\geq\frac{1}{4}Z_Y D_Y+\frac{1}{4}Z_X D_X$ for some pair $(x,y)$ inside the domain $\|x\|\leq D_X$, $\|y\|\leq D_Y$.
    \end{itemize} 
    % \begin{itemize}
    %     \item There exists a convex concave (bilinear) function $f(x,y)$ and a noise sequence realisation, such that $\mathcal{A}$ returns a point that has distance at least $\frac{Z_X+Z_Y}{\sqrt{2}}$ from any saddle point of $f$.
    %     \item There exists a convex concave (bilinear) function $f(x,y)$ and a noise sequence realisation such that $\mathcal{A}$ returns a point $(x_0,y_0)$ with duality gap $f(x_0,y)-f(x,y_0)\geq\frac{1}{4}Z_Y D_Y+\frac{1}{4}Z_X D_X$ for some pair $(x,y)$ inside the domain $\|x\|\leq D_X$, $\|y\|\leq D_Y$.
    %     %\item There exists a smooth strongly convex - strongly concave function $f(x,y)$ and a noise sequence realisation, such that $\mathcal{A}$ returns a point that has distance at least $\frac{Z_X+Z_Y}{\sqrt{2}}$ from any saddle point of $f$.
    % \end{itemize} 
\end{theorem}

Theorem \ref{thr:aoftrl} and Theorem \ref{thr:lb} provide a rather complete characterization of the convergence properties of robust OFTRL under corrupted gradients. To the best of our knowledge, these characterization results are novel (see the related work section for comparison to prior work). Moreover, we can prove similar results for Optimistic Mirror Descent Ascent, which is known to have an exponential last iterate convergence rate for objectives that satisfy the \textit{Metric Subregularity} (MS) condition \cite{wei2021linear}. We provide this analysis in the appendix. The latter results are novel and of independent interest for minimax optimization and they could be useful for performative RL if we prove that the Lagrangian \eqref{eq.perf_rl.langrangian} satisfies the MS condition or if we simply make it satisfy MS by adding regularization on both variables.

\section{Formal Setting}

%Since the application of interest is performative reinforcement learning, we further define the corresponding framework, introduced by \cite{mandal2023performative}. 

%The focus of this work is on the performative reinforcement learning framework, introduced by \cite{mandal2023performative}.
%

The focus of this work is on the performative reinforcement learning framework, introduced by \cite{mandal2023performative}.

\subsection{Policy-dependent Markov Decision Process}

The performative reinforcement learning framework considers a policy-dependent Markov Decision Process (MDP) defined as tuple  $M(\pi) = (\mathcal{S}, \mathcal{A}, P^\pi, r^\pi, \rho, \gamma)$, where: $\mathcal{S}$ is a finite state space, $\mathcal{A}$ is a finite action space, $\pi: \mathcal{S} \rightarrow \mathcal P (A)$ is a stochastic policy,  $P^{\pi}:\mathcal{S}\times\mathcal{A}\rightarrow\mathcal{P}(\mathcal{S})$ is a transition function, with $P^{\pi}(s, a, s')$ denoting the probability of transition to state $s'$ when action $a$ is taken in state $s$, $r^\pi:\mathcal{S}\times\mathcal{A}\rightarrow \mathbb{R}$ is the reward function, $\rho \in\mathcal{P}(\mathcal{S})$ is the initial state distribution, and $\gamma \in [0,1)$ is the discount factor. We denote $S = |\mathcal{S}|$ and $A = |\mathcal{A}|$, and we assume that rewards are bounded, i.e., $r^\pi(s,a)\leq R$ for some (unknown) constant $R$.
%
%$\forall (s,a)\in\mathcal{S}\times\mathcal{A}$, $r^\pi(s,a)\leq R$ for some (unknown) constant $R$.

\subsection{Performatively Stable Policy} 
%\textbf{Performatively Stable Policy} 
To define a solution concept in this framework, we  define the value of policy $\pi$ in $M(\pi')$ given initial state distribution $\rho$ as $V_{\pi'}^{\pi}(\rho) = \expctu{\tau}{\sum_{t=0}^\infty \gamma^{t} \cdot r^{\pi'}(s_t, a_t) | \rho}$,
%\begin{align*}
%    V_{\pi'}^{\pi}(\rho) = \expctu{\tau}{\sum_{t=0}^\infty \gamma^{t} \cdot r^{\pi'}(s_t, a_t) | \rho},
%\end{align*}
% %
% %
% $
%     V_{\pi'}^{\pi}(\rho) = \expctu{\tau}{\sum_{t=1}^\infty \gamma^{t-1} \cdot r^{\pi'}(s_t, a_t) | \rho}  
% $,
%
%
% \begin{align*}
%     V_{\pi'}^{\pi}(\rho) = \expctu{\tau}{\sum_{t=1}^\infty \gamma^{t-1} \cdot r^{\pi'}(s_t, a_t) | \rho},  
% \end{align*}
where $\tau = (s_0, a_0, s_1, a_1, ...)$ is a trajectory obtained by executing policy $\pi$ in MDP $M(\pi')$.  The solution concept of interest is a {\em performatively stable} policy $\pi_S$, which satisfies: $\pi_{S} \in \argmax_{\pi} V^{\pi}_{\pi_S}(\rho)$. 

%\textbf{Occupancy Measure.} 
\subsection{Occupancy Measures}
We denote by $d^{\pi}(s, a)$ the occupancy measure of policy $\pi$ in MDP $M(\pi)$, i.e., 
% \begin{align*}
%     d^{\pi} = \expctu{\tau}{\sum_{t=0}^\infty \gamma^{t} \cdot \ind{s_t = s, a_t = a} | \rho}. 
% \end{align*}
$d^{\pi} = \expctu{\tau}{\sum_{t=0}^\infty \gamma^{t} \cdot \ind{s_t = s, a_t = a} | \rho}$.
%
%$d^{\pi} = \expctu{\tau}{\sum_{t=1}^\infty \gamma^{t-1} \cdot \ind{s_t = s, a_t = a} | \rho}$. 
%
Occupancy measure $d^\pi$ satisfies the Bellman flow constraint
$$\forall s: \rho(s) +\gamma \sum\limits_{s',a} d^\pi(s',a) P^\pi(s',a,s)= \sum\limits_{a}d^\pi(s,a).$$
%
%
%\goran{Define occupancy measure $d^{\pi}$ - is above ok?}
%When give an offline dataset $D_m$ containting tuples $m$ samples $(s_i, a_i, s_i', r_i)$, 
%
%
%
For a generic $d \ge 0$, we define policy $\pi^{\downarrow d}$ as 
\begin{align}\label{eq.occupancy_policy}
    \pi^{\downarrow d}(s, a) = \begin{cases}\frac{d(s, a)}{\sum_{a'} d(s, a')} &\text{ if } \sum_{a'} d(s, a') > 0 \\
    %\frac{1}{A} &\text{ othw. }
    \frac{1}{A} &\text{ othw. }
    \end{cases}.
\end{align}
If $d$ is a valid occupancy measure in $M(\pi^{\downarrow d})$ (i.e., if it satisfies the Bellman flow constraints), the occupancy measure of $\pi^{\downarrow d}$, i.e., $d^{\pi^{\downarrow d}}$, is equal to $d$. In general, $d^{\pi^{\downarrow d}}$ and $d$ may differ. The occupancy measure of a performatively stable policy is denoted by $d_{S}$.

\subsection{Data Generation Process}
%\textbf{Finite Sample Case. }
We are interested in a finite sample, offline RL regime. 
%
%Given dataset $D_n$ containing $m$ samples $(s_i, a_i, s_i', r_i)$, we can replace $\mathcal L$ with its empirical version \citep{mandal2023performative}.% 
The data generation process is assumed to be i.i.d.: $(s_i, a_i)$ is sampled from normalized $d_{n}$, i.e., $(s_i, a_i) \sim (1-\gamma) \cdot d_{n}$, $r_i = r_n(s_i, a_i)$ and $s'_{i}$ is sampled from $P_n$ transition kernel, i.e., $s'_{i} \sim P_n(s_i, a_i, \cdot)$. We also make a coverage assumption that $d_{n}(s,a)$ is positive for all $s\in\mathcal{S}, a \in\mathcal{A}$. This assumption can be satisfied if the dynamics make all states reachable and we mix some exploratory random policy with $\pi_n$. In particular, we assume that if $d(s,a)\geq c$  then $\forall (s,a)\in\mathcal{S}\times\mathcal{A}$, $d^{\pi^{\downarrow d}}(s,a)\geq B(c)>0$. For the rest of this work, all deployed policies will satisfy this condition and thus the coverage assumption will hold with constant $B(c)$.
%\goran{is the above correct}
%\goran{Should we assume say: } 

%Given dataset $D_n$ containing tuples $m$ samples $(s_i, a_i, s_i', r_i)$, we can replace $\mathcal L$ with an empirical version of $\mathcal{\hat L}$ 
%\begin{align*}
%    \mathcal{\hat L} (d, h, D_{n}) &= \frac{1}{2}\norm{d}_2^2 + \sum_s h(s)\rho(s) \\&+ \sum_{i = 1}^{m} \frac{d(s_i, a_i)}{d_t(s_i, a_i)} 
%    \cdot \frac{r_i + \gamma h(s_i') - h(s_i)}{m (1-\gamma)}.
%\end{align*}
%To ensure convergence, the data generation process is i.i.d. and depends on $M_n$, as explained in \cite{mandal2023performative}.%\goran{Add a complete desription?}

%\textbf{Our Objective.}
%
%\subsection{Our Objective}
\subsection{Contamination Model}
%Our goal is to provide an approximate solution to problem \eqref{eq.perf_rl.langrangian} when the data generation process is corrupted. 
%
We consider the Huber $\epsilon$-contamination model, where an $\epsilon$ fraction of data points can be corrupted. In this case, both transition and reward can be corrupted. In particular, when the adversary corrupts a sample $(s_i, a_i, s_i', r_i)$,
next state $s_i'$ can be replaced by any $s_c \in \mathcal{S}$ and reward $r_i$ can be placed by any $r_c \in \mathbb R$. We assume that $s_i$ and $a_i$ are not corrupted. The latter assumption is made for technical simplicity. In the appendix, at the end of proof of Theorem \ref{thr:robust_grad}, we show how the assumption could be avoided and still derive similar results with a more complicated analysis.

% \goran{I'm not sure if this remark is needed, but it's also not clear to me what exactly we are showing.}

% \begin{remark}
%     In the main paper, we focus on guarantees related to the quality of the solution to \eqref{eq.perf_rl.langrangian}. In the appendix, we show how these could be incorporated in the convergence analysis of repeatedly solving  \eqref{eq.perf_rl.langrangian}.  
% \end{remark}

% ......~\\
% ......~\\
% ......~\\
% ......~\\
% ......~\\
% ......~\\
% ......~\\
% ......~\\
% ......~\\
% ......~\\
% ......~\\
% ......~\\
% ......~\\
% ......~\\
% ......~\\
% ......~\\
% ......~\\
% ......~\\
% ......~\\

% ......~\\
% ......~\\
% ......~\\
% ......~\\
% ......~\\
% ......~\\
% ......~\\
% ......~\\
% ......~\\
% ......~\\
% ......~\\

%\todo{finish by 2}

%% file: 4_corruption_robust_prl.tex
% !TEX root =  main.tex
%%%%%%%%%%%%%%%%%%%%%%%%%%%%%%%%%%%%%
%%%%%%%%%%%%%%%%%%%%%%%%%%%%%%%%%%%%%

\section{Corruption-Robust Performative RL}\label{sec.robust_performative_rl}
 
We follow prior work on performative RL, and study repeated retraining and its convergce to an approximate performatively stable point. In repeated retraining, the policy is retrained after each deployment round. In a canonical setting, we have access to the MDP model $M(\pi_{n})$, where $\pi_{n}$ is the policy deployed in round $n$. 
We will denote this MDP model by $M_n = M(\pi_{n})$, and its reward and transition function by $r_n$ and $P_n$, respectively. Assuming access to $M_n$, 
repeated retraining optimizes the following regularized RL objective after each round $n$: 
%solves  MDP $M_n = M(\pi_n)$ induced by the policy deployed in at round $n$, the policy is obtained from Eq. \eqref{eq.occupancy_policy} and the occupancy measure that is a solution to 
\begin{align*}
    d_{n+1}^* \in \arg\max_{d \in \mathcal C(M_n)} \sum_{s, a} d(s, a) \cdot r_n - \frac{\lambda}{2} \cdot \norm{d}_2^2,
\end{align*}
where $\mathcal C(M_n)$ is the space of occupancy measures compatible with the MDP $M_n$, satisfying: 
\begin{align*}
     \forall s: \rho(s) +\gamma \sum\limits_{s',a} d(s', a) P_n(s',a,s)= \sum\limits_{a}d(s,a).
\end{align*}
Note that $\pi_n$ is obtained from $d_{n}^*$, and is defined as $\pi_n := \pi^{\downarrow d_{n}^*}$. 
To directly optimize from data, \citet{mandal2023performative} consider the following minimax optimization problem:
%
%\citet{mandal2023performative} propose the following minimax optimization problem
%$D_n$ 
%As shown by
%\citet{mandal2023performative}, repeated retraining  can instead repeatedly solve
\begin{align}\label{eq.perf_rl.langrangian}
    d_{n+1}^*, h_{n+1}^* \in \arg\max_{d \ge 0} \arg\min_{h} \mathcal{L} (d, h, M_{n}),
\end{align}
where objective $\mathcal{L}$ is the Lagrangian of the regularized RL objective, defined as:
\begin{align*}
    \mathcal{L} (d, h, M_n) &= -\frac{\lambda}{2}\norm{d}_2^2 + \sum_s h(s)\rho(s) + \sum_{s,a} d(s, a) \times 
    \\&  \left [ r_n(s, a) - h(s) + \gamma \sum_{s'}P_n(s, a, s') h(s') \right ].
\end{align*}
%They further that show repeated optimization converges to a solution $(\tilde d, \tilde h)$ which yields an approximately stable policy $\pi^{\tilde d}$.
%\citet{mandal2023performative} shows that 
%
% if $\pi_n = \pi^{\downarrow d_{n-1}^*}$
Given dataset $D_n$ containing $m$ samples $(s_i, a_i, s_i', r_i)$, we can replace $\mathcal L$ with its empirical version $\mathcal {\hat L}$:
\begin{align*}
    \mathcal{\hat L} (d, h, M_n) &= -\frac{\lambda}{2}\norm{d}_2^2 + \sum_s h(s)\rho(s)\\ 
    &+ \sum_{(s,a, r, s') \in D_n} \frac{d(s, a)}{d_n(s, a)} \cdot \frac{r - h(s) + \gamma \cdot h(s')}{m \cdot (1-\gamma)},
\end{align*}
where $m$ is the size of $D_n$ and $d_n(s, a)$ is the occupancy measure of policy $\pi_n$ in $M_n$. 
This allow us to directly optimize from data and establish last-iterate convergence guarantees from finite samples. 

\begin{remark}
In the above framework we assumed for simplicity knowledge of the occupancy measure $d_{n}$ that generates the samples. In practice, we can estimate it up to arbitrary accuracy from the samples using Monte-Carlo. 
\end{remark}

% \citep{mandal2023performative}.

%%%%
%%%%
%%%%
% \textbf{Repeated Retraining.} \cite{mandal2023performative} show that repeatedly solving 
% \begin{align}\label{eq.perf_rl.langrangian}
%     d_{n+1}^*, h_{n+1}^* \in \arg\max_{d \ge 0} \arg\min_{h} \mathcal{L} (d, h, M_{n})
% \end{align}
% converges to a solution $(\tilde d, \tilde h)$ which yields an approximately stable policy $\pi^{\tilde d}$. Objective $\mathcal{L}$ is defined as
% \begin{align*}
%     \mathcal{L} (d, h, M_n) &= \frac{\lambda}{2}\norm{d}_2^2 + \sum_s h(s)\rho(s) + \sum_{s,a} d(s, a) \times 
%     \\&  \left [ r_n(s, a) - h(s) + \gamma \sum_{s'}P_n(s, a, s') h(s') \right ],
% \end{align*}
% They further show that 
%  Moreover, let $d_n$ be the  occupancy measure of $\pi_n$ in $M(\pi_n)$. The convergence property holds if $\pi_n = \pi^{\downarrow d_{n-1}^*}$.
% %
% Given dataset $D_n$ containing $m$ samples $(s_i, a_i, s_i', r_i)$, we can replace $\mathcal L$ with its empirical version \citep{mandal2023performative}.
%%%%
%%%%

%\subsection{Corruption-Robust Repeated Retraining}

%\textbf{Robust Repeated Retraining.}

\subsection{Robust Repeated Retraining}

We build on the repeated retraining approach described above, but consider the case where dataset $D_n$ is corrupted, according to the corruption model specified in the formal setting. Our approach is depicted in Algorithm \ref{alg:rob_rr_prl}.

In each round $n$, the algorithm first collects a contaminated data $D_n$ by deploying policy $\pi_n$ in $M(\pi_n)$. The next step is to approximately solve problem \eqref{eq.perf_rl.langrangian} -- given that $D_n$ is corrupted directly utilizing $\mathcal{\hat L}$ instead of $\mathcal L$ may not yield any guarantees. Hence, we apply robust OFTR, described in the preliminaries, with $f = \mathcal L$ and the robust gradient estimators from the next subsection.
Finally, the algorithm calculates new policy $\pi_{n+1}$ by mixing the occupancy measure obtained via robust OFTRL with an exploratory random policy and applying Eq. \eqref{eq.occupancy_policy}.  We do this by adding some positive constant $c$ to $\bar d$.  The mixing step ensures the coverage property for the next iteration will be satisfied.

In the next subsection, we first propose a novel robust gradient estimator of $\mathcal{L}$, which can be combined with robust OFTRL to approximately solve \eqref{eq.perf_rl.langrangian}. We provide guarantees on the estimation error for this estimator that scales with the corruption level $\epsilon$. We then focus on the convegence analysis of Algorithm \ref{alg:rob_rr_prl}, and show that it exhibits last-iterate convergence to an approximately stable policy, with the approximation error proportional to $\sqrt{\epsilon}$.

%yields a policy which is close to a performatively stable policy. 

%The next step is to provide an approximate solution to problem \eqref{eq.perf_rl.langrangian} when the data generation process is corrupted. To approximately solve \eqref{eq.perf_rl.langrangian}, we apply robust OFTR, described in the preliminaries. Our approach is depicted in Algorithm \ref{alg:rob_rr_prl}. In each round $n$, the algorithm first collects a contaminated data $D_n$ by deploying policy $\pi_n$ in $M(\pi_n)$. The algorith m then applices

%We build on this appraoch and aim, but consider the case where dataset $D_n$ us corrupted, according to the corruption model specified in the formal setting. 

%The approach is similar to a standard repeated retraining, however, to approximately solve \eqref{eq.perf_rl.langrangian}, we apply robust OFTRL -- we note that one can also utilize robust GDA, as discussed in the appendix.  We restrict the domain of variables $d$ and $h$ in \eqref{eq.perf_rl.langrangian}. We denote the corresponding domains by $\mathcal D = \{ d : 0 \le d(s, a) \le \frac{1}{1-\gamma}\}$ and $\mathcal H = \{ h : -h_{max} \le h(s, a) \le h_{max} \}$, where $h_{max} > 0$.

 \begin{algorithm}%[H]
\begin{algorithmic}[1]
\caption{Robust Repeated Retraining}\label{alg:rob_rr_prl}
%\KwData{ $\{g_d^i\ |\ i \in [m]\}$ }
\STATE $\pi_0(a|s) = 1/|\mathcal{A}|, \forall s,a \in \mathcal{S}\times\mathcal{A}$\\
\FOR{$n=1,..., N$}
  \STATE $D_n \leftarrow$ Sample $d^{\pi_{n}}$ + Huber $\epsilon$-contamination 
  \STATE$ \bar d_{n+1} \leftarrow$ Apply Robust OFTRL with $f = \mathcal L$ and gradient estimators from Section {\em Robust Gradient Estimation} on $D_n$ %\ref{sec.robust_gradients} on $D_n$ 
  \STATE $\tilde{d}_{n+1} \leftarrow \bar{d}_{n+1}+c$, where $c > 0$
  %\STATE$ \pi_{n+1}(s, a) \leftarrow \frac{{d}_{n+1}(s, a)+c}{\sum_{a'} ({d}_{n+1}(s, a')+c)} $
    \STATE $\pi_{n+1} \leftarrow \pi^{\downarrow \tilde{d}_{n+1}}$, where $\pi^{\downarrow d}$ is defined in Eq. \eqref{eq.occupancy_policy}.
 \ENDFOR
 \STATE \textbf{Return} $\tilde{d}_N$ %=\bar{d}_{N}+c$ 
 \end{algorithmic}
\end{algorithm}

\subsection{Robust Gradient Estimation}\label{sec.robust_gradients} 

We now focus on robust estimation of gradients $g_d := \nabla_d \mathcal{L}(d,h, M_n)$ and $g_h := \nabla_h \mathcal{L}(d,h, M_n)$. 
If some of the samples  
are corrupted, naive averaging may not suffice. Therefore, we explore robust alternatives. We propose the following steps:
%
%{\bf Our approach.} We propose the following steps:
\begin{itemize}
\item For the gradient w.r.t. $d$, given a subset $\{(s_i,s'_i,a_i,r_i)\ |\ i \in [\tilde{m}]\}$ of $D_n$, with corruption level $\epsilon$, we  apply a robust mean estimator to the dataset $D_d:=\{\hat g_d^i\ |\ i \in [\tilde{m}]\}$, where each sample $\hat  g_d^i$ is a single-entry $|\mathcal{S}|\cdot | \mathcal{A} |$-dimensional vector constructed by sample $(s_i,s'_i,a_i,r_i)$ according to the formula $\hat  g_d^i(s,a)=\ind{(s_i,a_i)=(s,a)} \cdot  \frac{\gamma h(s_i') -h(s_i)+r_i}{(1-\gamma)\cdot d_{n}(s_i,a_i)}$. $D_d$ contains both corrupted and clean samples. Each clean sample $\hat g_d^i$ is an unbiased estimator of $g_d+\lambda d$ (proof provided in the appendix). Finally, we add $-\lambda d(s,a)$ to the robust mean of $D_d$.
\item For the gradient w.r.t. $h$, given a subset $\{(s_i,s'_i,a_i,r_i)\ |\ i \in [\tilde{m}]\}$ of $D_n$ (disjoint with that used for $g_d$), with corruption level $\epsilon$, we  apply a robust mean estimator to the dataset $D_h:=\{\hat g_h^i\ |\ i \in [\tilde{m}]\}$, where each sample $\hat 
 g_h^i$ is a $|\mathcal{S}|$-dimensional vector constructed by sample $(s_i,s'_i,a_i,r_i)$ according to the formula $\hat 
 g_h^i(s')=d(s_i,a_i) \frac{\gamma \ind{s'=s_i'} -\ind{s'=s_i}}{(1-\gamma)\cdot d_{n}(s_i,a_i)}$. Each clean sample $\hat g_h^i$ is an unbiased estimator of $g_h-\rho$. 
 %Finally, we add to the robust mean of  $D_h$ the term $\rho$.
 Finally, we add $\rho$ to the robust mean of  $D_h$.
 \end{itemize}

% \begin{remark}
% %In the above framework we assumed for simplicity knowledge of the occupancy measure $d_{n}$ that generates the samples. In practice, we can estimate it up to arbitrary accuracy from the samples using Monte-Carlo. 
% %
% We note that corruption does not affect the empirical occupancy measure $d_n$, since at each transition sample  $s_i$ and $a_i$ are always uncorrupted.
% \end{remark}

\begin{algorithm}%[H]
\begin{algorithmic}[1]
\caption{Robust coordinate-wise mean}\label{alg:rob_mean}
\STATE Input: $\{\hat  g_d^i\ |\ i \in [\tilde{m}]\}$, $\epsilon$ 
%\For{$k=1,...|\mathcal{S}|\cdot|\mathcal{A}|$}{
\FOR{$k=1,..., S\cdot A$}
  \STATE $data \gets \{\hat g_d^1[k],..., \hat g_d^{\tilde{m}}[k]\}$
  \STATE $Med_k \gets median(data)$ 
  \STATE$clean\gets (1-\epsilon)\cdot \tilde{m}\text{ closest }data\text{ entries to }Med_k$ 
  \STATE $\hat{g_d}[k] \gets mean(clean)$
\ENDFOR
 \STATE \textbf{Return} $\hat{g_d}$ 
 \end{algorithmic}
\end{algorithm} 

{\bf Robust Mean Estimators.} As a robust mean estimator we consider any estimator whose error has one (statistical) term, vanishing with the number $\tilde{m}$ of samples and one bias term polynomial to the frequency $\epsilon$ of corrupted samples. The error of a robust estimator should not scale with the magnitude of corruption, especially when corruption can be unbounded, as it can happen in the reward samples in our setting. For the estimation of $g_d$, we use Algorithm \ref{alg:rob_mean}. 
For the estimation of $g_h$ it suffices to apply naive averaging to achieve the kind of result that we wish.
The errors of gradient estimators  are analysed in the following theorem.

\begin{theorem}\label{thr:robust_grad}
     Let us use $\hat{g}_h =\frac{1}{\tilde{m}}\sum\limits_{i=1}^{\tilde{m}} \hat g_h^i$ to estimate $g_h$ and Algorithm \ref{alg:rob_mean} to estimate $g_d$, and assume that the corruption level in the respective datasets is bounded by $\epsilon <0.5$.
    % 
    %
    %We assume $\tilde{\epsilon}<0.5$. We use $\hat{g}_h =\frac{1}{\tilde{m}}\sum\limits_{i=1}^{\tilde{m}} \hat g_h^i$ to estimate $g_h$ and apply Algorithm \ref{alg:rob_mean} to estimate $g_d$. 
    Then with probability at least $1-\delta$ the estimation errors satisfy the following guarantees: 
    \begin{align*}
        \|\hat{g}_h-g_h\|_1 &\leq \underbrace{\frac{4}{(1-\gamma)^2B(c)}\left(\frac{\sqrt{S\ \log(4S/\delta)}}{\sqrt{\tilde{m}}}+\epsilon\right)}_{E_1(\tilde{m},\epsilon,\delta)},
        \\
    \|\hat{g}_d-g_d\|_2 &\leq \underbrace{6\sqrt{SA}\frac{2h_{max}+R}{(1-\gamma)B(c)}\left(\frac{\sqrt{2 \log\left(\frac{4SA}{\delta}\right)}}{\sqrt{\tilde{m}}}+2\epsilon\right)}_{E_2(\tilde{m},\epsilon,\delta)}.
    \end{align*}
\end{theorem}

\input{./5_analysis}

%% file: 5_analysis.tex
% !TEX root =  main.tex
%%%%%%%%%%%%%%%%%%%%%%%%%%%%%%%%%%%%%
%%%%%%%%%%%%%%%%%%%%%%%%%%%%%%%%%%%%%
\subsection{Convergence Analysis}\label{sec.main_results}

%As our main result, 
Our goal is to show that
%We are interested in showing that 
the repeated optimization approach as specified in Algorithm \ref{alg:rob_rr_prl}, outputs a solution that is approximately stable. 
We restrict the domain of variables $d$ and $h$ in  %\eqref{eq.perf_rl.langrangian}
 Algorithm \ref{alg:rob_rr_prl} as follows: $\mathcal D = \{ d : 0 \le d(s, a) \le \frac{1}{1-\gamma}\}$ and $\mathcal H = \{ h : -h_{max} \le h(s) \le h_{max} \}$, where $h_{max} > 0$. 
 %Furthermore, we will assume a bounded corruption level without explicitly stating this in the formal statements. More specifically, and as explained in the next paragraph, the we will will assume that the dataset $D_p$ is large enough, so that $m$ samples can be split in $2T$ batches, each having corruption level of at most $\tilde \epsilon < 0$. 
 Furthermore, we will assume that $D_n$ has a large enough number $m$ of samples and a bounded corruption level, as specified by the following assumption, whose role we explain in the next paragraph.  
\begin{assumption}\label{assume.bound_corrupt} (bounded corruption)
     Every $D_n$ can be split in $2T$ batches, each having corruption level of at most $\epsilon < 0$.
\end{assumption}

\textbf{Robust OFTRL Guarantees.} To prove the convergence of Algorithm \ref{alg:rob_rr_prl}, we first need to provide an upper bound on the quality of the solution that robust OFTRL (Algorithm \ref{alg:aoftrl}) outputs.  
We show that after sufficiently many iterations $T$, robust OFTRL with gradient estimators defined in the previous section 
%Section \ref{sec.robust_gradients} 
can find an approximate saddle point to \eqref{eq.perf_rl.langrangian}, but with bounded domains $\mathcal D$ and $\mathcal H$. 
Similarly to \cite{https://doi.org/10.1111/rssb.12364}, to avoid statistical issues, we split the original dataset of $m$ samples in $2T$ equal batches, assuming that each batch has corruption level at most $\epsilon<0.5$. In each iteration, we apply each gradient estimator on a fresh batch. We refer to this process as {\em batch-splitting}. 
\begin{lemma}\label{lm.approx.bound}
There exists $T$ such that the output of Algorithm \ref{alg:aoftrl} run for $T$ iterations on $f = \mathcal L$, with $\mathcal X = \mathcal D$, $\mathcal Y = \mathcal H$, $g_{X, t} = \hat g_h$,  $g_{Y, t} = \hat g_d$, and {\em batch-splitting}, satisfies
\begin{align}\label{eq.gda_duality_gap_bound}
    \max\limits_{d \in \mathcal{D}}\mathcal{L}(d,\bar{h}, M_n)- \min\limits_{h \in \mathcal{H}}\mathcal{L}(\bar{d},h, M_n) \le 7C(\delta)
\end{align}  
under Assumption \ref{assume.bound_corrupt}, with probability at least $1-\delta$, where\\ $C(\delta) := \sqrt{S}\left(E_1\left(\frac{m}{2T},\epsilon,\frac{\delta}{T}\right) h_{max} +  \frac{\sqrt{A}E_2\left(\frac{m}{2T},\epsilon,\frac{\delta}{T}\right)}{1-\gamma}\right)$.
\end{lemma}

Now, we want to provide a bound on the quality of the output of robust OFTRL w.r.t. the true solution of $\eqref{eq.perf_rl.langrangian}$. Consider the set of $d \ge 0$ that satisfy the Bellman flow constraint in $M_n$ and denote its {\em Hoffman constant} \citep{garber2019logarithmic} by $\sigma_n$. To simplify the exposition, we define quantity $\alpha(M_n, \delta)$:
%quantity $\alpha(M_n, \delta)$:
    \begin{align*}
        \alpha(M_n,\delta) = \sqrt{\frac{14C(\delta)}{\lambda}}+ C_n' +\sqrt{ 2C_n' \left(\frac{\|r_n\|}{\lambda}+\frac{1}{1-\gamma}\right)}
    \end{align*}
    where $C_n'=\frac{ \|r_n\|_2\sqrt{S}\sigma_n^{-1/2}}{(1-\gamma) h_{max}}$. 
    %\goran{which norm are we using here? fix the notation.} 
    For a generic MDP $M(\pi)$, we analogously define $\alpha(M(\pi),\delta)$. 
    Next we show that robust OFTRL outputs an approximately optimal solution to \eqref{eq.perf_rl.langrangian}.

\begin{theorem}\label{thm.pl_robust_gda}
    Consider the robust OFTRL from Lemma \ref{lm.approx.bound}, and assume its number of iterations $T$ is s.t. \eqref{eq.gda_duality_gap_bound} is holds. Under Assumption \ref{assume.bound_corrupt}, the output of robust OFTRL $\bar d$ satisfies $\|d^*_n-\bar{d}\|_2 \leq \alpha(M_n,\delta)$ with probability at least $1-\delta$.
    %$\left\|d^*_n-\tilde{d}\right\| \leq \alpha(M_n)$.
\end{theorem}

\textbf{Convergence of Algorithm \ref{alg:rob_rr_prl}.} 
Now, we are ready to derive convergence guarantees of  Algorithm \ref{alg:rob_rr_prl}. To do so, we need two additional assumptions, which we use to establish contraction properties of repeated retraining.
The first one, $\epsilon$-sensitivity is a standard in the literature on performative prediction, and we take it from prior work on performative RL \citep{mandal2023performative}. 

\begin{assumption}\label{assume.sensitivity} 
%\goran{Add here the sensitivity assumption} 
($\tilde \epsilon$-sensitivity) For any two MDPs $M(\pi)$ and $M(\pi')$, the following holds $\|r^{\pi} - r^{\pi'}\|_2 \le \tilde \eps_r \|d^{\pi} - d^{\pi'}\|_2$ and $\|P^{\pi} - P^{\pi'}\|_2 \le \tilde \eps_p \|d^{\pi} - d^{\pi'}\|_2$.
\end{assumption}

The second one, is a rather weak assumption requiring that any MDP induced by the deployed policy does not have an infinite factor $\alpha$ for fixed $\delta$. 

\begin{assumption}\label{assume.hoffman} 
%\goran{Add here the sensitivity assumption} 
($\bar \alpha$-boundedness) For any MDP $M(\pi)$ induced by stationary policy $\pi$ we have $\alpha(M(\pi),\delta) \le \bar \alpha(\delta)$.   
\end{assumption}
%These two assumptions are needed to establish contraction properties of repeated retraining. 
%
%Finally, we mix the calculated policy at $n$-th iteration with an exploratory random policy in order to ensure the coverage property for the next iteration. We do this by adding some positive constant $c$ to $\bar d$. 
%
%After this step, the distance between $\bar d$ and $d^*_n$ is at most $\bar{C}(\delta):=\bar \alpha(\delta) + c\sqrt{SA}$. 

\begin{figure*}[h!]
    \centering
    \begin{subfigure}[b]{0.24\textwidth}
         \centering
         \includegraphics[width=\textwidth]{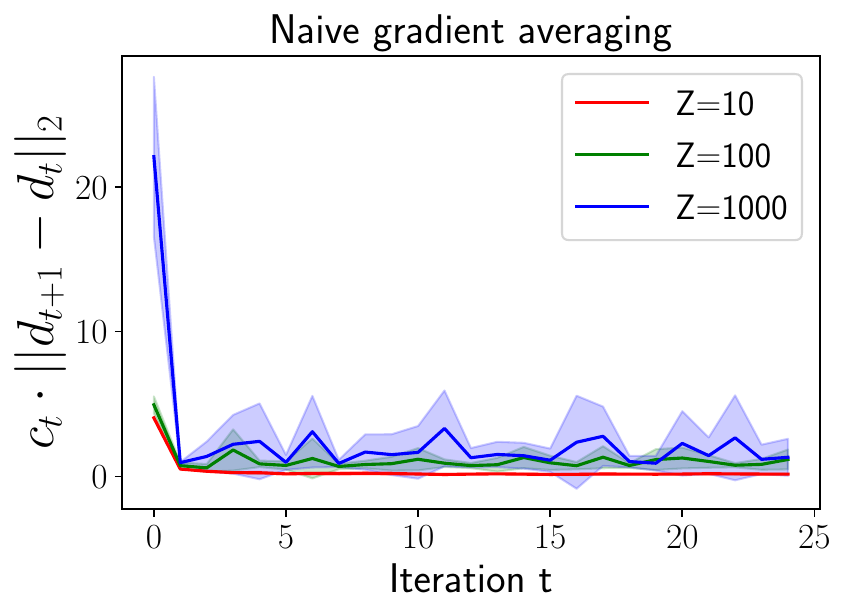}
         \caption{Fixed $\epsilon$, naive}
         \label{plot:fixed_epsilon_naive}
     \end{subfigure}
     \hfill
     \begin{subfigure}[b]{0.24\textwidth}
         \centering
         \includegraphics[width=\textwidth]{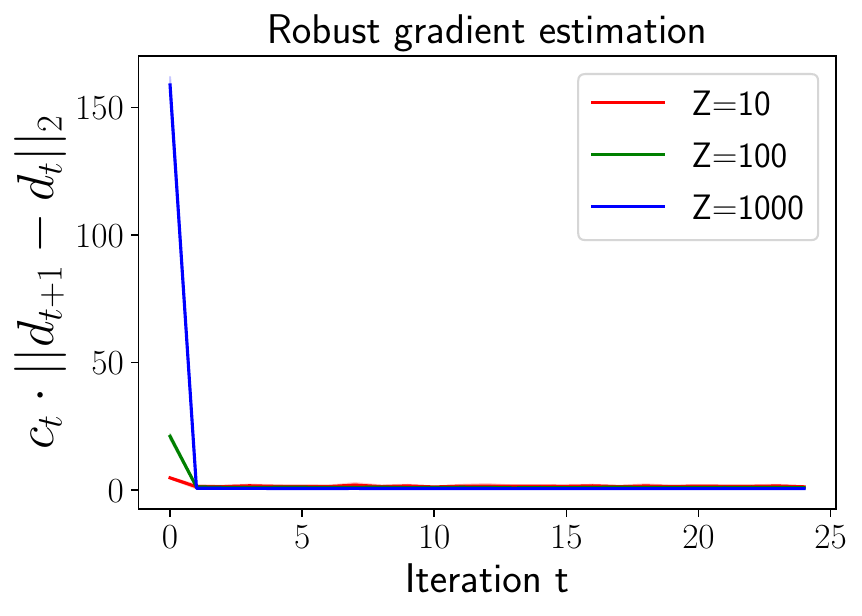}
         \caption{Fixed $\epsilon$, robust}
         \label{plot:fixed_epsilon_robust}
     \end{subfigure}
     \hfill
     \begin{subfigure}[b]{0.24\textwidth}
         \centering
         \includegraphics[width=\textwidth]{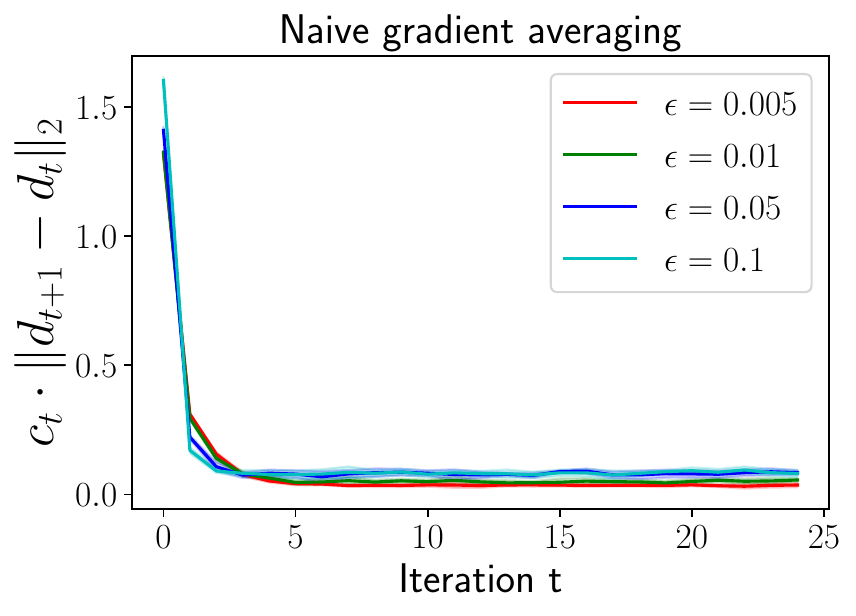}
         \caption{Fixed $Z$, naive}
         \label{plot:fixed_mu_naive}
     \end{subfigure}
     \hfill
     \begin{subfigure}[b]{0.24\textwidth}
         \centering
         \includegraphics[width=\textwidth]{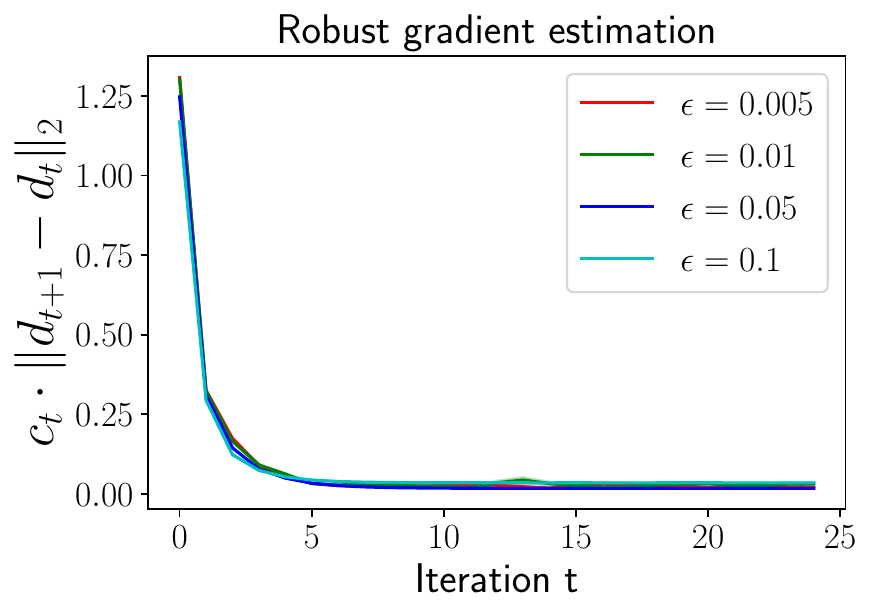}
         \caption{Fixed $Z$, robust}
         \label{plot:fixed_mu_robust}
     \end{subfigure}
    \caption{
    Convergence of repeated retraining approaches: (a) and (d) consider a non-robust variant of Algorithm \ref{alg:rob_mean} that utilizes naive gradient averaging; (b) and (c) consider Algorithm \ref{alg:rob_mean} that utilizes robust gradient estimation. The $y$ axis is  the normalised distance between $d_{t+1}$ and $d_t$ with $c_t=1/\|d_t\|_2$, similar to the experiments in \citep{mandal2023performative}, while the $x$ axis is the number of repeated retraining iterations. For Fig. \ref{plot:fixed_epsilon_naive} and Fig. \ref{plot:fixed_epsilon_robust}, $\epsilon=0.01$ and we vary $Z$. For Fig. \ref{plot:fixed_mu_naive} and Fig. \ref{plot:fixed_mu_robust}, $Z=15$ and we vary $\epsilon$. 
    Each line presents the average over $10$ experiments initialised with different seeds. 
    }
    \label{fig:plots-perf-RL}
\end{figure*}

Note that after each round $n$, the distance between $d_n$ and $d^*_n$ is at most $\bar{C}(\delta):=\bar \alpha(\delta) + c\sqrt{SA}$. This is due to the definition of $\tilde{d}_n$---we obtain $\tilde d$ from $\bar d$ by adding a constant $c>0$ to each of its entry --- and the robust OFTRL guarantees for $\bar d$.  
Using this bound and the previous two assumptions we derive a convergence result for Algorithm \ref{alg:rob_rr_prl} (see the appendix for a more formal statement).

\begin{theorem}\label{thr:perf_rl} (Informal Statement) Under Assumption \ref{assume.bound_corrupt}, Assumption \ref{assume.sensitivity} and Assumption \ref{assume.hoffman}, there exist $\lambda$ and $N$ such that the output of %robust repeated optimization 
Algorithm \ref{alg:rob_rr_prl}
satisfies $\tilde{d}_N \in \{ d \in \mathcal D :\|d-d_S\|_2 \leq \tilde{C}:=4\cdot \bar{C}(\delta/N) \}$ with probability at least $1-\delta$, where $d_S$ is a performatively stable policy. 
\end{theorem}

%\goran{Add the following not as a part of the theorem, but as a discussion point:}

Theorem \ref{thr:perf_rl} states that Robust Repeated Retraining (Algorithm \ref{alg:rob_rr_prl}) exhibits last-iterate convergence to an approximate stable point. Ignoring $\epsilon$-independent additive terms, the approximation error is proportional $\sqrt{\epsilon}$.
%\footnote{The approximation error depends also on $c\sqrt{SA}$. However, this error can be controlled by in Algorithm \ref{alg:rob_rr_prl} through hyperparameter $c$. } 
Namely, factor $\bar{C}(\delta/N)$ depends on $\bar \alpha(\delta)$, which has square root dependence on $\epsilon$ since $C$ is linear in $\epsilon$. In prior work on corruption-robust offline RL, suboptimality gaps can scale as $O(\sqrt{\epsilon})$ or $O(\epsilon)$, depending on the exact setting and assumption (e.g., see~\cite{zhang2022corruption, nika2024corruption}). However, we note that our setting is not directly comparable, since it combines offline RL with repeated retraining. 
We leave the analysis of the tightness of this bound for future work. 

%\begin{remark}
The result in Theorem~\ref{thr:perf_rl} provides an asymptotic convergence guarantee w.r.t. $\lambda$ and $N$. Following the proof of Theorem 1 from~\cite{mandal2023performative}, it is easy to show that $\lambda > \lambda_0$ 
and $N > \frac{1}{1-\lambda_0} \cdot \log(\frac{2}{\bar{C}_0 \cdot (1-\gamma)})$ guarantee the convergence.\footnote{This lower bound on $\lambda$ is the same as the one in Theorem 3 from~\cite{mandal2023performative}, which considers the finite sample case. The lower bounds on $N$ are comparable but not the same, due to differences in  the convergence criteria.} Here, $\lambda_0 = \frac{24 S^{3/2} (2 \tilde \epsilon_r + 5\tilde \epsilon_p)}{(1-\gamma^4)}$, while $\bar{C}_0$ is a lower bound on $4 \cdot \bar{C}(\delta/N)$ independent of $N$, but possibly dependent on $\lambda$.   

% As for $N$, we can again utilize the the proof of Theorem 1 from~\cite{mandal2023performative}. By noting that $\bar{C}(\delta/N)$ has a lower bound independent of $N$, denote it by $\bar{C}_0$, one can show that the statement of the theorem holds for $\lamda > \lambda_0$  and $N > \frac{1}{1-\lambda_0} \cdot \log(\frac{2}{\bar{C}_0 \cdot (1-\gamma)})$. 

We can further assess the return that policy defined by $\tilde{d}_N$, i.e., $\pi^{\downarrow \tilde{d}_N}$, achieves in $\mathcal{M}(\pi_{N})$.
Assuming that the initial state distribution has a full support over state space, one can show that this return is comparable to the return of a performatively stable policy $\pi_{S}$ 
in $M(\pi_{S})$:
%in $\mathcal{M}(\pi_{S})$. 
%In particular, 
it is worse by at most an instance-specific constant proportional to $\tilde{C}$.
We provide further details in the appendix.

%% file: 6_experiments.tex
\section{Experimental Evaluation}\label{sec.experiments}

In this section we experimentally test the efficacy of our approach in performative RL under corruption.
% In the appendix, we report additional experimental results, considering an application scenario based on normal form games. 

\textbf{Environment.} Our MDP model is a $W\times W$ gridworld environment, inspired by the gridworld environment in \citep{triantafyllou2021blame}, where state $s = (i, j)$ encodes the location/cell that the agent occupies. In round $n$, the reward function is defined as $r_n(s, a)=R_{i,j}-c_p \cdot \sum\limits_{a\in\mathcal{A}}d_n(i\cdot W+j,a)$. We set $W=8$, while the values $R_{i,j}$ are defined as in the gridworld environment used in \citep{triantafyllou2021blame}. The transitions are deterministic and only controlled by the four agent's actions (left, right, up, down). Samples are transition tuples $(s_i,s'_i,a_i,r_i)$. In corrupted samples we add Gaussian noise $N(Z,0.5)$ to $r_i$ and we replace $s_i'$ with a random state $s'$ with probability exponentially decreasing with the distance between $s'$ and $s_i'$ on the grid. We study the convergence of the robust repeated retraining based on OFTRL (Algorithm \ref{alg:rob_rr_prl}). As a baseline approach, we consider a version of repeated retraining based on OFTRL which uses a naive estimator of $g_d$ instead of Algorithm \ref{alg:rob_mean}.
%
%To solve \eqref{eq.perf_rl.langrangian}, we apply robust OFTRL and as robust gradient estimators we use simple averaging for the $h-$gradients and algorithm \ref{alg:rob_mean} for the $d-$gradients. 
%
In all the experiments, we set $\gamma=0.99,c_p=1$ and $\lambda=0.001$. To create transition samples, we collect $1000$ trajectories with an effective horizon of $1/(1-\gamma)=100$.

\textbf{Results.} The plots in Fig. \ref{fig:plots-perf-RL} show the convergence results for different values of the noise magnitude $Z$ and the corruption frequency $\epsilon$.   
We observe that naive gradient estimation results in more noisy convergence of repeated retraining, compared to robust gradient estimation. The effect is stronger when we fix $\epsilon$ and progressively increase $Z$. Then the curve of repeated retraining with a naive estimator oscillates with magnitude scaling with $Z$, while the curve of repeated retraining with a robust estimator stays virtually unaffected. Fixing $Z$ and progressively increasing $\epsilon$ we see that both robust and naive retraining are affected, with the error increasing with $\epsilon$. All these effects are in agreement with our theoretical results and they showcase the utility of robust gradient estimation.

%% file: 7_conclusion.tex
% !TEX root =  main.tex
%%%%%%%%%%%%%%%%%%%%%%%%%%%%%%%%%%%%%
%%%%%%%%%%%%%%%%%%%%%%%%%%%%%%%%%%%%%
% !TEX root =  main.tex
%%%%%%%%%%%%%%%%%%%%%%%%%%%%%%%%%%%%%
%%%%%%%%%%%%%%%%%%%%%%%%%%%%%%%%%%%%%
\section{Conclusion}\label{sec.conclusion}

We considered performative reinforcement learning under corrupted data. We introduced a repeated retraining approach and showed that it converges to an approximately stable policy, where the approximation error depends on the level of corruption. One of the most interesting future research directions is to investigate the tightness of the approximation error. Extending this work to RL settings with function approximation is another important avenue for future research.

%% file: 8_acknowledgement.tex
% !TEX root =  main.tex
%%%%%%%%%%%%%%%%%%%%%%%%%%%%%%%%%%%%%
%%%%%%%%%%%%%%%%%%%%%%%%%%%%%%%%%%%%%
\subsection*{Acknowledgements}\label{sec.ack}
This work has been partially supported by project MIS 5154714 of the National Recovery and Resilience Plan Greece 2.0 funded by the European Union under the NextGenerationEU Program. This research was, in part, funded by the Deutsche Forschungsgemeinschaft (DFG, German Research Foundation) – project number $467367360$.

% ......~\\
% ......~\\
% ......~\\
% ......~\\
% ......~\\
% ......~\\
% ......~\\
% ......~\\
% ......~\\
% ......~\\
% ......~\\
% ......~\\
% ......~\\
% ......~\\
% ......~\\
% ......~\\
% ......~\\
% ......~\\
% ......~\\
% ......~\\
% ......~\\
% ......~\\
% ......~\\
% ......~\\
% ......~\\
% ......~\\
% ......~\\
% ......~\\
% ......~\\
% ......~\\
% ......~\\
% ......~\\
% ......~\\
% ......~\\
% ......~\\
% ......~\\
% ......~\\
% ......~\\
% ......~\\
% ......~\\
% ......~\\
% ......~\\

%% file: 10.0_appendix_main.tex
% !TEX root =  main.tex
%%%%%%%%%%%%%%%%%%%%%%%%%%%%%%%%%%%%%
%%%%%%%%%%%%%%%%%%%%%%%%%%%%%%%%%%%%%
\newpage
\onecolumn
\section{Appendix}

\subsection{List of Appendices}
\begin{itemize}
   
    \item \textbf{Proof of Theorem 1}, studying the convergence of robust OFTRL.
    \item \textbf{Special Classes of Smooth Convex-Concave Objectives}. This subsection contains the proof of Theorem 7 and Corollary 1, studying the convergence of robust OMDA on smooth convex-concave functions satisfying the \textit{metric subregularity condition}.
    \item \textbf{Proof of Theorem 2}, deriving lower bounds for the asymptotic error of minimax optimization under noisy gradients.
    \item \textbf{Unbiased Gradient Estimators for the MDP Lagrangian}, where we show that uncorrupted gradient samples created by the formula we propose are unbiased estimators of the true gradients of the Lagrangian.
    \item \textbf{Proof of Theorem 3}, bounding the error of the robust gradient estimators.
    \item \textbf{Proof of Lemma 1}.
    \item \textbf{Proof of Theorem 4}, bounding the distance of the estimated approximate occupancy measure to the optimal occupancy measure of $M_n$.
    \item \textbf{Proof of Theorem 5}, studying the convergence of repeated retraining to an approximately stable policy. 
    \item \textbf{Return Suboptimality of the Approximately Stable Policy} compared to the return of the performatively stable policy.
    \item \textbf{Implementation Details for the Experimental Analysis}.
    
\end{itemize}
\subsection{Proof of Theorem \ref{thr:aoftrl}}
A classical approach to convex-concave optimization is to fix two players and alternate between two no-regret online learning algorithms - one for each player. For smooth convex-concave functions we proposed alternating OFTRL. This approach is based on online learning with predictions (see \citep{10.5555/3495724.3496773,mohri2015accelerating,chiang2012online,yang2011regret}). The idea is to consider gradients as loss functions and use the gradient of the previous iteration as a prediction for the gradient of the current iteration. Due to smoothness, we should be able to use an appropriately small learning rate and ensure that the gradient variation is small and thus predictions are accurate. We will first analyze the online learning algorithm that each one of the \textit{min} and  \textit{max} players deploys. This algorithm is Optimistic FTRL on noisy linear losses (Algorithm \ref{alg:three}). Then, we show how alternating this online algorithm between one \textit{min} and one \textit{max} player converges to an approximate saddle point with $\mathcal{O}(1/T)$ convergence rate. The analysis of optimistic FTRL and alternating optimistic FTRL is based on \citep{orabona2019modern}. We introduce noise in the exact gradient methods found  in \citep{orabona2019modern} and study the effect of this noise in the quality of the calculated approximate saddle point.
\subsubsection{Optimistic FTRL on Noisy Linear Losses}
Suppose there is a hidden sequence of linear loss vectors $g_1, g_2, ... g_T$. At each timestep $t$ we receive a  noisy estimation $\hat{g}_{t}$ of  $g_{t}$. Based on the history of these estimates we apply a version of optimistic FTRL as follows:\\
\begin{algorithm}%[H]
\begin{algorithmic}[1]
\caption{Optimistic FTRL on noisy linear losses}\label{alg:three}
\STATE Initialize regulariser $\psi$
% \KwResult{$y = x^n$}
\STATE $g_{0}\gets0$
\FOR{$t=1,...T$}
  \STATE Produce prediction $\tilde{g}_t$ of next loss 
  \STATE Play $x_t=argmin_{x\in\mathcal{X}}\ \  \psi(x)+\langle \tilde{g}_t,x\rangle+\sum_{i=1}^{t-1}\langle \hat{g}_i,x\rangle$ 
  
 \STATE  Receive noisy estimation $\hat{g}_{t}$ of  $g_{t}$

 \ENDFOR
 \STATE $\bar{x} \gets \frac{1}{T}\sum\limits_{t=1}^T x_t$ 

 \STATE \textbf{Return} $\bar{x}$ 
 \end{algorithmic}
\end{algorithm} 

\noindent Let $\ell_t(x)=\langle g_t,x\rangle$, $\hat{\ell}_t(x)=\langle \hat{g}_t,x\rangle$ 
and $\tilde{\ell}_t(x)=\langle \tilde{g}_t,x\rangle$. Also let $F_t(x)=\psi(x)+\sum_{i=1}^{t-1}\langle g_i,x\rangle$ and  $\hat{F}_t(x)=\psi(x)+\sum_{i=1}^{t-1}\langle \hat{g}_i,x\rangle$. \\ \\
Moreover let $\zeta_t=g_t-\hat{g}_t$ be the noise in the estimation of loss vectors. \\ \\

\begin{theorem} \label{thr:oftrl}
The output $\bar{x}$ of Algorithm \ref{alg:three} satisfies for all $ u\in\mathcal{X}$  :
\[
\sum_{t=1}^T \ell_t(x_t)-\sum_{t=1}^T \ell_t(u) \leq \psi(u)-\psi(x_1)+  \sum_{t=1}^T \left\langle \hat{g}_t-\tilde{g}_t, x_t-x_{t+1}\right\rangle-\frac{\lambda}{2}\|x_t-x_{t+1}\|^2+2D\sum_{t=1}^T \|\zeta_t\|
\]
where  $D=\max\limits_{x\in\mathcal{X}} \|x\|$ and
$\lambda$ is the strong convexity constant of regulariser $\psi$.
\end{theorem}
\begin{proof}
Although we apply optimistic FTRL based on the observations $\hat{g}_t$, the regret we wish to minimize is with respect to the true loss vectors $g_t$:
\[Regret(u)=\sum_{t=1}^T \ell_t(x_t)-\sum_{t=1}^T \ell_t(u)\]
We can also define the regret we obtain with respect to the loss estimations $\hat{g}_t$:
\[\hat{Regret}(u)=\sum_{t=1}^T \hat{\ell_t}(x_t)-\sum_{t=1}^T \hat{\ell}_t(u)\]
The two regrets are connected in the following way:
\begin{align*}
Regret(u) &=\sum_{t=1}^T \ell_t(x_t)-\sum_{t=1}^T \ell_t(u) \\
 &=\sum_{t=1}^T \langle g_t,x_t\rangle-\sum_{t=1}^T \langle g_t,u\rangle \\
  &=\sum_{t=1}^T \langle\zeta_t+\hat{g}_t,x_t\rangle-\sum_{t=1}^T \langle \zeta_t+\hat{g}_t,u\rangle \\
  &=\hat{Regret}(u)+\sum_{t=1}^T \langle\zeta_t,x_t\rangle-\sum_{t=1}^T \langle \zeta_t,u\rangle \\
\end{align*}
This way we can derive an upper bound for $Regret(u)$:
\[Regret(u)\leq \hat{Regret}(u)+\sum_{t=1}^T \|\zeta_t\|\|x_t\|+\sum_{t=1}^T \|\zeta_t\|\|u\| \leq \hat{Regret}(u)+2D\sum_{t=1}^T \|\zeta_t\|\]
where $D$ is the radius of domain $\mathcal{X}$. \\ \\ 
Next we will try to bound $\hat{Regret}(u)$. We can interpret the Optimistic-FTRL as FTRL with a regularizer $\tilde{\psi}_t(x)=\psi_t(x)+\tilde{\ell}_t(x)$. Also, note that $\tilde{\ell}_{T+1}(x)$ has no influence on the algorithm, so we can set it to the null function.
Hence, from the equality for FTRL (see \citep{orabona2019modern}), we immediately get
$$
\begin{aligned}
& \sum_{t=1}^T \hat{\ell_t}(x_t)-\sum_{t=1}^T \hat{\ell}_t(u) \\
& \quad \leq \tilde{\ell}_{T+1}(u)+\psi(u)-\min _{x \in V}(\tilde{\ell}_1(x)+\psi(x))+\sum_{t=1}^T\left[\hat{F}_t(x_t)-\hat{F}_{t+1}(x_{t+1})+\hat{\ell}_t(x_t)+\tilde{\ell}_t(x_t)-\tilde{\ell}_{t+1}(x_{t+1})\right] \\
& \quad=\psi(u)-\psi(x_1)+\sum_{t=1}^T\left[\hat{F}_t(x_t)-\hat{F}_{t+1}(x_{t+1})+\hat{\ell}_t(x_t)\right] .
\end{aligned}
$$
Next we focus on the terms $\hat{F}_t(x_t)-\hat{F}_{t+1}(x_{t+1})+\hat{\ell}_t(x_t)$. Observe that $\hat{F}_t(x)+\hat{\ell}_t(x)+i_V(x)$ is $\lambda$-strongly convex w.r.t. $\|\cdot\|$, hence we have
$$
\begin{aligned}
\hat{F}_t(x_t)-\hat{F}_{t+1}(x_{t+1})+\hat{\ell}_t(x_t) & =(\hat{F}_t(x_t)+\hat{\ell}_t(x_t))-(\hat{F}_t(x_{t+1})+\hat{\ell}_t(x_{t+1})) \\
& \leq\left\langle g_t^{\prime}, x_t-x_{t+1}\right\rangle-\frac{\lambda}{2}\|x_t-x_{t+1}\|^2,
\end{aligned}
$$
where $g_t^{\prime} \in \partial(\hat{F}_t(x_t)+\hat{\ell}_t(x_t)+i_V(x_t))$. Observing that $x_t=\operatorname{argmin}_{x \in V} \hat{F}_t(x)+\tilde{\ell}_t(x)$, we have $\mathbf{0} \in \partial(\hat{F}_t(x_t)+\tilde{\ell}_t(x_t)+i_V(x_t))$. Hence, there exists $\tilde{g}_t \in \partial \tilde{\ell}_t(x_t)$ such that $g_t^{\prime}=\hat{g}_t-\tilde{g}_{t}$ So, we have
$$
\begin{aligned}
& \hat{F}_t(x_t)-\hat{F}_{t+1}(x_{t+1})+\hat{\ell}_t(x_t) \leq\left\langle \hat{g}_t-\tilde{g}_t, x_t-x_{t+1}\right\rangle-\frac{\lambda}{2}\|x_t-x_{t+1}\|^2 .
\end{aligned}
$$
Thus we obtain:
\[
\hat{Regret}(u)=\sum_{t=1}^T \hat{\ell_t}(x_t)-\sum_{t=1}^T \hat{\ell}_t(u) \leq \psi(u)-\psi(x_1)+  \sum_{t=1}^T \left\langle \hat{g}_t-\tilde{g}_t, x_t-x_{t+1}\right\rangle-\frac{\lambda}{2}\|x_t-x_{t+1}\|^2
\]
\end{proof}
\subsubsection{Proof of Theorem \ref{thr:aoftrl}\\}
\textbf{Statement.} The output $(\bar{x},\bar{y})$ of Algorithm \ref{alg:aoftrl} satisfies for all $ x\in\mathcal{X}$ and $ y\in\mathcal{Y}$  : 
\begin{align*}
    &f\left(\bar{x}, y\right) -f\left(x, \bar{y}\right) \leq 
    \frac{\psi_X (x)+\psi_Y (y)}{T}+ \frac{\| \nabla_x f(x_1,y_1)  \|^2}{\lambda_X \cdot T} +\frac{\|\nabla_y f(x_1,y_1)\|^2}{\lambda_Y \cdot T} +\frac{6D_X}{T}\sum_{t=1}^{T} \|\zeta_{t}^X\| + \frac{6D_Y}{T}\sum_{t=1}^{T} \|\zeta_{t}^Y\|.
\end{align*}

\begin{proof}
First we decompose the duality gap into the regrets of the two players.
\begin{align*}
f(y,\bar{x})-f(\bar{y},x)&\leq\frac{1}{T}\sum\limits_{t=1}^T f(y,x_t)-\frac{1}{T}\sum\limits_{t=1}^T f(y_t,x) \Leftrightarrow\\
f(y,\bar{x})-f(\bar{y},x)&\leq\frac{1}{T}\sum\limits_{t=1}^T f(y,x_t)-\frac{1}{T}\sum\limits_{t=1}^T f(y_t,x_t)+\frac{1}{T}\sum\limits_{t=1}^T f(y_t,x_t)-\frac{1}{T}\sum\limits_{t=1}^T f(y_t,x) \Rightarrow\\
f(y,\bar{x})-f(\bar{y},x)&\leq\frac{1}{T}\left[\sum\limits_{t=1}^T \langle \nabla_y f(x_t,y_t), y \rangle-\sum\limits_{t=1}^T \langle \nabla_y f(x_t,y_t), y_t \rangle\right]+\frac{1}{T}\left[\sum\limits_{t=1}^T \langle \nabla_x f(x_t,y_t), x_t \rangle-\sum\limits_{t=1}^T \langle \nabla_x f(x_t,y_t), x \rangle\right]\\
& = \frac{1}{T}Regret_X(x)+\frac{1}{T}Regret_Y(y)
\end{align*}

\noindent Each of the players X and Y play Optimistic FTRL on noisy linear losses $\nabla_x f(x_t,y_t)$ and $-\nabla_y f(x_t,y_t)$ respectively .
Using our result for Optimistic FTRL on noisy linear losses we have:
\begin{align*}
f(y,\bar{x})-f(\bar{y},x)&\leq \frac{1}{T}\hat{Regret}_X(x)+\frac{1}{T}\hat{Regret}_Y(y)+\frac{2D_X}{T}\sum_{t=1}^T\|\zeta_t^X\|+\frac{2D_Y}{T}\sum_{t=1}^T\|\zeta_t^Y\|
\end{align*}
where
\[\hat{Regret}_X(x)\leq \psi_X(x)+  \sum_{t=1}^T \left\langle g_{X,t}-g_{X,t-1}, x_t-x_{t+1}\right\rangle-\frac{\lambda_X}{2}\|x_t-x_{t+1}\|^2\]
and
\[\hat{Regret}_Y(y)\leq \psi_Y(y)+  \sum_{t=1}^T \left\langle g_{Y,t}-g_{Y,t-1}, y_t-y_{t+1}\right\rangle-\frac{\lambda_Y}{2}\|y_t-y_{t+1}\|^2\]
For $\hat{Regret}_X(x)$ we have \begin{align*}
\hat{Regret}_X(x)&\leq \psi_X(x)+  \sum_{t=1}^T \left\langle g_{X,t}-g_{X,t-1}, x_t-x_{t+1}\right\rangle-\frac{\lambda_X}{2}\|x_t-x_{t+1}\|^2 \\
 &= \psi_X(x)+  \sum_{t=1}^T \left\langle \nabla_x f(x_t,y_t)-\nabla_x f(x_{t-1},y_{t-1}), x_t-x_{t+1}\right\rangle-\frac{\lambda_X}{2}\|x_t-x_{t+1}\|^2 + \\
 &\quad\sum_{t=1}^T \left\langle \zeta^X_t-\zeta^X_{t-1}, x_t-x_{t+1}\right\rangle\\
 &\leq \psi_X(x)+  \sum_{t=1}^T \left\langle \nabla_x f(x_t,y_t)-\nabla_x f(x_{t-1},y_{t-1}), x_t-x_{t+1}\right\rangle-\frac{\lambda_X}{2}\|x_t-x_{t+1}\|^2 + 4D_X\sum_{t=1}^T \|\zeta^X_t \|\\
\end{align*}

\noindent From the Fenchel-Young inequality, we have:
$$\langle \nabla_x f(x_t,y_t)-\nabla_x f(x_{t-1},y_{t-1}), x_t-x_{t+1}\rangle \leq \frac{\lambda_X }{4}\|x_t-x_{t+1}\| ^2+\frac{1}{\lambda_X }\|\nabla_x f(x_t,y_t)-\nabla_x f(x_{t-1},y_{t-1}\|^2$$ 
Thus we have $\forall x \in \mathcal{X}:$
$$
\hat{Regret}_X(x) \leq \psi_X (x)+\sum_{t=1}^T\left(\frac{1}{\lambda_X }\|\nabla_x f(x_t,y_t)-\nabla_x f(x_{t-1},y_{t-1}\|^2-\frac{\lambda_X }{4}\|x_t-x_{t+1}\| ^2\right) + 4D_X\sum_{t=1}^T \|\zeta^X_t \|
$$
For $t\geq 2$ we have:
$$
\begin{aligned}
\|\nabla_x f(x_t,y_t)-\nabla_x f(x_{t-1},y_{t-1}\|^2 & \leq (\| \nabla_x  f(x_t, y_t)- \nabla_x  f(x_{t-1}, y_t)\| +\| \nabla_x f(x_{t-1}, y_t)- \nabla_x  f(x_{t-1}, y_{t-1})\| )^2 \\
& \leq 2 L_{X X}^2\|x_{t-1}-x_t\| ^2+2 L_{X Y}^2\|y_{t-1}-y_t\| ^2 \\
\end{aligned}
$$
We can proceed in the exact same way for the $Y$-player too.
Summing the regret of the two algorithms, we have:
$$
\begin{aligned}
& \hat{Regret}_X(x)+\hat{Regret}_Y(y) \leq \psi_X (x)+\psi_Y (y)+\frac{\| \nabla_x f(x_1,y_1)  \|^2}{\lambda_X }+\frac{\|\nabla_y f(x_1,y_1)\|^2}{\lambda_Y } \\
& \quad+\sum_{t=2}^T\left(\left(\frac{2 L_{X X}^2}{\lambda_X }+\frac{2 L_{X Y}^2}{\lambda_Y }-\frac{\lambda_X }{4}\right)\|x_t-x_{t-1}\| ^2+\left(\frac{2 L_{Y Y}^2}{\lambda_Y }+\frac{2 L_{X Y}^2}{\lambda_X }-\frac{\lambda_Y }{4}\right)\|y_t-y_{t-1}\| ^2\right) \\
& \quad+4D_X\sum_{t=1}^T \|\zeta^X_t \|+4D_Y\sum_{t=1}^T \|\zeta^Y_t \|
\end{aligned}
$$
Choosing $\lambda_X  \geq 2 \sqrt{2}\left(L_{X X}+L_{X Y} \alpha\right)$ and $\lambda_  Y  \geq 2 \sqrt{2}\left(L_{Y Y}+L_{X Y} / \alpha\right)$ for any $\alpha>0$ kills all the terms in the sum in the second line. In fact, we have:
$$
\frac{2 L_{X X}^2}{\lambda_X }+\frac{2 L_{X Y}^2}{\lambda_Y } \leq \frac{2 L_{X X}^2}{2 \sqrt{2} L_{X X}}+\frac{2 L_{X Y}^2 \alpha}{2 \sqrt{2} L_{X Y}} \leq \frac{\lambda_X }{4}
$$
and similarly for the other term. \\ \\
Thus we have:
$$
\begin{aligned}
\hat{Regret}_X(x)+\hat{Regret}_Y(y) &\leq \psi_X (x)+\psi_Y (y)+\frac{\| \nabla_x f(x_1,y_1)  \|^2}{\lambda_X }+\frac{\|\nabla_y f(x_1,y_1)\|^2}{\lambda_Y } \\ & \quad +4D_X\sum_{t=1}^T \|\zeta^X_t \|+4D_Y\sum_{t=1}^T \|\zeta^Y_t \|\\
\end{aligned}
$$ \\
Finally, for the duality gap we have that for all $x\in\mathcal{X}$ and $y\in\mathcal{Y}$:

$\begin{aligned}
 f\left(\bar{x}, y\right)-f\left(x, \bar{y}\right) \leq & \frac{\psi_X (x)+\psi_Y (y)+\frac{\| \nabla_x f(x_1,y_1)  \|^2}{\lambda_X }+\frac{\|\nabla_y f(x_1,y_1)\|^2}{\lambda_Y }}{T} 
+ \frac{6D_X}{T}\sum_{t=1}^{T} \|\zeta_{t}^X\| + \frac{6D_Y}{T}\sum_{t=1}^{T} \|\zeta_{t}^Y\| \\
\end{aligned}
$
\end{proof}
\subsection{Special Classes of Smooth Convex-Concave Objectives}

In this subsection, we consider classes of smooth convex-concave functions $f$ that satisfy a condition called Saddle Point Metric Subregularity (SP-MS)~\citep{wei2021linear}. For such a function $f$ let $z_t = (x_t, y_t)$, $\mathcal{Z}=\mathcal{X}\times\mathcal{Y}$, $\mathcal{Z^*}$ be the set of saddle points of $f$ and $F(z_t)=\left(\nabla_x f(x_t,y_t),-\nabla_y f(x_t,y_t)\right)$. Moreover we use $\Pi_{\mathcal{A}}(x)$ to denote the projection of a vector $x$ on a set $\mathcal{A}$. Then the SP-MS condition is defined as follows. 

\begin{definition}\label{spms} 
 Given constants $\beta\geq0$ and $C>0$, function $F$ over a domain $\mathcal{Z}$ satisfies the {\em Saddle-Point Metric Subregularity} condition iff for all $z \in \mathcal{Z}\backslash\mathcal{Z^*}$: 
$$\sup_{z^{\prime} \in \mathcal{Z}} \frac{F(z)^{\top}\left(z-z^{\prime}\right)}{\left\|z-z^{\prime}\right\|}\geq C\left\|z-\Pi_{\mathcal{Z}^{*}}\left(z\right)\right\|^{\beta+1},
$$
%and some constant parameters $\beta\geq0$ and $C>0$.

\end{definition}
The SP-MS condition with $\beta=0$ captures two well known settings: bilinear games over polytope constraints and smooth strongly convex - strongly concave functions %(for proof 
(see Theorems 5 and 6 in \citep{wei2021linear}). 
In the context of performative RL this condition is relevant because one may use regularization for both variables $d$ and $h$ of the  Lagrangian \eqref{eq.perf_rl.langrangian}. In this case, the regularized Lagrangian is a smooth strongly convex - strongly concave function and the SP-MS condition holds with $\beta=0$. It may also hold that the Lagrangian as defined in equation \eqref{eq.perf_rl.langrangian} already satisfies SP-MS with $\beta=0$ due to its regularized bilinear structure, however in this work we leave this as an open question.

To achieve better convergence guarantees under the SP-MP condition, we consider a robust version of Optimistic Mirror Descent Ascent (OMDA) algorithm, defined in Algorithm \ref{alg:omda}. This algorithm assumes that $\|F(z)-F(z')\|_{q}^{2}\leq L^2 \|z-z'\|_{p}^{2}$ for $q\geq 1$ with $\frac{1}{p}+\frac{1}{q}=1$, where $L$ is a given constant. The learning rate $\eta$ depends on constant $L$, and should satisfy $0< \eta <\frac{1}{8L}$. Here we focus on the Euclidean norm case, where $p=q=2$ and we use $\psi(u)=\frac{1}{2}\|u\|^2$.
% as the regularizer. 

\begin{algorithm}%[H]
\begin{algorithmic}[1]
%\caption{Robust alternating Optimistic FTRL}\label{alg:aoftrl}
\caption{Robust OMDA}\label{alg:omda}
\STATE Initialize $0 < \eta < \frac{1}{8 L}$ 
\STATE $(z_{0},\widehat z_{1})\gets(\textbf{0},\textbf{0})$\\
\FOR{$t=1,...T$}
  \STATE $z_t \gets \argmin_{z \in \mathcal{Z}}\{\eta\langle z,\tilde{F}(z_{t-1})\rangle+ D_{\psi}(z,\widehat{z}_{t})\}$ 
  \STATE $\widehat{z}_{t+1} \gets \argmin_{z \in \mathcal{Z}}\{\eta\langle z,\tilde{F}(z_{t})\rangle+ D_{\psi}(z,\widehat{z}_{t})\}$ 
  \STATE Calculate robust estimation $\tilde{F}(z_t)$ of $F(z_t)$
 \ENDFOR
 \STATE \textbf{Return} $z_T$ 
 \end{algorithmic}
\end{algorithm} 

% Note that unlike Robust GDA and Robust OFTRL, robust OMDA outputs the last iterate of the algorithm. 
The errors of the robust gradients are denoted as before: $$\tilde{F}(z_t)=\left(\nabla_x f(x_t,y_t)+\zeta_{t}^X,-\nabla_y f(x_t,y_t)-\zeta_{t}^Y\right)$$ Moreover let $\operatorname{dist}\left(\widehat{z}_{t}, \mathcal{Z}^{*}\right) $ denote the distance between $\widehat{z}_{t}$ and the set $\mathcal{Z}^{*}$ of saddle points of $f$. Then, the following guarantee holds.
%the following holds:
\begin{theorem}\label{thrm8usc}
Robust OMDA satisfies the following last iterate convergence guarantee for functions satisfying SP-MS with $\beta=0$:
\begin{align*}
\operatorname{dist}^{2}\left(\widehat{z}_{t}, \mathcal{Z}^{*}\right) \leq &\operatorname{dist}^{2}\left(\widehat{z}_{1}, \mathcal{Z}^{*}\right)\left(\frac{1}{1+C^{\prime}}\right)^{t-1} + \\
    &\sum\limits_{i=1}^{t-1} \left(\frac{1}{1+C^{\prime}}\right)^{t-i}(\epsilon_i+\delta_i)
\end{align*}
where $C'=\min \left\{\frac{\eta^{2} C^{2}}{8}, \frac{1}{2}\right\}$, 
$\epsilon_t=\frac{3}{8}\eta^2\left(\|\zeta_{t}^X\|^2+\|\zeta_{t}^Y\|^2\right)$ and
$\delta_t=3\eta(D_X+D_Y)(\|\zeta_{t-1}^X\|+\|\zeta_{t}^X\|+\|\zeta_{t-1}^Y\|+\|\zeta_{t}^Y\|)$.
%$\delta_t=3\eta(D_X+D_Y)\left(\|\zeta_{t-1}^X\|+\|\zeta_{t}^X\|+\|\zeta_{t-1}^Y\|+\|\zeta_{t}^Y\|\right)$
\end{theorem}
% $T=O(log(1/\epsilon)/log(1 + C'^2))=O(log(1/\epsilon)/C'^2)=\tilde{O}(k^2)$, where $k$ is the condition number. The optimal number of iterations is $O(k)$.\\ \\
We observe an interesting behaviour at the tuning of the learning rate. Decreasing the learning rate we face a trade-off between convergence speed and last iterate accuracy. If the gradient corruption is upper bounded by some universal constant, i.e., if $\|\zeta_{t}^X\|\leq \zeta_X$ and $\|\zeta_{t}^Y\|\leq \zeta_Y$ for all $t \in[T]$, the following result holds. 
%
%
%\[\|\zeta_{t}^X\|\leq \zeta_X, \quad \|\zeta_{t}^Y\|\leq \zeta_Y \quad t\in [T]\]
%Then we can derive the following Corollary
%
\begin{corollary}\label{cor1}
At the $t$-th iteration of Robust OMDA we have the following guarantee:
\begin{align*}
\operatorname{dist}^{2}\left(\widehat{z}_{t}, \mathcal{Z}^{*}\right) \leq  (\operatorname{dist}^{2}\left(\widehat{z}_{1}, \mathcal{Z}^{*}\right) -d_{\infty})\left(\frac{1}{1+C^{\prime}}\right)^{t-1} + d_{\infty}
\end{align*}
where $d_{\infty}=\max\left\{\frac{8}{\eta^2C^2}, 2 \right\} (Z\eta+Z^2\eta^2)$, 
$C'=\min \left\{\frac{\eta^{2} C^{2}}{8}, \frac{1}{2}\right\}$ and $Z=3(\zeta_X+\zeta_Y)\max\{D_X+D_Y,1\}$.
\end{corollary}
Based on the above we have two options for $\eta$ (see the next subsection for details): 
%\begin{corollary}\label{cor2}
%Regarding the upper bound in Corollary \ref{cor1}:
\begin{itemize}
    \item If we want to maximize the convergence speed (or equivalently maximize $C^{\prime}$) we set $\eta=\frac{1}{8L}$.
    \item If we want to minimize the asymptotic error $d_{\infty}$ we set $\eta=\min\{\frac{1}{8L},\frac{2}{C}\}$.
\end{itemize}
%\end{corollary}
In both cases the algorithm converges exponentially fast towards a point that has distance $\mathcal{O}((\zeta_X+\zeta_Y)(D_X+D_Y+1))$ to the set of optimal solutions. The dependence on  $(\zeta_X+\zeta_Y)$ is optimal but perhaps the factor $(D_X+D_Y)$ could be avoided, as the lower bounds 
in Theorem \ref{thr:lbext} suggest.
%as we will see in the lower bounds section below.

\subsubsection{Choice of OMDA Learning Rate}
Assuming that $\operatorname{dist}^{2}\left(\widehat{z}_{1}, \mathcal{Z}^{*}\right) > d_{\infty}$ we see that the upper bound of $\operatorname{dist}^{2}\left(\widehat{z}_{t}, \mathcal{Z}^{*}\right)$ is decreasing in $t$. The opposite case is trivial and it implies that we should keep the initial point and run no iterations of the algorithm. Thus, we will focus on the case that the upper bound is decreasing in $t$.
\begin{itemize}
    \item In terms of convergence speed, $C'$ is increasing in $\eta$ so for fast convergence we want an $\eta$ as high as possible, satisfying the constraint that $\eta\leq\frac{1}{8L}$. Thus we set $\eta=\frac{1}{8L}$.
    \item In terms of asymptotic error, we want to minimize $d_{\infty}$.
    \begin{itemize}
        \item  If $\eta \leq \frac{2}{C}$, then $d_{\infty}=\frac{8Z}{\eta C^2}+\frac{8Z^2}{C^2}$, which is decreasing in $\eta$ so we want $\eta$ to be as high as possible satisfying the constraint that $\eta\leq\frac{1}{8L}$ (and $\eta \leq \frac{2}{C}$). Thus we set $\eta=\min\{\frac{1}{8L},\frac{2}{C}\}$.
        \item If $\eta\geq\frac{2}{C}$ (and $\frac{2}{C}\leq\frac{1}{8L}$), then $d_{\infty}=2 (Z\eta+Z^2\eta^2)$. This is increasing in $\eta$, so we want $\eta$ to be as low as possible, satisfying $\eta\geq\frac{2}{C}$ and $\eta\leq\frac{1}{8L}$. Thus we set $\eta=\frac{2}{C}=\min\{\frac{1}{8L},\frac{2}{C}\}$.
    \end{itemize}
    As we see, the general rule to minimize $d_{\infty}$ is to set $\eta=\min\{\frac{1}{8L},\frac{2}{C}\}$. 
\end{itemize}

\subsubsection{Proof of Theorem \ref{thrm8usc}}
To prove Theorem \ref{thrm8usc} we follow the same steps as \cite{wei2021linear}, with the difference that we use noisy gradients instead of exact ones. We modify two auxiliary lemmas to account for the gradient noise. Lemmas \ref{lemmausc1} and \ref{lemmausc4} generalize Lemmas 1 and 4 respectively in  \citep{wei2021linear}. Then, using these two lemmas, we derive the result stated in Theorem \ref{thrm8usc}.\\\\
\textbf{Statement.} Robust OMDA satisfies the following last iterate convergence guarantee for functions satisfying SP-MS with $\beta=0$:
\begin{align*}
\operatorname{dist}^{2}\left(\widehat{z}_{t}, \mathcal{Z}^{*}\right) \leq &\operatorname{dist}^{2}\left(\widehat{z}_{1}, \mathcal{Z}^{*}\right)\left(\frac{1}{1+C^{\prime}}\right)^{t-1} + \sum\limits_{i=1}^{t-1} \left(\frac{1}{1+C^{\prime}}\right)^{t-i}(\epsilon_i+\delta_i)
\end{align*}
where $C'=\min \left\{\frac{\eta^{2} C^{2}}{8}, \frac{1}{2}\right\}$, 
$\epsilon_t=\frac{3}{8}\eta^2\left(\|\zeta_{t}^X\|^2+\|\zeta_{t}^Y\|^2\right)$ and
$\delta_t=3\eta(D_X+D_Y)(\|\zeta_{t-1}^X\|+\|\zeta_{t}^X\|+\|\zeta_{t-1}^Y\|+\|\zeta_{t}^Y\|)$.

\begin{lemma}\label{lemmausc1}
The iterates of OMDA satisfy:
    \[\eta F\left(z_{t}\right)^{\top}\left(z_{t}-z\right) 
 \leq D_{\psi}\left(z, \widehat{z}_{t}\right)-D_{\psi}\left(z, \widehat{z}_{t+1}\right)-D_{\psi}\left(\widehat{z}_{t+1}, z_{t}\right)-\frac{15}{16} D_{\psi}\left(z_{t}, \widehat{z}_{t}\right)+\frac{1}{16} D_{\psi}\left(\widehat{z}_{t}, z_{t-1}\right)+\delta_t\]\\
where $\delta_t=3\eta(D_X+D_Y)\left(\|\zeta_{t-1}^X\|_q+\|\zeta_{t}^X\|_q+\|\zeta_{t-1}^Y\|_q+\|\zeta_{t}^Y\|_q\right)$
\end{lemma}
\begin{proof}
To prove the above lemma we use the following two lemmas.

\begin{lemma}\label{lemmausc10}
    Let $\mathcal{A}$ be a convex set and $u^{\prime}=\operatorname{argmin}_{u^{\prime} \in \mathcal{A}}\left\{\left\langle u^{\prime}, g\right\rangle+D_{\psi}\left(u^{\prime}, u\right)\right\}$. Then for any $u^{*} \in \mathcal{A}$

$$
\left\langle u^{\prime}-u^{*}, g\right\rangle \leq D_{\psi}\left(u^{*}, u\right)-D_{\psi}\left(u^{*}, u^{\prime}\right)-D_{\psi}\left(u^{\prime}, u\right) .
$$
\end{lemma}

% Proof. Since $D_{\psi}\left(u^{\prime}, u\right)=\psi\left(u^{\prime}\right)-\psi(u)-\left\langle\nabla \psi(u), u^{\prime}-u\right\rangle$, by the first-order optimality condition of $u^{\prime}$, we have

% $$
% \left(g+\nabla \psi\left(u^{\prime}\right)-\nabla \psi(u)\right)^{\top}\left(u^{*}-u^{\prime}\right) \geq 0
% $$

% On the other hand, notice that the right-hand side of Eq. (10) is

% $$
% \begin{aligned}
% & \psi\left(u^{*}\right)-\psi(u)-\left\langle\nabla \psi(u), u^{*}-u\right\rangle \\
% & \quad-\psi\left(u^{*}\right)+\psi\left(u^{\prime}\right)+\left\langle\nabla \psi\left(u^{\prime}\right), u^{*}-u^{\prime}\right\rangle \\
% & \quad-\psi\left(u^{\prime}\right)+\psi(u)+\left\langle\nabla \psi(u), u^{\prime}-u\right\rangle \\
% & =\left\langle\nabla \psi\left(u^{\prime}\right)-\nabla \psi(u), u^{*}-u^{\prime}\right\rangle .
% \end{aligned}
% $$

% Therefore, Eq. (10) is equivalent to $\left\langle g+\nabla \psi\left(u^{\prime}\right)-\nabla \psi(u), u^{*}-u^{\prime}\right\rangle \geq 0$, which we have already shown above.

\begin{lemma}\label{lemmausc11}
    Suppose that $\psi$ satisfies $D_{\psi}\left(x, x^{\prime}\right) \geq \frac{1}{2}\left\|x-x^{\prime}\right\|_{p}^{2}$ for some $p \geq 1$, and let $u, u_{1}, u_{2} \in$ $\mathcal{A}$ ( a convex set) be related by the following:

$$
\begin{aligned}
& u_{1}=\underset{u^{\prime} \in \mathcal{A}}{\operatorname{argmin}}\left\{\left\langle u^{\prime}, g_{1}\right\rangle+D_{\psi}\left(u^{\prime}, u\right)\right\}, \\
& u_{2}=\underset{u^{\prime} \in \mathcal{A}}{\operatorname{argmin}}\left\{\left\langle u^{\prime}, g_{2}\right\rangle+D_{\psi}\left(u^{\prime}, u\right)\right\} .
\end{aligned}
$$

Then we have

$$
\left\|u_{1}-u_{2}\right\|_{p} \leq\left\|g_{1}-g_{2}\right\|_{q}
$$

where $q \geq 1$ and $\frac{1}{p}+\frac{1}{q}=1$

\end{lemma}

% Proof. By the first-order optimality conditions of $u_{1}$ and $u_{2}$, we have

% $$
% \begin{aligned}
% & \left\langle\nabla \psi\left(u_{1}\right)-\nabla \psi(u)+g_{1}, u_{2}-u_{1}\right\rangle \geq 0, \\
% & \left\langle\nabla \psi\left(u_{2}\right)-\nabla \psi(u)+g_{2}, u_{1}-u_{2}\right\rangle \geq 0 .
% \end{aligned}
% $$

% Summing them up and rearranging the terms, we get

% $$
% \left\langle u_{2}-u_{1}, g_{1}-g_{2}\right\rangle \geq\left\langle\nabla \psi\left(u_{1}\right)-\nabla \psi\left(u_{2}\right), u_{1}-u_{2}\right\rangle .
% $$

% By the condition on $\psi$, we have $\left\langle\nabla \psi\left(u_{1}\right), u_{1}-u_{2}\right\rangle \geq \psi\left(u_{1}\right)-\psi\left(u_{2}\right)+\frac{1}{2}\left\|u_{1}-u_{2}\right\|_{p}^{2}$ and $\left\langle\nabla \psi\left(u_{2}\right), u_{2}-u_{1}\right\rangle \geq \psi\left(u_{2}\right)-\psi\left(u_{1}\right)+\frac{1}{2}\left\|u_{1}-u_{2}\right\|_{p}^{2}$. Summing them up we get $\left\langle\nabla \psi\left(u_{1}\right)-\right.$ $\left.\nabla \psi\left(u_{2}\right), u_{1}-u_{2}\right\rangle \geq\left\|u_{1}-u_{2}\right\|_{p}^{2}$. Combining this with Eq. (11) we get
% $$
% \left\langle u_{2}-u_{1}, g_{1}-g_{2}\right\rangle \geq\left\|u_{1}-u_{2}\right\|_{p}^{2}
% $$
% Since $\left\langle u_{2}-u_{1}, g_{1}-g_{2}\right\rangle \leq\left\|u_{1}-u_{2}\right\|_{p}\left\|g_{1}-g_{2}\right\|_{q}$ by Hölder's inequality, we further get $\| u_{1}-$ $u_{2}\left\|_{p} \leq\right\| g_{1}-g_{2} \|_{q}$.
\noindent The previous two lemmas are Lemma 10 and 11 in \cite{wei2021linear} and their proof can be found there.\\ \\
We procceed in the proof of Lemma \ref{lemmausc1}. \\ \\Considering the formula of OMDA and using Lemma \ref{lemmausc10} with $u=\widehat{z}_{t}, u^{\prime}=\widehat{z}_{t+1}, u^{*}=z$, and $g=\eta \tilde{F}\left(z_{t}\right)$, we get
\begin{align*}
 &\eta\tilde{F}\left(z_{t}\right)^{\top}\left(\widehat{z}_{t+1}-z\right) \leq D_{\psi}\left(z, \widehat{z}_{t}\right)-D_{\psi}\left(z, \widehat{z}_{t+1}\right)-D_{\psi}\left(\widehat{z}_{t+1}, \widehat{z}_{t}\right) \Leftrightarrow \\
 &\eta{F}\left(z_{t}\right)^{\top}\left(\widehat{z}_{t+1}-z\right)+\eta\left(\zeta_{t}^X,-\zeta_{t}^Y\right)^{\top}\left((\widehat{x}_{t+1},\widehat{y}_{t+1})-(x,y)\right) \leq D_{\psi}\left(z, \widehat{z}_{t}\right)-D_{\psi}\left(z, \widehat{z}_{t+1}\right)-D_{\psi}\left(\widehat{z}_{t+1}, \widehat{z}_{t}\right) \Leftrightarrow \\
 &\eta{F}\left(z_{t}\right)^{\top}\left(\widehat{z}_{t+1}-z\right)\leq D_{\psi}\left(z, \widehat{z}_{t}\right)-D_{\psi}\left(z, \widehat{z}_{t+1}\right)-D_{\psi}\left(\widehat{z}_{t+1}, \widehat{z}_{t}\right) - \eta\zeta_{t}^X\cdot(\widehat{x}_{t+1}-x)+\eta\zeta_{t}^Y\cdot(\widehat{y}_{t+1}-y) \Rightarrow \\
&\eta{F}\left(z_{t}\right)^{\top}\left(\widehat{z}_{t+1}-z\right)\leq D_{\psi}\left(z, \widehat{z}_{t}\right)-D_{\psi}\left(z, \widehat{z}_{t+1}\right)-D_{\psi}\left(\widehat{z}_{t+1}, \widehat{z}_{t}\right) + 2\eta\|\zeta_{t}^X\|_qD_X+2\eta\|\zeta_{t}^Y\|_qD_Y
\end{align*}
Considering the formula of OMDA and using Lemma \ref{lemmausc10} with $u=\widehat{z}_{t}, u^{\prime}=z_{t}, u^{*}=\widehat{z}_{t+1}$, and $g=\eta \tilde{F} \left(z_{t-1}\right)$, we get
\begin{align*}
&\eta \tilde{F}\left(z_{t-1}\right)^{\top}\left(z_{t}-\widehat{z}_{t+1}\right) \leq D_{\psi}\left(\widehat{z}_{t+1}, \widehat{z}_{t}\right)-D_{\psi}\left(\widehat{z}_{t+1}, z_{t}\right)-D_{\psi}\left(z_{t}, \widehat{z}_{t}\right)\Leftrightarrow\\
&\eta F\left(z_{t-1}\right)^{\top}\left(z_{t}-\widehat{z}_{t+1}\right) +\eta\left(\zeta_{t-1}^X,-\zeta_{t-1}^Y\right)^{\top}\left((x_t,y_t)-(\widehat{x}_{t+1},\widehat{y}_{t+1})\right)\\
&\leq D_{\psi}\left(\widehat{z}_{t+1}, \widehat{z}_{t}\right)-D_{\psi}\left(\widehat{z}_{t+1}, z_{t}\right)-D_{\psi}\left(z_{t}, \widehat{z}_{t}\right)\Leftrightarrow\\
&\eta F\left(z_{t-1}\right)^{\top}\left(z_{t}-\widehat{z}_{t+1}\right) \\
&\leq D_{\psi}\left(\widehat{z}_{t+1}, \widehat{z}_{t}\right)-D_{\psi}\left(\widehat{z}_{t+1}, z_{t}\right)-D_{\psi}\left(z_{t}, \widehat{z}_{t}\right)-\eta\zeta_{t-1}^X\cdot(x_t-\widehat{x}_{t+1})+\eta\zeta_{t-1}^Y\cdot(y_t-\widehat{y}_{t+1})\Rightarrow\\
&\eta F\left(z_{t-1}\right)^{\top}\left(z_{t}-\widehat{z}_{t+1}\right) \leq  D_{\psi}\left(\widehat{z}_{t+1}, \widehat{z}_{t}\right)-D_{\psi}\left(\widehat{z}_{t+1}, z_{t}\right)-D_{\psi}\left(z_{t}, \widehat{z}_{t}\right)+ 2\eta\|\zeta_{t-1}^X\|_qD_X+2\eta\|\zeta_{t-1}^Y\|_qD_Y
\end{align*}
Summing up the two inequalities above, and adding $\eta\left(F\left(z_{t}\right)-F\left(z_{t-1}\right)\right)^{\top}\left(z_{t}-\widehat{z}_{t+1}\right)$ to both sides, we get \\
\begin{align}
    \notag \eta F\left(z_{t}\right)^{\top}\left(z_{t}-z\right)\leq &D_{\psi}\left(z, \widehat{z}_{t}\right)-D_{\psi}\left(z, \widehat{z}_{t+1}\right)-D_{\psi}\left(\widehat{z}_{t+1}, z_{t}\right)-D_{\psi}\left(z_{t}, \widehat{z}_{t}\right)+\eta\left(F\left(z_{t}\right)-F\left(z_{t-1}\right)\right)^{\top}\left(z_{t}-\widehat{z}_{t+1}\right)\\
    & \label{eq:equsc1} +2\eta\|\zeta_{t-1}^X\|_qD_X+2\eta\|\zeta_{t-1}^Y\|_qD_Y+ 2\eta\|\zeta_{t}^X\|_qD_X+2\eta\|\zeta_{t}^Y\|_qD_Y
\end{align}
Using Lemma \ref{lemmausc11} with $u=\widehat{x}_{t}, u_{1}=x_{t}, u_{2}=\widehat{x}_{t+1}, g_{1}=\eta \left(\nabla_{x} f(z_{t-1})+\zeta_{t-1}^X\right)$ and $g_{2}=\eta \left(\nabla_{x} f(z_{t})+\zeta_{t}^X\right)$, we get 
\begin{align*}
    &\left\|x_{t}-\widehat{x}_{t+1}\right\|_{p} \leq \eta\left\|\nabla_{x} f(z_{t-1})+\zeta_{t-1}^X-\nabla_{x} f(z_{t})-\zeta_{t}^X\right\|_{q}
    \Rightarrow\\
    &\left\|x_{t}-\widehat{x}_{t+1}\right\|_{p}\leq\eta\left\|\nabla_{x} f(z_{t-1})-\nabla_{x} f(z_{t})\right\|_{q}+\eta\|\zeta_{t-1}^X\|_q+\eta\|\zeta_{t}^X\|_q
\end{align*}
 Similarly, we have $$\left\|y_{t}-\widehat{y}_{t+1}\right\|_{p} \leq \eta \| \nabla_{y} f\left(z_{t}\right)-\nabla_{y} f\left(z_{t-1}\right) \|_{q}+\eta\|\zeta_{t-1}^Y\|_q+\eta\|\zeta_{t}^Y\|_q$$ Therefore, by Hölder's inequality, we have

$$
\begin{aligned}
& \eta\left(F\left(z_{t}\right)-F\left(z_{t-1}\right)\right)^{\top}\left(z_{t}-\widehat{z}_{t+1}\right) \\
& \leq \eta\left\|x_{t}-\widehat{x}_{t+1}\right\|_{p}\left\|\nabla_{x} f\left(z_{t-1}\right)-\nabla_{x} f\left(z_{t}\right)\right\|_{q}+\eta\left\|y_{t}-\widehat{y}_{t+1}\right\|_{p}\left\|\nabla_{y} f\left(z_{t-1}\right)-\nabla_{y} f\left(z_{t}\right)\right\|_{q} \\
& \leq \eta^{2}\left\|\nabla_{x} f\left(z_{t-1}\right)-\nabla_{x} f\left(z_{t}\right)\right\|_{q}^{2}+\eta^{2}\left\|\nabla_{y} f\left(z_{t-1}\right)-\nabla_{y} f\left(z_{t}\right)\right\|_{q}^{2}+\\
&\eta^{2}\left\|\nabla_{x} f\left(z_{t-1}\right)-\nabla_{x} f\left(z_{t}\right)\right\|_{q}\left(\|\zeta_{t-1}^X\|_q+\|\zeta_{t}^X\|_q\right)+\eta^{2}\left\|\nabla_{y} f\left(z_{t-1}\right)-\nabla_{y} f\left(z_{t}\right)\right\|_{q}\left(\|\zeta_{t-1}^Y\|_q+\|\zeta_{t}^Y\|_q\right)\\
& \leq\eta^{2} \operatorname{dist}_{q}^{2}\left(F\left(z_{t}\right), F\left(z_{t-1}\right)\right) +\eta^{2}L\left\|z_{t-1}-z_{t}\right\|_{q}\left(\|\zeta_{t-1}^X\|_q+\|\zeta_{t}^X\|_q\right)+\eta^{2}L\left\|z_{t-1}-z_{t}\right\|_{q}\left(\|\zeta_{t-1}^Y\|_q+\|\zeta_{t}^Y\|_q\right)\\
& \leq \eta^{2} L^{2} \operatorname{dist}_{p}^{2}\left(z_{t}, z_{t-1}\right) +2\eta^{2}L(D_X+D_Y)\left(\|\zeta_{t-1}^X\|_q+\|\zeta_{t}^X\|_q+\|\zeta_{t-1}^Y\|_q+\|\zeta_{t}^Y\|_q\right)\\
& \leq \frac{1}{64} \operatorname{dist}_{p}^{2}\left(z_{t}, z_{t-1}\right)+2\eta^{2}L(D_X+D_Y)\left(\|\zeta_{t-1}^X\|_q+\|\zeta_{t}^X\|_q+\|\zeta_{t-1}^Y\|_q+\|\zeta_{t}^Y\|_q\right)\
\end{aligned}
$$

Continuing from Eq. \eqref{eq:equsc1}, we then have

$$
\begin{aligned}
& \eta F\left(z_{t}\right)^{\top}\left(z_{t}-z\right) \\
& \leq D_{\psi}\left(z, \widehat{z}_{t}\right)-D_{\psi}\left(z, \widehat{z}_{t+1}\right)-D_{\psi}\left(\widehat{z}_{t+1}, z_{t}\right)-D_{\psi}\left(z_{t}, \widehat{z}_{t}\right)\\
&\quad +2\eta\|\zeta_{t-1}^X\|_qD_X+2\eta\|\zeta_{t-1}^Y\|_qD_Y+ 2\eta\|\zeta_{t}^X\|_qD_X+2\eta\|\zeta_{t}^Y\|_qD_Y\\
&\quad+\frac{1}{64} \operatorname{dist}_{p}^{2}\left(z_{t}, z_{t-1}\right)+2\eta^{2}L(D_X+D_Y)\left(\|\zeta_{t-1}^X\|_q+\|\zeta_{t}^X\|_q+\|\zeta_{t-1}^Y\|_q+\|\zeta_{t}^Y\|_q\right)\\\
& \leq D_{\psi}\left(z, \widehat{z}_{t}\right)-D_{\psi}\left(z, \widehat{z}_{t+1}\right)-D_{\psi}\left(\widehat{z}_{t+1}, z_{t}\right)-D_{\psi}\left(z_{t}, \widehat{z}_{t}\right)+\frac{1}{32} \operatorname{dist}_{p}^{2}\left(z_{t}, \widehat{z}_{t}\right)+\frac{1}{32} \operatorname{dist}_{p}^{2}\left(\widehat{z}_{t}, z_{t-1}\right) \\
& \quad\left(\|u+v\|_{p}^{2} \leq\left(\|u\|_{p}+\|v\|_{p}\right)^{2} \leq 2\|u\|_{p}^{2}+2\|v\|_{p}^{2}\right) \\
&\quad+3\eta(D_X+D_Y)\left(\|\zeta_{t-1}^X\|_q+\|\zeta_{t}^X\|_q+\|\zeta_{t-1}^Y\|_q+\|\zeta_{t}^Y\|_q\right)\\
& \leq D_{\psi}\left(z, \widehat{z}_{t}\right)-D_{\psi}\left(z, \widehat{z}_{t+1}\right)-D_{\psi}\left(\widehat{z}_{t+1}, z_{t}\right)-D_{\psi}\left(z_{t}, \widehat{z}_{t}\right)+\frac{1}{16} D_{\psi}\left(z_{t}, \widehat{z}_{t}\right)+\frac{1}{16} D_{\psi}\left(\widehat{z}_{t}, z_{t-1}\right)\\
&\quad+3\eta(D_X+D_Y)\left(\|\zeta_{t-1}^X\|_q+\|\zeta_{t}^X\|_q+\|\zeta_{t-1}^Y\|_q+\|\zeta_{t}^Y\|_q\right)\\
&=D_{\psi}\left(z, \widehat{z}_{t}\right)-D_{\psi}\left(z, \widehat{z}_{t+1}\right)-D_{\psi}\left(\widehat{z}_{t+1}, z_{t}\right)-\frac{15}{16} D_{\psi}\left(z_{t}, \widehat{z}_{t}\right)+\frac{1}{16} D_{\psi}\left(\widehat{z}_{t}, z_{t-1}\right)\\
&\quad +3\eta(D_X+D_Y)\left(\|\zeta_{t-1}^X\|_q+\|\zeta_{t}^X\|_q+\|\zeta_{t-1}^Y\|_q+\|\zeta_{t}^Y\|_q\right)
\end{aligned}
$$
This completes the proof.
\end{proof}
% \section{AN AUXILIARY LEMMA ON RECURSIVE FORMULAS}

% We continue to bound the last term as

% $$
% \begin{aligned}
% & 4\left(\frac{\left\|z^{*}-\widehat{z}_{t+1}\right\|^{2}+\left\|\widehat{z}_{t+1}-z_{t}\right\|^{2}}{\epsilon^{2}}\right) \\
% & =4\left(\frac{\left\|x^{*}-\widehat{x}_{t+1}\right\|^{2}+\left\|y^{*}-\widehat{y}_{t+1}\right\|^{2}+\left\|\widehat{x}_{t+1}-x_{t}\right\|^{2}+\left\|\widehat{y}_{t+1}-y_{t}\right\|^{2}}{\epsilon^{2}}\right) \\
% & \left.=4\left(\frac{\left\|x^{*}-\widehat{x}_{t+1}\right\|_{1}^{2}+\left\|y^{*}-\widehat{y}_{t+1}\right\|_{1}^{2}+\left\|\widehat{x}_{t+1}-x_{t}\right\|_{1}^{2}+\left\|\widehat{y}_{t+1}-y_{t}\right\|_{1}^{2}}{\epsilon^{2}}\right) \quad \quad \quad \quad \quad\|x\|_{2} \leq\|x\|_{1}\right) \\
% & \leq \frac{128}{\epsilon^{2}}\left(\frac{\operatorname{KL}\left(z^{*}, \widehat{z}_{t+1}\right)}{16}+\frac{\operatorname{KL}\left(\widehat{z}_{t+1}, z_{t}\right)}{16}\right) \quad \text { (Pinsker's inequality) } \\
% & \leq \frac{128}{\epsilon^{2}} \Theta_{t+1} .
% \end{aligned}
% $$

% Combining everything, we get

% $$
% \mathrm{KL}\left(z^{*}, z_{t}\right) \leq \frac{\sqrt{128}}{\epsilon} \sqrt{\Theta_{t+1}} \leq \frac{\sqrt{128 \ln (M N)}}{\epsilon}\left(1+\frac{15 \eta^{2} C_{2}}{32}\right)^{\frac{T_{0}-t-1}{2}}
% $$

% which completes the proof.

\begin{lemma}\label{lemmausc4}
For any $z'\in\mathcal{Z}$ with $z'\neq \widehat{z}_{t+1}$ it holds that:
$$
4\|z_t-\widehat{z}_{t}\|^2+4\left\|z_{t}-\widehat{z}_{t+1}\right\|^2+3\eta^2\left(\|\zeta_{t}^X\|^2+\|\zeta_{t}^Y\|^2\right) \geq \frac{\eta^{2}\left[F\left(\widehat{z}_{t+1}\right)^{\top}\left(\widehat{z}_{t+1}-z^{\prime}\right)\right]_{+}^{2}}{\left\|\widehat{z}_{t+1}-z^{\prime}\right\|^{2}} $$
%  Similarly for any $z'\in\mathcal{Z}$ with $z'\neq z_{t+1}$:
% $$
% 3\left\|z_{t+1}-\widehat{z}_{t+1}\right\|^2+3\left\|z_{t}-z_{t+1}\right\|^2 + 3\eta^2\left(\|\zeta_{t}^X\|^2+\|\zeta_{t}^Y\|^2\right) \geq \frac{\eta^{2}\left[F\left(z_{t+1}\right)^{\top}\left(z_{t+1}-z^{\prime}\right)\right]_{+}^{2}}{\left\|z_{t+1}-z^{\prime}\right\|^{2}}
% $$
where $[a]_{+}:=\max\{a,0\}$.
\end{lemma}

\begin{proof}
Below we consider any $z^{\prime} \neq \widehat{z}_{t+1} \in \mathcal{Z}$. Considering OMDA with $D_{\psi}(u, \boldsymbol{v})=$ $\frac{1}{2}\|u-\boldsymbol{v}\|^{2}$, and using the first-order optimality condition of $\widehat{z}_{t+1}$ and $z_{t+1}$, we have

$$
\begin{aligned}
\left(\widehat{z}_{t+1}-\widehat{z}_{t}+\eta \tilde{F}\left(z_{t}\right)\right)^{\top}\left(z^{\prime}-\widehat{z}_{t+1}\right) & \geq 0, \\
\left(z_{t+1}-\widehat{z}_{t+1}+\eta \tilde{F}\left(z_{t}\right)\right)^{\top}\left(z^{\prime}-z_{t+1}\right) & \geq 0 .
\end{aligned}
$$
Rearranging the terms and we get
$$
\begin{aligned}
\left(\widehat{z}_{t+1}-\widehat{z}_{t}\right)^{\top}\left(z^{\prime}-\widehat{z}_{t+1}\right) & \geq \eta \tilde{F}\left(z_{t}\right)^{\top}\left(\widehat{z}_{t+1}-z^{\prime}\right) \\
&=\eta F\left(z_{t}\right)^{\top}\left(\widehat{z}_{t+1}-z^{\prime}\right) + \eta\left(\zeta_{t}^X,-\zeta_{t}^Y\right)\cdot(\widehat{z}_{t+1}-z') \\
& =\eta F\left(\widehat{z}_{t+1}\right)^{\top}\left(\widehat{z}_{t+1}-z^{\prime}\right)+\eta\left(F\left(z_{t}\right)-F\left(\widehat{z}_{t+1}\right)\right)^{\top}\left(\widehat{z}_{t+1}-z^{\prime}\right) \\
&\quad + \eta\left(\zeta_{t}^X,-\zeta_{t}^Y\right)\cdot(\widehat{z}_{t+1}-z') \\
& \geq \eta F\left(\widehat{z}_{t+1}\right)^{\top}\left(\widehat{z}_{t+1}-z^{\prime}\right)-\eta L\left\|z_{t}-\widehat{z}_{t+1}\right\|\left\|\widehat{z}_{t+1}-z^{\prime}\right\| + \eta\left(\zeta_{t}^X,-\zeta_{t}^Y\right)\cdot(\widehat{z}_{t+1}-z') \\
& \geq \eta F\left(\widehat{z}_{t+1}\right)^{\top}\left(\widehat{z}_{t+1}-z^{\prime}\right)-\frac{1}{8}\left\|z_{t}-\widehat{z}_{t+1}\right\|\left\|\widehat{z}_{t+1}-z^{\prime}\right\| + \eta\left(\zeta_{t}^X,-\zeta_{t}^Y\right)\cdot(\widehat{z}_{t+1}-z')
\end{aligned}
$$
and
$$
\begin{aligned}
\left(z_{t+1}-\widehat{z}_{t+1}\right)^{\top}\left(z^{\prime}-z_{t+1}\right) & \geq \eta \tilde{F}\left(z_{t}\right)^{\top}\left(z_{t+1}-z^{\prime}\right) \\
& = \eta F\left(z_{t}\right)^{\top}\left(z_{t+1}-z^{\prime}\right)  + \eta\left(\zeta_{t}^X,-\zeta_{t}^Y\right)\cdot(z_{t+1}-z') \\
& =\eta F\left(z_{t+1}\right)^{\top}\left(z_{t+1}-z^{\prime}\right)+\eta\left(F\left(z_{t}\right)-F\left(z_{t+1}\right)\right)^{\top}\left(z_{t+1}-z^{\prime}\right) \\
&\quad  + \eta\left(\zeta_{t}^X,-\zeta_{t}^Y\right)\cdot(z_{t+1}-z') \\
& \geq \eta F\left(z_{t+1}\right)^{\top}\left(z_{t+1}-z^{\prime}\right)-\eta L\left\|z_{t}-z_{t+1}\right\|\left\|z_{t+1}-z^{\prime}\right\| \\
&\quad  + \eta\left(\zeta_{t}^X,-\zeta_{t}^Y\right)\cdot(z_{t+1}-z')  \\
& \geq \eta F\left(z_{t+1}\right)^{\top}\left(z_{t+1}-z^{\prime}\right)-\frac{1}{8}\left\|z_{t}-z_{t+1}\right\|\left\|z_{t+1}-z^{\prime}\right\|   + \eta\left(\zeta_{t}^X,-\zeta_{t}^Y\right)\cdot(z_{t+1}-z') 
\end{aligned}
$$
Here, for both block, the third step uses Hölder's inequality and the smoothness assumption and the last step uses the condition $\eta \leq 1 /(8 L)$. Upper bounding the left-hand side of the two inequalities by $\left\|\widehat{z}_{t+1}-\widehat{z}_{t}\right\|\left\|\widehat{z}_{t+1}-z^{\prime}\right\|$ and $\left\|z_{t+1}-\widehat{z}_{t+1}\right\|\left\|z_{t+1}-z^{\prime}\right\|$ respectively and then rearranging, we get
$$
\begin{aligned}
& \left\|\widehat{z}_{t+1}-z^{\prime}\right\|\left(\left\|\widehat{z}_{t+1}-\widehat{z}_{t}\right\|+\frac{1}{8}\left\|z_{t}-\widehat{z}_{t+1}\right\|\right) \geq \eta F\left(\widehat{z}_{t+1}\right)^{\top}\left(\widehat{z}_{t+1}-z^{\prime}\right) + \eta\left(\zeta_{t}^X,-\zeta_{t}^Y\right)\cdot(\widehat{z}_{t+1}-z')\\
& \left\|z_{t+1}-z^{\prime}\right\|\left(\left\|z_{t+1}-\widehat{z}_{t+1}\right\|+\frac{1}{8}\left\|z_{t}-z_{t+1}\right\|\right) \geq \eta F\left(z_{t+1}\right)^{\top}\left(z_{t+1}-z^{\prime}\right)+ \eta\left(\zeta_{t}^X,-\zeta_{t}^Y\right)\cdot(z_{t+1}-z') 
\end{aligned}
$$
$$
\begin{aligned}
& \left(\left\|\widehat{z}_{t+1}-\widehat{z}_{t}\right\|+\frac{1}{8}\left\|z_{t}-\widehat{z}_{t+1}\right\|+ \frac{\eta\left(-\zeta_{t}^X,\zeta_{t}^Y\right)\cdot(\widehat{z}_{t+1}-z')}{\|\widehat{z}_{t+1}-z^{\prime}\|}\right) \geq \frac{\eta F\left(\widehat{z}_{t+1}\right)^{\top}\left(\widehat{z}_{t+1}-z^{\prime}\right)}{\|\widehat{z}_{t+1}-z^{\prime}\|} \\
& \left(\left\|z_{t+1}-\widehat{z}_{t+1}\right\|+\frac{1}{8}\left\|z_{t}-z_{t+1}\right\|+ \frac{\eta\left(-\zeta_{t}^X,\zeta_{t}^Y\right)\cdot(z_{t+1}-z')}{\|z_{t+1}-z^{\prime}\|} \right) \geq \frac{\eta F\left(z_{t+1}\right)^{\top}\left(z_{t+1}-z^{\prime}\right)}{\|z_{t+1}-z^{\prime}\|}
\end{aligned}
$$
Therefore, we have
$$
\begin{aligned}
& \left(\left\|\widehat{z}_{t+1}-\widehat{z}_{t}\right\|+\frac{1}{8}\left\|z_{t}-\widehat{z}_{t+1}\right\|+ \frac{\eta\left(-\zeta_{t}^X,\zeta_{t}^Y\right)\cdot(\widehat{z}_{t+1}-z')}{\|\widehat{z}_{t+1}-z^{\prime}\|}\right)^{2} \geq \frac{\eta^{2}\left[F\left(\widehat{z}_{t+1}\right)^{\top}\left(\widehat{z}_{t+1}-z^{\prime}\right)\right]_{+}^{2}}{\left\|\widehat{z}_{t+1}-z^{\prime}\right\|^{2}} \\
& \left(\left\|z_{t+1}-\widehat{z}_{t+1}\right\|+\frac{1}{8}\left\|z_{t}-z_{t+1}\right\|+ \frac{\eta\left(-\zeta_{t}^X,\zeta_{t}^Y\right)\cdot(z_{t+1}-z')}{\|z_{t+1}-z^{\prime}\|} \right)^{2} \geq \frac{\eta^{2}\left[F\left(z_{t+1}\right)^{\top}\left(z_{t+1}-z^{\prime}\right)\right]_{+}^{2}}{\left\|z_{t+1}-z^{\prime}\right\|^{2}}
\end{aligned}
$$
Finally, by the triangle inequality and the fact $(a+b+c)^{2} \leq 3 (a^2+ b^2+c^2)$, we have:
$$
\begin{aligned}
 \left(\left\|\widehat{z}_{t+1}-\widehat{z}_{t}\right\|+\frac{1}{8}\left\|z_{t}-\widehat{z}_{t+1}\right\|+ c \right)^{2} &\leq \left(\left\|z_t-\widehat{z}_{t}\right\|+\frac{9}{8}\left\|z_{t}-\widehat{z}_{t+1}\right\|+ c \right)^{2} \\
 &\leq\left(\frac{9}{8}\left\|z_t-\widehat{z}_{t}\right\|+\frac{9}{8}\left\|z_{t}-\widehat{z}_{t+1}\right\|+ c \right)^{2} \\
 &\leq \frac{243}{64}\|z_t-\widehat{z}_{t}\|^2+\frac{243}{64}\left\|z_{t}-\widehat{z}_{t+1}\right\|^2+3c^2\\
 &\leq 4\|z_t-\widehat{z}_{t}\|^2+4\left\|z_{t}-\widehat{z}_{t+1}\right\|^2+3c^2\\
 \end{aligned}
$$
% Similarly we have:
% $$
% \begin{aligned}
% \left(\left\|z_{t+1}-\widehat{z}_{t+1}\right\|+\frac{1}{8}\left\|z_{t}-z_{t+1}\right\|+c\right)^{2}  \leq &4\left\|z_{t+1}-\widehat{z}_{t+1}\right\|^2+4\left\|z_{t}-z_{t+1}\right\|^2 \\
% \end{aligned}
% $$
Moreover we have $c=\frac{\left(-\zeta_{t}^X,\zeta_{t}^Y\right)\cdot(z_{t+1}-z')}{\|z_{t+1}-z^{\prime}\|}\leq \frac{\left\|\left(-\zeta_{t}^X,\zeta_{t}^Y\right)\right\|\cdot\|z_{t+1}-z'\|}{\|z_{t+1}-z^{\prime}\|}=\left\|\left(-\zeta_{t}^X,\zeta_{t}^Y\right)\right\|$. Thus $\left(\frac{\left(-\zeta_{t}^X,\zeta_{t}^Y\right)\cdot(z_{t+1}-z')}{\|z_{t+1}-z^{\prime}\|} \right)^{2} \leq \left\|\left(-\zeta_{t}^X,\zeta_{t}^Y\right)\right\|^2 = \|\zeta_{t}^X\|^2+\|\zeta_{t}^Y\|^2$.\\ \\
Combining the above we get the result we wanted to prove.
% Similarly we have $\left(\frac{\left(-\zeta_{t}^X,\zeta_{t}^Y\right)\cdot(\widehat{z}_{t+1}-z')}{\|\widehat{z}_{t+1}-z^{\prime}\|}\right)^{2}\leq \|\zeta_{t}^X\|^2+\|\zeta_{t}^Y\|^2$.\\

\end{proof}
\noindent\textbf{Proof of main theorem}\\\\
We consider the following quantities: $\Theta_{t}=\left\|\widehat{z}_{t}-\Pi_{\mathcal{Z}^{*}}\left(\widehat{z}_{t}\right)\right\|^{2}+\frac{1}{16}\left\|\widehat{z}_{t}-z_{t-1}\right\|^{2}$, $\alpha_{t}=\left\|\widehat{z}_{t+1}-z_{t}\right\|^{2}+\left\|z_{t}-\widehat{z}_{t}\right\|^{2}$.
From lemma \ref{lemmausc1} we have:
\begin{align}
    \Theta_{t+1} \leq \Theta_{t}-\frac{15}{16} \alpha_{t}+\delta_t \label{eq21}
\end{align}
where $\delta_t=3\eta(D_X+D_Y)\left(\|\zeta_{t-1}^X\|+\|\zeta_{t}^X\|+\|\zeta_{t-1}^Y\|+\|\zeta_{t}^Y\|\right)$.\\ \\
% Below, we relate $\alpha_{t}$ to $\Theta_{t+1}$ using the SP-MS condition, and then apply Lemma \ref{lemmausc12} to show
% \begin{align}
%     \label{eq22}
%     \Theta_{t} \leq \begin{cases}2 \operatorname{dist}^{2}\left(\widehat{z}_{1}, \mathcal{Z}^{*}\right)\left(1+C_{5}\right)^{-t} & \text { if } \beta=0 \\ {\left[\left(1+4\left(\frac{4}{\beta}\right)^{\frac{1}{\beta}}\right) \operatorname{dist}^{2}\left(\widehat{z}_{1}, \mathcal{Z}^{*}\right)+2\left(\frac{2}{C_{5} \beta}\right)^{\frac{1}{\beta}}\right] t^{-\frac{1}{\beta}}} & \text { if } \beta>0\end{cases}
% \end{align}

% where $C_{5}=\min \left\{\frac{16 \eta^{2} C^{2}}{81}, \frac{1}{2}\right\}$ as defined in the statement of the theorem. This is enough to prove the theorem since

% $$
% \begin{aligned}
% \operatorname{dist}^{2}\left(z_{t}, \mathcal{Z}^{*}\right) & \leq\left\|z_{t}-\Pi_{\mathcal{Z}^{*}}\left(\widehat{z}_{t+1}\right)\right\|^{2} \\
% & \leq 2\left\|\widehat{z}_{t+1}-\Pi_{\mathcal{Z}^{*}}\left(\widehat{z}_{t+1}\right)\right\|^{2}+2\left\|\widehat{z}_{t+1}-z_{t}\right\|^{2} \\
% & \leq 32 \Theta_{t+1} \leq 32 \Theta_{t}+32 \delta_t .
% \end{aligned}
% $$

% Next, we prove Eq. \eqref{eq22}. We first show a simple fact by Eq. \eqref{eq21}:

% $$
% \left\|\widehat{z}_{t+1}-z_{t}\right\|^{2} \leq \alpha_{t} \leq \frac{16}{15} \Theta_{t}+\frac{16}{15} \delta_{t} \leq \cdots \leq \frac{16}{15} \Theta_{1}+\frac{16}{15}\sum\limits_{i=1}^t\delta_{i}
% $$
\noindent We will relate $\alpha_{t}$ to $\Theta_{t+1}$ using the SP-MS condition and this will give an exponential convergence rate for $\Theta_{t}$.
\begin{align}
\alpha_{t} & \geq \frac{1}{2}\left\|\widehat{z}_{t+1}-z_{t}\right\|^{2}+\frac{1}{2}\left(\left\|\widehat{z}_{t+1}-z_{t}\right\|^{2}+\left\|z_{t}-\widehat{z}_{t}\right\|^{2}\right) \nonumber\\
& \geq \frac{1}{2}\left\|\widehat{z}_{t+1}-z_{t}\right\|^{2}+\frac{ \eta^{2}}{8} \sup _{z^{\prime} \in \mathcal{Z}} \frac{\left[F\left(\widehat{z}_{t+1}\right)^{\top}\left(\widehat{z}_{t+1}-z^{\prime}\right)\right]_{+}^{2}}{\left\|\widehat{z}_{t+1}-z^{\prime}\right\|^{2}}-\epsilon_t \quad \text { (Lemma \ref{lemmausc4}, $\epsilon_t=\frac{3}{8}\eta^2\left(\|\zeta_{t}^X\|^2+\|\zeta_{t}^Y\|^2\right)$) } \nonumber\\
& \geq \frac{1}{2}\left\|\widehat{z}_{t+1}-z_{t}\right\|^{2}+\frac{ \eta^{2} C^{2}}{8}\left\|\widehat{z}_{t+1}-\Pi_{\mathcal{Z}^{*}}\left(\widehat{z}_{t+1}\right)\right\|^{2} -\epsilon_t\quad \text { (SP-MS condition) } \nonumber\\
& \geq \min \left\{\frac{\eta^{2} C^{2}}{8}, \frac{1}{2}\right\}\left(\left\|\widehat{z}_{t+1}-z_{t}\right\|^{2}+\left\|\widehat{z}_{t+1}-\Pi_{\mathcal{Z}^{*}}\left(\widehat{z}_{t+1}\right)\right\|^{2}\right)  -\epsilon_t\nonumber\\
& =\min \left\{\frac{\eta^{2} C^{2}}{8}, \frac{1}{2}\right\}\Theta_{t+1}  -\epsilon_t\nonumber
\end{align}

\noindent Combining this with Eq. \eqref{eq21}, we get:
\begin{align*}
   &\Theta_{t+1} \leq \Theta_{t}-C^{\prime} \Theta_{t+1}+\epsilon_t+\delta_t \Leftrightarrow\\
   &(1+C^{\prime})\Theta_{t+1} \leq \Theta_{t} +\epsilon_t+\delta_t \Leftrightarrow\\
   &\Theta_{t+1} \leq \frac{1}{1+C^{\prime}}\Theta_{t} +\frac{1}{1+C^{\prime}}(\epsilon_t+\delta_t) \Rightarrow\\
   &\Theta_t \leq \Theta_1\left(\frac{1}{1+C^{\prime}}\right)^{t-1} + \sum\limits_{i=1}^{t-1} \left(\frac{1}{1+C^{\prime}}\right)^{t-i}(\epsilon_i+\delta_i)
\end{align*}
where $C'=\min \left\{\frac{\eta^{2} C^{2}}{8}, \frac{1}{2}\right\}$

% Eq. \eqref{eq22} is then proven by noticing that $\Theta_{1}=\operatorname{dist}^{2}\left(\widehat{z}_{1}, \mathcal{Z}^{*}\right)$.
\subsubsection{Proof of Corollary \ref{cor1}\\}
\textbf{Statement.} At the $t$-th iteration of Robust OMDA we have the following guarantee:
\begin{align*}
\operatorname{dist}^{2}\left(\widehat{z}_{t}, \mathcal{Z}^{*}\right) \leq  (\operatorname{dist}^{2}\left(\widehat{z}_{1}, \mathcal{Z}^{*}\right) -d_{\infty})\left(\frac{1}{1+C^{\prime}}\right)^{t-1} + d_{\infty}
\end{align*}
where $d_{\infty}=\max\left\{\frac{8}{\eta^2C^2}, 2 \right\} (Z\eta+Z^2\eta^2)$, 
$C'=\min \left\{\frac{\eta^{2} C^{2}}{8}, \frac{1}{2}\right\}$ and $Z=3(\zeta_X+\zeta_Y)\max\{D_X+D_Y,1\}$.\\
\begin{proof}
We set $Z=3\eta(\zeta_X+\zeta_Y)\max\{D_X+D_Y,1\}$. Then we have:
\[\epsilon_t\leq Z^2\eta^2, \quad \delta_t\leq  Z\eta \quad t\in[T]\]
% E.g. such a $Z$ is $6\eta(D_X+D_Y)\left(\|\zeta_{t-1}^X\|+\|\zeta_{t}^X\|+\|\zeta_{t-1}^Y\|+\|\zeta_{t}^Y\|\right)$, assuming $D_X+D_Y>1$
% Thus, setting $\eta=\epsilon/Z$ we get $\epsilon_t<\epsilon/2$ and $\delta_t<\epsilon/2$ for all t.
% \[C'= \frac{\epsilon^{2} C^{2}}{8Z^2}\]
\begin{align*}
\operatorname{dist}^{2}\left(\widehat{z}_{t}, \mathcal{Z}^{*}\right) &\leq \operatorname{dist}^{2}\left(\widehat{z}_{1}, \mathcal{Z}^{*}\right) \left(\frac{1}{1+C^{\prime}}\right)^{t-1} + (Z\eta+Z^2\eta^2)\sum\limits_{i=1}^{t-1} \left(\frac{1}{1+C^{\prime}}\right)^{t-i}
\\
&= \operatorname{dist}^{2}\left(\widehat{z}_{1}, \mathcal{Z}^{*}\right) \left(\frac{1}{1+C^{\prime}}\right)^{t-1} + (Z\eta+Z^2\eta^2)\sum\limits_{i=1}^{t-1} \left(\frac{1}{1+C^{\prime}}\right)^{i}
\\
&= \operatorname{dist}^{2}\left(\widehat{z}_{1}, \mathcal{Z}^{*}\right) \left(\frac{1}{1+C^{\prime}}\right)^{t-1} + (Z\eta+Z^2\eta^2) \left(1/C' + (-1 - 1/C') \frac{1}{(1 + C')^t}\right)
\\
&= (\operatorname{dist}^{2}\left(\widehat{z}_{1}, \mathcal{Z}^{*}\right) -(Z\eta+Z^2\eta^2)/C')\left(\frac{1}{1+C^{\prime}}\right)^{t-1} + (Z\eta+Z^2\eta^2)/C'
\\
&= (\operatorname{dist}^{2}\left(\widehat{z}_{1}, \mathcal{Z}^{*}\right) -d_{\infty})\left(\frac{1}{1+C^{\prime}}\right)^{t-1} + d_{\infty}
\end{align*}
where $d_{\infty}=(Z\eta+Z^2\eta^2)/C'=\max\left\{\frac{8}{\eta^2C^2}, 2 \right\} (Z\eta+Z^2\eta^2)$.
\end{proof}

\subsection{Proof of Theorem \ref{thr:lb}}
We will prove the following extension of Theorem \ref{thr:lb}. The extension was made in order to match the results in Corollary \ref{cor1}, which gave an upper bound linear in $Z_X+Z_Y$ for the quantity $\operatorname{dist}\left(\widehat{z}_{t}, \mathcal{Z}^{*}\right)$  for strongly convex strongly concave functions and bilinear functions. 
\begin{theorem}\label{thr:lbext} Consider a deterministic algorithm $\mathcal{A}$ that estimates saddle points of convex concave functions $f(x,y)$ over the domain $\|x\|\leq D_X$, $\|y\|\leq D_Y$, where $x$ and $y$ are $d$-dimensional vectors, using $T$ adaptive queries on noisy gradient oracles  with $\|\zeta^X_t\|\leq Z_X$ and  $\|\zeta^Y_t\| \leq Z_Y$  for all $t\in[T]$ and $Z_Y\leq D/2$, $Z_X\leq D/2$, where $D=min\{D_X,D_Y\}$. For any such $\mathcal{A}$:
      \begin{itemize}
        \item There exists a convex concave (bilinear) function $f(x,y)$ and a noise sequence realisation, such that $\mathcal{A}$ returns a point that has distance at least $\frac{Z_X+Z_Y}{\sqrt{2}}$ from any saddle point of $f$.
        \item There exists a convex concave (bilinear) function $f(x,y)$ and a noise sequence realisation such that $\mathcal{A}$ returns a point $(x_0,y_0)$ with duality gap $f(x_0,y)-f(x,y_0)\geq\frac{1}{4}Z_Y D_Y+\frac{1}{4}Z_X D_X$ for some pair $(x,y)$ inside the domain $\|x\|\leq D_X$, $\|y\|\leq D_Y$.
        \item There exists a smooth strongly convex - strongly concave function $f(x,y)$ and a noise sequence realisation, such that $\mathcal{A}$ returns a point that has distance at least $\frac{Z_X+Z_Y}{\sqrt{2}}$ from any saddle point of $f$.
    \end{itemize} 
\end{theorem}
\noindent For all three lower bounds we will consider the following quantities: Let $\zeta_X,\zeta_Y$ be $d$-dimensional vectors, such that $\|\zeta_X\|=Z_X$ and $\|\zeta_Y\|=Z_Y$.
\subsubsection{Lower Bound for Bilinear Functions}
 We consider three scenarios for the original function $f$ and the noise. In the first scenario $f(x,y)=f_1(x,y):=xy^T+ x \zeta_X^T+ y \zeta_Y^T$ and the noise is 
$\zeta^X_t=-\zeta_X$  and $\zeta^Y_t=-\zeta_Y$ for all $t\in[T]$. In the second scenario $f(x,y)=f_2(x,y):=xy^T- x \zeta_X^T- y \zeta_Y^T$ and the noise is 
$\zeta^X_t=\zeta_X$  and $\zeta^Y_t=\zeta_Y$ for all $t\in[T]$. In the third scenario $f(x,y)=f_3(x,y):=xy^T$ and the noise is zero for all $t\in[T]$. In all scenarios the domain is $\|x\|\leq D_X, \|y\|\leq D_Y$. From the hypothesis we have that $\|\zeta_Y\|=Z_Y\leq D_X/2$ and $\|\zeta_X\|=Z_X\leq D_Y/2$. It is easy to see that in all three cases the noisy gradient oracle is $g_x(x,y)=y $, $g_y(x,y)=x $ and in all three scenarios $\mathcal{A}$ will return the same point $(x_0,y_0)$, as $\mathcal{A}$ is assumed to be deterministic.\\
In the first scenario the true saddle point is $(x_1,y_1)=(-\zeta_Y,-\zeta_X)$ and in the second scenario the true saddle point is $(x_2,y_2)=(\zeta_Y,\zeta_X)$. There is no point $(x_0,y_0)$ that has distance less than 
$\sqrt{\|\zeta_X\|^2+\|\zeta_Y\|^2}$ from both $(x_1,y_1)$ and $(x_2,y_2)$.
Thus, either in the first or in the second scenario $\mathcal{A}$ returns a saddle point that has distance at least
$\sqrt{\|\zeta_X\|^2+\|\zeta_Y\|^2}$ from the true saddle point.\\ Note that $\sqrt{\|\zeta_X\|^2+\|\zeta_Y\|^2}\geq \frac{\|\zeta_X\|+\|\zeta_Y\|}{\sqrt{2}}=\frac{Z_X+Z_Y}{\sqrt{2}}$. \\ \\
Next we will prove the lower bound for the duality gap. First we analyse the duality gap of function $f_1$ at point $(x_0,y_0)$. Let $(x^*,y^*)$ be the true saddle point of $f_1$.
\begin{align*} 
  f_1(x_0,y)-f_1(x,y_0)&=x_0\cdot y^T+x_0\cdot {\zeta_X}^T + y\cdot {\zeta_Y}^T-x\cdot y_0 ^ T-x\cdot {\zeta_X}^T - y_0\cdot {\zeta_Y}^T \\
  &= (x_0-x^*)\cdot y^T + (x_0-x^*)\cdot {\zeta_X}^T + y\cdot {\zeta_Y}^T-x\cdot (y_0-y^*) ^ T-x\cdot {\zeta_X}^T  \\
  &\quad - (y_0-y^*)\cdot {\zeta_Y}^T+ x^*\cdot y^T + x^*\cdot {\zeta_X}^T + y\cdot {\zeta_Y}^T-x\cdot {y^*}^T-x\cdot {\zeta_X}^T - y^*\cdot {\zeta_Y}^T \\
    & = (x_0-x^*)\cdot y^T + (x_0-x^*)\cdot {\zeta_X}^T + y\cdot {\zeta_Y}^T\\
    &\quad -x\cdot (y_0-y^*) ^ T-x\cdot {\zeta_X}^T - (y_0-y^*)\cdot {\zeta_Y}^T \\
  &=  (x_0-x^*)\cdot ({\zeta_X}^T+y^T)+(y^*-y_0)\cdot ({\zeta_Y}^T+x^T) + y\cdot {\zeta_Y}^T-x\cdot {\zeta_X}^T
\end{align*}
We set $(x',y')=(x+\zeta_Y,y+\zeta_X)$ and we have:
\begin{align*}
    f_1(x_0,y)-f_1(x,y_0)&=  (x_0-x^*)\cdot (y'^T)+(y^*-y_0)\cdot (x'^T) + y'\cdot {\zeta_Y}^T-x'\cdot {\zeta_X}^T\\
    &=  (x_0-x^*+\zeta_Y)\cdot (y'^T)-(y_0-y^*+\zeta_X)\cdot (x'^T) \\
    &=  (x_0+2\zeta_Y)\cdot (y'^T)-(y_0+2\zeta_X)\cdot (x'^T) 
\end{align*}
For $x'=-\frac{D_X}{2\|y_0+2\zeta_X\|}(y_0+2\zeta_X)\Leftrightarrow x=-\frac{D_X}{2\|y_0+2\zeta_X\|}(y_0+2\zeta_X)-\zeta_Y$ and \\$y'=\frac{D_Y}
{2\|x_0+2\zeta_Y\|}(x_0+2\zeta_Y)\Leftrightarrow y=\frac{D_Y}
{2\|x_0+2\zeta_Y\|}(x_0+2\zeta_Y) -\zeta_X$ we have
\begin{align}
    &f_1(x_0,y)-f_1(x,y_0)=  \frac{1}{2}\|x_0+2\zeta_Y\|D_Y+\frac{1}{2}\|y_0+2\zeta_X\|D_X \tag{1}
\end{align}
Note that $(x,y)$ we use in the above equation is inside the domain $\|x\|\leq D_X$, $\|y\|\leq D_Y$ that we initially assumed. \\
Then we analyse the duality gap of function $f_3$ at point $(x_0,y_0)$.  
\begin{align*} 
  &f_3(x_0,y)-f_3(x,y_0)=x_0\cdot y^T-x\cdot y_0 ^ T=\|x_0\|D_Y+\|y_0\|D_X \tag{2}
\end{align*}
for $x=-\frac{D_X}{\|y_0\|}y_0$ and $y=\frac{D_Y}{\|x_0\|}x_0$. Pair $(x,y)$ is inside the domain $\|x\|\leq D_X$, $\|y\|\leq D_Y$.\\ \\ 
Now we need to consider the following cases:
\begin{itemize}
\item If $\|x_0\| \leq \|\zeta_Y\| $ and $\|y_0\| \leq \|\zeta_X\| $ then \\
$\frac{1}{2}\|x_0+2\zeta_Y\|D_Y+\frac{1}{2}\|y_0+2\zeta_X\|D_X \geq \frac{1}{2}(2\|\zeta_Y\|-\|x_0\|)D_Y+\frac{1}{2}(2\|\zeta_X\|-\|y_0\|)\|D_X \geq \\ \frac{1}{2}\|\zeta_Y\|D_Y+\frac{1}{2}\|\zeta_X\|D_X$\\
Thus from equation (1) there exists $(x,y)$ such that $f_1(x_0,y)-f_1(x,y_0)\geq\frac{1}{2}\|\zeta_Y\|D_Y+\frac{1}{2}\|\zeta_X\|D_X$
\item If $\|x_0\| \geq \|\zeta_Y\| $ and $\|y_0\| \leq \|\zeta_X\| $ then:
    \begin{itemize}
        \item[*] If $\|\zeta_Y\|D_Y \geq \|\zeta_X\|D_X$ then: \\
        $\|x_0\|D_Y+\|y_0\|D_X\geq\|x_0\|D_Y\geq\|\zeta_Y\|D_Y \geq \frac{1}{2}\|\zeta_Y\|D_Y+\frac{1}{2}\|\zeta_X\|D_X$ \\
        Thus from eq. (2) there exists $(x,y)$ such that $f_3(x_0,y)-f_3(x,y_0)\geq\frac{1}{2}\|\zeta_Y\|D_Y+\frac{1}{2}\|\zeta_X\|D_X$

        \item[*] If $\|\zeta_Y\|D_Y \leq \|\zeta_X\|D_X$ then: \\
        $\frac{1}{2}\|x_0+2\zeta_Y\|D_Y+\frac{1}{2}\|y_0+2\zeta_X\|D_X \geq \frac{1}{2}(2\|\zeta_X\|-\|y_0\|)\|D_X \geq \\ \frac{1}{2}\|\zeta_X\|D_X \geq \frac{1}{4}\|\zeta_Y\|D_Y+\frac{1}{4}\|\zeta_X\|D_X$\\
        Thus from eq. (1) there exists $(x,y)$ such that $f_1(x_0,y)-f_1(x,y_0)\geq\frac{1}{4}\|\zeta_Y\|D_Y+\frac{1}{4}\|\zeta_X\|D_X$
    \end{itemize}
\item If $\|x_0\| \leq \|\zeta_Y\| $ and $\|y_0\| \geq \|\zeta_X\| $ then:
    \begin{itemize}
         \item[*] If $\|\zeta_Y\|D_Y \leq \|\zeta_X\|D_X$ then: \\
        $\|x_0\|D_Y+\|y_0\|D_X\geq\|y_0\|D_X\geq\|\zeta_X\|D_X \geq \frac{1}{2}\|\zeta_Y\|D_Y+\frac{1}{2}\|\zeta_X\|D_X$ \\
        Thus from eq. (2) there exists $(x,y)$ such that $f_3(x_0,y)-f_3(x,y_0)\geq\frac{1}{2}\|\zeta_Y\|D_Y+\frac{1}{2}\|\zeta_X\|D_X$

        \item[*] If $\|\zeta_Y\|D_Y \geq \|\zeta_X\|D_X$ then: \\
        $\frac{1}{2}\|x_0+2\zeta_Y\|D_Y+\frac{1}{2}\|y_0+2\zeta_X\|D_X \geq \frac{1}{2}(2\|\zeta_Y\|-\|x_0\|)\|D_Y \geq \\ \frac{1}{2}\|\zeta_Y\|D_Y \geq \frac{1}{4}\|\zeta_Y\|D_Y+\frac{1}{4}\|\zeta_X\|D_X$\\
        Thus from eq. (1) there exists $(x,y)$ such that $f_1(x_0,y)-f_1(x,y_0)\geq\frac{1}{4}\|\zeta_Y\|D_Y+\frac{1}{4}\|\zeta_X\|D_X$
    \end{itemize}
\end{itemize}

\noindent In all cases there exists a function $f$ over domain $\|x\|\leq D_X$, $\|y\|\leq D_Y$ and a pair $(x,y)$ inside the domain such that $f(x_0,y)-f(x,y_0)\geq\frac{1}{4}\|\zeta_Y\|D_Y+\frac{1}{4}\|\zeta_X\|D_X=\frac{1}{4}Z_Y D_Y+\frac{1}{4}Z_X D_X$.
\subsubsection{Lower Bound for Strongly Convex - Strongly Concave Functions}
 We consider two scenarios for the original function $f$ and the noise. In the first scenario $f(x,y)=f_1(x,y):=\frac{1}{2}\|x\|^2-\frac{1}{2}\|y\|^2+ x \zeta_X^T+ y \zeta_Y^T$ and the noise is 
$\zeta^X_t=-\zeta_X$  and $\zeta^Y_t=-\zeta_Y$ for all $t\in[T]$. In the second scenario $f(x,y)=f_2(x,y):=\frac{1}{2}\|x\|^2-\frac{1}{2}\|y\|^2- x \zeta_X^T- y \zeta_Y^T$ and the noise is 
$\zeta^X_t=\zeta_X$  and $\zeta^Y_t=\zeta_Y$ for all $t\in[T]$. In both scenarios the domain is $\|x\|\leq D_X, \|y\|\leq D_Y$. We also assume that $\|\zeta_Y\|\leq D_X/2$ and $\|\zeta_X\|\leq D_Y/2$. It is easy to see that in both scenarios the noisy gradient oracle is $g_x(x,y)=x $, $g_y(x,y)=-y $ and in both scenarios $\mathcal{A}$ will return the same point $(x_0,y_0)$, as $\mathcal{A}$ is assumed to be deterministic.\\
In the first scenario the true saddle point is $(x_1,y_1)=(-\zeta_X,\zeta_Y)$ and in the second scenario the true saddle point is $(x_2,y_2)=(\zeta_X,-\zeta_Y)$. There is no point $(x_0,y_0)$ that has distance less than 
$\sqrt{\|\zeta_X\|^2+\|\zeta_Y\|^2}$ from both $(x_1,y_1)$ and $(x_2,y_2)$.
Thus, either in the first or in the second scenario $\mathcal{A}$ returns a saddle point that has distance at least
$\sqrt{\|\zeta_X\|^2+\|\zeta_Y\|^2}$ from the true saddle point. \\ Note that $\sqrt{\|\zeta_X\|^2+\|\zeta_Y\|^2}\geq \frac{\|\zeta_X\|+\|\zeta_Y\|}{\sqrt{2}}=\frac{Z_X+Z_Y}{\sqrt{2}}$. \\ \\

\subsection{Unbiased Gradient Estimators for the MDP Lagrangian}
We remind that each sample $\hat  g_d^i$ is constructed by sample $(s_i,s'_i,a_i,r_i)$ according to the formula $\hat  g_d^i(s,a)=\ind{(s_i,a_i)=(s,a)} \cdot  \frac{\gamma h(s_i') -h(s_i)+r_i}{(1-\gamma)\cdot d_{n}(s_i,a_i)}$. Similarly, each sample $\hat 
 g_h^i$ is constructed by sample $(s_i,s'_i,a_i,r_i)$ according to the formula $\hat  g_h^i(s')=d(s_i,a_i) \frac{\gamma \ind{s'=s_i'} -\ind{s'=s_i}}{(1-\gamma)\cdot d_{n}(s_i,a_i)}$.\\\\ 
 \textbf{Statement.} Each clean sample $\hat g_d^i$ is an unbiased estimator of $g_d+\lambda d$ and each clean sample $\hat g_h^i$ is an unbiased estimator of $g_h-\rho$. 
 \begin{proof}
 For clean $\hat g_d^i$  samples we have $\hat  g_d^i(s,a)=\ind{(s_i,a_i)=(s,a)} \cdot  \frac{\gamma h(s_i') -h(s)+r_n(s,a)}{(1-\gamma)\cdot d_{n}(s,a)}$\\
 \[\mathbb{E}[\hat g_d^i (s,a)]=\mathbb{P}[(s_i,a_i)=(s,a)]\cdot  \frac{\gamma \mathbb{E}[h(s_i')|(s_i,a_i)=(s,a)] -h(s)+r_n(s,a)}{(1-\gamma)\cdot d_{n}(s,a)}\] 
 Moreover we have $\mathbb{P}[(s_i,a_i)=(s,a)]=(1-\gamma)\cdot d_{n}(s,a)$ and $\mathbb{E}[h(s_i')|(s_i,a_i)=(s,a)]=\sum_{s'}P_n(s, a, s') h(s')$.\\
 Thus, we obtain $\mathbb{E}[\hat g_d^i (s,a)]=r_n(s, a) - h(s) + \gamma \sum_{s'}P_n(s, a, s') h(s')=g_d(s,a)+\lambda d(s,a)$. \\ \\
 For clean $\hat g_h^i$  samples we have 
 \begin{align*}
   \mathbb{E}[\hat g_h^i (s'')]&=\sum_{s\in\mathcal{S}}\sum_{a\in\mathcal{A}}\sum_{s'\in\mathcal{S}}\mathbb{P}[(s_i,a_i,s'_i)=(s,a,s')]d(s,a) \frac{\gamma \ind{s''=s'} -\ind{s''=s}}{(1-\gamma)\cdot d_{n}(s,a)}\\    &=\sum_{s\in\mathcal{S}}\sum_{a\in\mathcal{A}}\sum_{s'\in\mathcal{S}}d_{n}(s,a)P_n(s, a, s')d(s,a) \frac{\gamma \ind{s''=s'} -\ind{s''=s}}{(1-\gamma)\cdot d_{n}(s,a)}\\ 
   &=\sum_{s\in\mathcal{S}}\sum_{a\in\mathcal{A}}\sum_{s'\in\mathcal{S}}P_n(s, a, s')d(s,a) (\gamma \ind{s''=s'} -\ind{s''=s})\\
   &=\sum_{s\in\mathcal{S}}\sum_{a\in\mathcal{A}}\gamma P_n(s, a, s'')d(s,a)-\sum_{s'\in\mathcal{S}}\sum_{a\in\mathcal{A}} P_n(s'', a, s')d(s'',a)\\
   &=\sum_{s\in\mathcal{S}}\sum_{a\in\mathcal{A}}\gamma P_n(s, a, s'')d(s,a)-\sum_{a\in\mathcal{A}} d(s'',a)\sum_{s'\in\mathcal{S}} P_n(s'', a, s')\\
   &=\sum_{s\in\mathcal{S}}\sum_{a\in\mathcal{A}}\gamma P_n(s, a, s'')d(s,a)-\sum_{a\in\mathcal{A}} d(s'',a)\\
   &=g_h(s'')-\rho(s'')
 \end{align*}
 \end{proof}

\subsection{Proof of Theorem \ref{thr:robust_grad}}
We will first prove an auxiliary lemma (Lemma \ref{aux_mean}) about the naive mean of a set of bounded observations containing a corrupted subset. This lemma will be directly applicable to the dataset of $h-$gradients. Moreover, we will prove that after filtering the dataset of $d-$gradients also satisfies the boundedness condition and so Lemma \ref{aux_mean} will be also applicable to the filtered $d-$gradients dataset. Finally we apply the lemma on the two datasets and derive the result stated in Theorem \ref{thr:robust_grad}.\\\\
\textbf{Statement.} We consider two datasets of gradient samples $D_d:=\{\hat g_d^i\ |\ i \in [\tilde{m}]\}$ and $D_h:=\{\hat g_h^i\ |\ i \in [\tilde{m}]\}$ with corruption level at most ${\epsilon}<0.5$. The uncorrupted samples of each dataset are iid and the uncorrupted samples of the one dataset are independent of the uncorrupted samples of the other dataset. We use $\hat{g}_h =\frac{1}{\tilde{m}}\sum\limits_{i=1}^{\tilde{m}} \hat g_h^i$ to estimate $g_h$ and apply Algorithm \ref{alg:rob_mean} to estimate $g_d$. Then with probability at least $1-\delta$ the estimation error satisfies the following guarantees: \begin{align*}
        |\hat{g}_h-g_h\|_1 &\leq \underbrace{\frac{4}{(1-\gamma)^2B(c)}\left(\frac{\sqrt{S\ \log(4S/\delta)}}{\sqrt{\tilde{m}}}+{\epsilon}\right)}_{E_1(\tilde{m},{\epsilon},\delta)},
        \\
    \|\hat{g}_d-g_d\|_2 &\leq \underbrace{6\sqrt{SA}\frac{2h_{max}+R}{(1-\gamma)B(c)}\left(\frac{\sqrt{2 \log\left(\frac{4SA}{\delta}\right)}}{\sqrt{\tilde{m}}}+2{\epsilon}\right)}_{E_2(\tilde{m},{\epsilon},\delta)}.
    \end{align*}

\begin{lemma}\label{aux_mean}
Let $\{x_1,x_2,...x_m\}$ be a collection of $m$ $d-$ dimensional samples. We consider three constants: $\epsilon$, $b$ and $\bar{x}$, relevant to this collection. We assume that $\epsilon<0.5$. The collection consists of two parts: 
\begin{itemize}
    \item The first part is a set of $(1-\epsilon)\cdot m$ i.i.d samples, all drawn from a common distribution $\mathcal{D}$ with expectation $\bar{x}$. We assume that $\|x\|_2\leq b$ for any $x\sim\mathcal{D}$. We refer to this part as the "good" samples, $G$.
    \item The second part is a set of $\epsilon \cdot m$ arbitrary $d-$ dimensional vectors. We refer to this part as the "bad" samples, $B$. We assume that $\|x\|_2\leq b$ for any $x \in B$. 
\end{itemize}
Then for any $\delta \in(0,1)$ the naive mean $\tilde{x}:=\frac{1}{m}\sum\limits_{i=1}^{m} x_i$ of the collection satisfies with probability at least $1-\delta$:
\[\|\tilde{x}-\bar{x}\|_2 \leq b\left(\frac{2\sqrt{d\ log(2d/\delta)}}{\sqrt{m}}+2\epsilon\right)\] 
\end{lemma}
\begin{proof}
\begin{align*}
\tilde{x}&=\frac{1}{m}\sum\limits_{i=1}^{m} x_i\\
&=\frac{|G|}{m}\frac{1}{|G|}\sum\limits_{x\in G} x + \frac{|B|}{m}\frac{1}{|B|}\sum\limits_{x\in B} x \\
&=(1-\epsilon)\frac{1}{|G|}\sum\limits_{i\in G} x + \epsilon\frac{1}{|B|}\sum\limits_{x\in B} x \\
&=\frac{1}{|G|}\sum\limits_{x\in G} x+\epsilon\left(\frac{1}{|B|}\sum\limits_{x\in B} x-\frac{1}{|G|}\sum\limits_{x\in G} x\right)  \\
\end{align*}
From Hoeffding's inequality we have for each $i\in [d]$ and $t>0$:
\[\mathbb{P}\left[\ \left|\frac{1}{|G|}\sum\limits_{x\in G} x[i]-\bar{x}[i] \right|\geq t \right ]\leq 2exp \left(-\frac{2|G|^2t^2}{|G|4b^2} \right)\]
We set 
\begin{align*}
    &\delta/d= 2exp \left(-\frac{|G|t^2}{2b^2} \right)\Leftrightarrow \\ &log(\delta/2d)=-\frac{(1-\epsilon)mt^2}{2b^2} \Leftrightarrow \\
    & t^2=log(2d/\delta)\frac{2b^2}{(1-\epsilon)m} \Leftrightarrow\\
    & t=\frac{b\sqrt{2log(2d/\delta)}}{ \sqrt{(1-\epsilon)m}}
\end{align*}
Taking a union bound over all $i \in [d]$ we obtain
\begin{align*}
    &\mathbb{P}\left[\bigcup_{i \in [d]}\left\{\left|\frac{1}{|G|}\sum\limits_{x\in G} x[i]-\bar{x}[i]\right|\geq \frac{b\sqrt{2log(2d/\delta)}}{ \sqrt{(1-\epsilon)m}} \right \}\right ]\leq \delta \Rightarrow\\
    &\mathbb{P}\left[\left\|\frac{1}{|G|}\sum\limits_{x\in G} x-\bar{x}\right\|_2\geq \frac{b\sqrt{2d\ log(2d/\delta)}}{ \sqrt{(1-\epsilon)m}} \right ]\leq \delta \Rightarrow (\epsilon<0.5)\\
    &\mathbb{P}\left[\left\|\frac{1}{|G|}\sum\limits_{x\in G} x-\bar{x}\right\|_2\geq \frac{2b\sqrt{d\ log(2d/\delta)}}{ \sqrt{m}} \right ]\leq \delta 
\end{align*}
Thus we have 
\begin{align*}
    \|\tilde{x}-\bar{x}\|_2 & =
    \left\|\frac{1}{|G|}\sum\limits_{x\in G} x-\bar{x}+\epsilon\left(\frac{1}{|B|}\sum\limits_{x\in B} x-\frac{1}{|G|}\sum\limits_{x\in G} x\right)\right\|_2 \\
    & \leq \left\|\frac{1}{|G|}\sum\limits_{x\in G} x-\bar{x}\right\|_2+\epsilon\left(\frac{1}{|B|}\sum\limits_{x\in B} \left\|x\right\|_2+\frac{1}{|G|}\sum\limits_{x\in G} \left\|x\right\|_2\right)\\
    & \leq b\left(\frac{2\sqrt{d\ log(2d/\delta)}}{\sqrt{m}}+2\epsilon\right) \text{with probability at least } 1-\delta\\
\end{align*}
This completes the proof of the auxiliary lemma. \end{proof}
\noindent Now we will proceed to the proof of the first claim of Theorem \ref{thr:robust_grad}. Each sample of the gradient wrt $h$ is bounded. In particular we have $\|g_h^i\|_1= \frac{(\gamma +1)d(s_i,a_i)}{(1-\gamma)\cdot d_{n}(s_i,a_i)}$ if $s_i\neq s_i'$ and $\|g_h^i\|_1= \frac{d(s_i,a_i)}{d_{n}(s_i,a_i)}$ if $s_i= s_i'$. In both cases $\|g_h^i\|_2 \leq \frac{2}{B(1-\gamma)^2}$. Note that this bound holds both for corrupted and uncorrupted $h$-gradient samples. Moreover, the uncorrupted samples are unbiased estimates of the true gradient $g_h$. Thus, all conditions of Lemma \ref{aux_mean} are satisfied for the collection $\{g_h^1,... g_h^m\}$. Applying the Lemma we conclude that for any $\delta \in(0,1)$ the naive mean $\hat{g_h}:=\frac{1}{m}\sum\limits_{i=1}^{m} g_h^i$ satisfies with probability at least $1-\delta/2$:
\[\|\hat{g_h}-g_h\|_2 \leq \frac{4}{B(1-\gamma)^2}\left(\frac{\sqrt{S\ log(4S/\delta)}}{\sqrt{m}}+\epsilon\right)\] 
We move on to the proof of the second claim of Theorem \ref{thr:robust_grad}.\\ \\
We remind that for uncorrupted samples it holds that:
\[{g_d}^i(s,a)=\ind{(s_i,a_i)=(s,a)} \cdot  \frac{\gamma h({s_i}^{\prime}) -h(s_i)+r_n(s_i,a_i)}{(1-\gamma)\cdot d_{n}(s_i,a_i)}\]
and for corrupted: 
\[g_d^i(s,a)=\ind{(s_i,a_i)=(s,a)} \cdot  \frac{\gamma h(s_i') -h(s_i)+r_i}{(1-\gamma)\cdot d_{n}(s_i,a_i)}\]
where $r_i$ could differ drastically from $r(s_i,a_i)$.
The uncorrupted samples have their norm bounded as follows:
\[\|g_d^i\|_2\leq\frac{2h_{max}+R}{(1-\gamma)B}\]
Since $\epsilon<0.5$, the median $M_k$ at coordinate $k$ is equal to one uncorrupted sample (or the mean of two uncorrupted samples, if their cardinality is even) and thus the median will also be bounded as :
\[|M_k|\leq\frac{2h_{max}+R}{(1-\gamma)B}\]
\begin{corollary} Let $dif_{cr}=\frac{2h_{max}+R}{(1-\gamma)B}$. In the cleaning step any entry differing more than $2dif_{cr}$ from $M_k$ will be dropped. 
\end{corollary} 
\begin{proof}
First, we observe that any entry differing more than $2dif_{cr}$ from $M_k$ will be corrupted, because for uncorrupted entries we have $|g_d^i[k]-M_k|\leq |g_d^i[k]|+|M_k| \leq 2\frac{2h_{max}+R}{(1-\gamma)B}=2dif_{cr}$. Since the number of corrupted entries is $m\cdot\epsilon$, and entries differing more than $2dif_{cr}$ from $M_k$ are a subset of corrupted samples we conclude that there are at most $m\cdot\epsilon$ entries differing more than $2dif_{cr}$ from $M_k$. Algorithm \ref{alg:rob_mean} drops the $m\cdot\epsilon$ most distant entries from $M_k$ at $k$-th coordinate. Thus all entries differing more than $2dif_{cr}$ from $M_k$ will be dropped.\end{proof}

\noindent  All the entries that will be kept will have absolute value at most $|M_k|+2\frac{2h_{max}+R}{(1-\gamma)B}\leq3\frac{2h_{max}+R}{(1-\gamma)B}$.
 Thus for each coordinate we will end up with a dataset of bounded values and a fraction of them is corrupted and we can apply Lemma \ref{aux_mean} for this dataset. The conditions of Lemma \ref{aux_mean} are met as follows: we have a total of $(1-\epsilon)m$ samples and at most $\epsilon m$ of them are corrupted, dimension $d$ is one, norm bound $b$ is 
 $3\frac{2h_{max}+R}{(1-\gamma)B}$, $\hat{g_d}[k]$ is the empirical mean of the dataset, $g_d[k]$ is the latent true mean and we set the error probability tolerance to $\frac{\delta}{2SA}$.
 We obtain that with probability at least $1-\frac{\delta}{2SA}$it holds for $k$-th coordinate that 
\[|\hat{g_d}[k]-g_d[k]| \leq 3\frac{2h_{max}+R}{(1-\gamma)B}\left(\frac{2 \sqrt{log(4SA/\delta)}}{\sqrt{(1-\epsilon)m}}+2\epsilon/(1-\epsilon)\right)\] 
The RHS is at most  $3\frac{2h_{max}+R}{(1-\gamma)B}\left(\frac{2 \sqrt{2log(4SA/\delta)}}{\sqrt{m}}+4\epsilon\right)$ 
 , since $\epsilon<0.5$.\\ \\
 Since the result holds for each particular coordinate $k$ with probability at least $1-\frac{\delta}{2SA}$, we can conclude through a union bound over all $SA$ coordinates that the results holds for all coordinates with probability at least $1-\delta/2$.
 The error between the estimator and the true gradient can be bounded as follows:
 \[\|\hat{g_d}-g_d\|_2 =  \sqrt{ \sum_{k=1}^{SA}(\hat{g_d}[k]-g_d[k])^2} \]
 The RHS is at most $3\sqrt{SA}\frac{2h_{max}+R}{(1-\gamma)B}\left(\frac{2\sqrt{2log(4SA/\delta)}}{\sqrt{m}}+4\epsilon\right)$ with probability at least $1-\delta/2$.\\ \\ 
 The proof of Theorem \ref{thr:robust_grad}
is completed by taking a union bound over the event of error in the estimation of $g_h$ and event of error in the estimation of $g_d$. 
 % Let $h_{max}$ be an upper bound to $|h(s)|$ for all $s \in \mathcal{S}$. For $d(s)$ we use the upper bound $\frac{1}{1-\gamma}$ because variable $d$ is supposed to estimate an occupancy measure that has sum $\frac{1}{1-\gamma}$.
 \paragraph{Corruption of data $(s_i,a_i)$} 
 We have assumed that the $(s_i,a_i)$ data are uncorrupted. The assumption was made to simplify the analysis by avoiding to account for the error of the empirical occupancy measure in the denominator of the gradient estimators. We can drop the assumption and with some extra arguments ensure a $O(\epsilon)$ estimation error . 
 
\noindent In particular, if corruption is also applied to the $(s_i,a_i)$ data, we could use a Monte Carlo estimator $\hat{d_n}$ of $d_n$. Then we would construct the gradient estimators using $\hat{d_n}$. We drop all samples that have $\hat{d_n}$ less than $B(c)/2$ (see Data Generation Process section). We claim that with high probability only an $O(\epsilon)$ fraction of the samples will de dropped in this step. For the remaining samples an $1-\epsilon$ fraction will be clean with gradient estimators $O(\epsilon)-$approximately unbiased and an $\epsilon$ fraction will be corrupted and will be handled similarly to the current version.
\subsection{Proof of Lemma \ref{lm.approx.bound}}
\textbf{Statement.} There exists $T$ such that the output of Robust OFTRL (Algorithm \ref{alg:aoftrl}) run for $T$ iterations on $f = \mathcal L$, with $\mathcal X = \mathcal D$, $\mathcal Y = \mathcal H$, $g_{X, t} = \hat g_h$ (the naive mean estimator) and $g_{Y, t} = \hat g_d$ ( the output of algorithm \ref{alg:rob_mean})  satisfies
\begin{align}\label{eq.gda_duality_gap_bound_appdx}
    \max\limits_{d \in \mathcal{D}}\mathcal{L}(d,\bar{h}, M_n)- \min\limits_{h \in \mathcal{H}}\mathcal{L}(\bar{d},h, M_n) \le 7C(\delta),
\end{align}
with probability at least $1-\delta$, where $C(\delta) := \sqrt{S}\left(E_1\left(\frac{m}{2T},{\epsilon},\frac{\delta}{T}\right) h_{max} +  \frac{\sqrt{A}E_2\left(\frac{m}{2T},{\epsilon},\frac{\delta}{T}\right)}{1-\gamma}\right)$.
\begin{proof}
The MDP Lagrangian is concave-convex, Lipschitz continuous and smooth and the domain of the variables is convex and bounded. Moreover, we have developed robust gradient estimators that have bounded error with high probability. Thus, the conditions of Theorem \ref{thr:aoftrl} are satisfied with probability at least $1-\delta$. Applying robust OFTRL for $T$ iterations we find a solution $(\bar{d},\bar{h})$ that satisfies the guarantee:
\begin{align*}
    \max\limits_{d \in \mathcal{D}}\mathcal{L}(d,\bar{h})- \min\limits_{h \in \mathcal{H}}\mathcal{L}(\bar{d},h)&\leq \frac{C_1( \mathcal{M}_n,h_{max})}{T}+6C_T
\end{align*}
for some constant $C_1( \mathcal{M}_n,h_{max})$ dependent on OFTRL, where $\mathcal{M}_n$ is a regularised MDP with parameters $S,A,P_n,r_n,\rho,\gamma,\lambda$, $C_T:= H 
\frac{1}{T}\sum\limits_{t=1}^T \|\hat{g_{h,t}}-g_{h,t}\|_2 +  D 
\frac{1}{T}\sum\limits_{t=1}^T \|\hat{g_{d,t}}-g_{d,t}\|_2,$ $H=\max\limits_{h\in\mathcal{H}} \|h\|_2 = \sqrt{S}h_{max}$ and $D=\max\limits_{d\in\mathcal{D}} \|d\|_2 = \frac{\sqrt{SA}}{1-\gamma}$. \\ 
At each one of the $T$ iterations  of OFTRL the gradient estimators for $d$ and $h$ are called. We set the error tolerance of the gradient estimators to $\delta/T$ and we feed each estimator with $\tilde{m}=\frac{m}{2T}$ fresh samples. From Theorem \ref{thr:robust_grad} and a union bound over $T$ iterations we have the guarantee that for all $t\in[T]$ $\|\hat{g_{h,t}}-g_{h,t}\|_2\leq E_1\left(\frac{m}{2T},{\epsilon},\frac{\delta}{T}\right)$ and $\|\hat{g_{d,t}}-g_{d,t}\|_2\leq E_2\left(\frac{m}{2T},{\epsilon},\frac{\delta}{T}\right)$ with probability at least $1-\delta$.\\
Thus we have $C_T\leq C(\delta) := \sqrt{S}\left(E_1\left(\frac{m}{2T},{\epsilon},\frac{\delta}{T}\right) h_{max} +  \frac{\sqrt{A}E_2\left(\frac{m}{2T},{\epsilon},\frac{\delta}{T}\right)}{1-\gamma}\right)$ with probability at least $1-\delta$. Thus for the duality gap we have \begin{align*}
    \max\limits_{d \in \mathcal{D}}\mathcal{L}(d,\bar{h})- \min\limits_{h \in \mathcal{H}}\mathcal{L}(\bar{d},h)&\leq \frac{C_1( \mathcal{M}_n,h_{max})}{T}+6C(\delta) 
\end{align*}
with probability at least $1-\delta$.\\
The duality gap upper bound has one term on the left that decreases with $T$ and the term $6C(\delta)$ on the right that increases with $T$. After enough iterations the term on the left will be less than $C(\delta)$ and the duality gap will be less than $7C(\delta)$, with probability at least $1-\delta$.\end{proof}
\subsection{Proof of Theorem \ref{thm.pl_robust_gda}}
\textbf{Statement.} Consider the robust OFTRL from Lemma \ref{lm.approx.bound}, and assume that its number of iterations $T$ is s.t. \eqref{eq.gda_duality_gap_bound_appdx} holds. Then the output of robust OFTRL $\bar d$ satisfies $\|d^*_n-\bar{d}\|_2 \leq \alpha(M_n,\delta)$ with probability at least $1-\delta$, where 
\begin{align*}
        \alpha(M_n,\delta) = \sqrt{\frac{14C(\delta)}{\lambda}}+ C_n' +\sqrt{ 2C_n' \left(\frac{\|r_n\|}{\lambda}+\frac{1}{1-\gamma}\right)}
\end{align*}
and $C_n'=\frac{ \|r_n\|_2\sqrt{S}\sigma_n^{-1/2}}{(1-\gamma) h_{max}}$. \\\\
To prove the above statement we first need to define some quantities.
A function $d$ is a valid occupancy measure in $M_n$ if $d\geq0$ pointwise and for all $s \in \mathcal{S}$ it holds that:
\[\sum\limits_{a}d(s,a) = \rho(s) +\gamma \sum\limits_{s',a} d(s',a) P_n(s',a,s)\]
We define the constraint violation functions $C_d(s)$ as follows:
\begin{align}
    C_d(s)=\rho(s) +\gamma \sum\limits_{s',a} d(s',a) P_n(s',a,s)-\sum\limits_{a}d(s,a)
    \label{constraint}
\end{align}
The Lagrangian can then be written as follows:
\[\mathcal{L}(d,h,M_n)=\sum\limits_{s,a} d(s,a) r_n(s,a)+\sum\limits_{s}h(s)C_d(s)- \frac{\lambda}{2} {\|d\|_2}^2\]
We consider four different quantities.
\begin{itemize}
    \item $(d^*,h^*)$ is an optimal solution to problem \ref{eq.perf_rl.langrangian}
    \item $(d^*_{\mathcal{L}},h^*_{\mathcal{L}})$ is an optimal solution to problem \ref{eq.perf_rl.langrangian} with bounded domain $\mathcal D = \{ d : 0 \le d(s, a) \le \frac{1}{1-\gamma}\}$ and $\mathcal H = \{ d : -h_{max} \le h(s, a) \le h_{max} \}$
    \item $\hat{d}$ is the projection of $d^*_{\mathcal{L}}$ on the set $\{d: C_d(s)\ \forall s \in \mathcal{S}\}$ of occupancy measures satisfying flow constraints, where $C_d(s)$ is defined in \ref{constraint}.
    \item $(\bar{d}$, $\bar{h})$ is the solution obtained by applying robust OFTRL on problem \ref{eq.perf_rl.langrangian} with bounded domain $\mathcal D$ and $\mathcal H$.
\end{itemize}
The first three quantities are theoretical, while the fourth one is the output of the proposed algorithm. We are interested in bounding the distance between the later and $d^*$, which is an optimal solution to the reguralised MDP. To achieve this we will first bound the distance between $\bar{d}$ and $d^*_{\mathcal{L}}$, then the distance between $d^*_{\mathcal{L}}$ and $\hat{d}$ and finally the distance between $\hat{d}$ and 
$d^*$.\\
\subsubsection{Bounding the Distance Between $\bar{d}$ and $d^*_{\mathcal{L}}$}
\noindent As we saw in Lemma \ref{lm.approx.bound}, after running enough iterations of robust OFTRL we have $\max\limits_{d \in \mathcal{D}}\mathcal{L}(d,\bar{h})- \min\limits_{h \in \mathcal{H}}\mathcal{L}(\bar{d},h)<  7C(\delta)$, with probability at least $1-\delta$, where  $C(\delta) := \sqrt{S}\left(E_1\left(\frac{m}{2T},{\epsilon},\frac{\delta}{T}\right) h_{max} +  \frac{\sqrt{A}E_2\left(\frac{m}{2T},{\epsilon},\frac{\delta}{T}\right)}{1-\gamma}\right)$

The objective $L(\cdot, h^*_{\mathcal{L}},M_n)$ is $\lambda$-strongly concave. Therefore, we have
\begin{align*}
    &\mathcal{L}\left(\bar{d}, h^*_{\mathcal{L}},M_n\right) - \mathcal{L}(d^*_{\mathcal{L}}, h^*_{\mathcal{L}},M_n) \leq  \nabla_d \mathcal{L}(d^*_{\mathcal{L}}, h^*_{\mathcal{L}},M_n)(\bar{d}-d^*_{\mathcal{L}}) -\frac{\lambda}{2}\left\|d^*_{\mathcal{L}}-\bar{d}\right\|^2 \Rightarrow \\
    &\mathcal{L}\left(\bar{d}, h^*_{\mathcal{L}},M_n\right) - \mathcal{L}(d^*_{\mathcal{L}}, h^*_{\mathcal{L}},M_n) \leq  -\frac{\lambda}{2}\left\|d^*_{\mathcal{L}}-\bar{d}\right\|^2 \Rightarrow \\
    &\left\|d^*_{\mathcal{L}}-\bar{d}\right\|\leq \sqrt{\frac{2\left(\mathcal{L}\left(d^*_{\mathcal{L}}, h^*_{\mathcal{L}},M_n\right) - \mathcal{L}(\bar{d}, h^*_{\mathcal{L}},M_n)\right)}{\lambda}}
\end{align*}
In the above steps the second inequality is obtain from the first using the first order optimality condition $\nabla_d \mathcal{L}(d^*_{\mathcal{L}}, h^*_{\mathcal{L}},M_n)( \bar{d}-d^*_{\mathcal{L}})\leq 0)$.
For the term inside the square root we have:
\begin{align*}
    &\mathcal{L}\left(d^*_{\mathcal{L}}, h^*_{\mathcal{L}},M_n\right) - \mathcal{L}(\bar{d}, h^*_{\mathcal{L}},M_n)=\mathcal{L}\left(d^*_{\mathcal{L}}, h^*_{\mathcal{L}},M_n\right)-\mathcal{L}(d^*_{\mathcal{L}}, \bar{h},M_n)+
    \mathcal{L}(d^*_{\mathcal{L}}, \bar{h},M_n) - \mathcal{L}(\bar{d}, h^*_{\mathcal{L}},M_n)\leq 7C(\delta)
\end{align*}
because $ \mathcal{L}\left(d^*_{\mathcal{L}},\bar{h} ,M_n \right)-\mathcal{L}\left(\bar{d},h^*_{\mathcal{L}}  ,M_n \right) \leq 7C(\delta)$ \\ 
and $\mathcal{L}\left(d^*_{\mathcal{L}}, h^*_{\mathcal{L}},M_n\right)-\mathcal{L}(d^*_{\mathcal{L}}, \bar{h},M_n)=\min_{h\in\mathcal{H}}\mathcal{L}\left(d^*_{\mathcal{L}}, h,M_n\right)-\mathcal{L}(d^*_{\mathcal{L}}, \bar{h},M_n)\leq0$\\
Thus we have  $\|\bar{d}-d^*_{\mathcal{L}}\|\leq \sqrt{\frac{14C(\delta)}{\lambda}}$, with probability at least $1-\delta$.
\\
\subsubsection{Bounding the Distance Between $d^*_{\mathcal{L}}$ and $\hat{d}$}
\noindent\\ 
First we prove the following Lemma, which bounds the constraint violation and the objective suboptimality of approximate solution $d^*_{\mathcal{L}}$.

\begin{lemma}\label{lemma:two}
    Let $\mathcal{H}=\{h: -h_{max} \leq h(s)\leq h_{max}\ \forall s \in \mathcal{S}\}$. Then, $d^*_{\mathcal{L}}$ satisfies :
    \begin{itemize}
     \item $|C_{d^*_{\mathcal{L}}}(s)|\leq \left(\frac{ \|r_n\|}{1-\gamma}\right)/h_{max}:= c \quad \forall s \in \mathcal{S}$ (approximate constraint satisfaction)
     \item $d^*_{\mathcal{L}}\cdot r_n - \frac{\lambda}{2}\|d^*_{\mathcal{L}}\|_2^2\geq d^*\cdot r_n- \frac{\lambda}{2}\|d^*\|_2^2$ (approximate objective optimality)
   \end{itemize}
    where $d^*$ is the optimal feasible occupancy measure (the solution of problem \ref{eq.perf_rl.langrangian}).
\end{lemma}
\textit{Proof:} 
\begin{itemize}
\item \textit{Approximate constraint satisfaction:}
Since $(d^*_{\mathcal{L}},h^*_{\mathcal{L}})$ is an optimal solution to problem \ref{eq.perf_rl.langrangian} with bounded domain $\mathcal D$ and $\mathcal H$,
we have $\max\limits_{d \in \mathcal{D}}\mathcal{L}(d,h^*_{\mathcal{L}},M_n)- \min\limits_{h \in \mathcal{H}}\mathcal{L}(d^*_{\mathcal{L}},h,M_n)=0$. If for some $s' \in \mathcal{S}\quad |C_{d^*_{\mathcal{L}}}(s')|>c$, then $\min\limits_{h \in \mathcal{H}}\mathcal{L}(d^*_{\mathcal{L}},h,M_n)< 0$. The zero value is achieved by setting $h(s')=h_{max} \cdot sign(C_{d^*_{\mathcal{L}}}(s'))$ and $h(s)=0$ for $s\neq s'$. Then,  $\sum\limits_{s,a} d^*_{\mathcal{L}}(s,a) r_n(s,a)-\frac{\lambda}{2}d^*_{\mathcal{L}}(s,a)^2\leq \frac{ \|r_n\|}{1-\gamma}$ and $\sum\limits_{s}h(s)C_{d^*_{\mathcal{L}}}(s)< -h_{max}\cdot c$, so $\mathcal{L}(d^*_{\mathcal{L}},h)<-h_{max}\cdot c+\frac{ \|r_n\|}{1-\gamma}=0$. \\
Moreover, $\max\limits_{d \in \mathcal{D}}\mathcal{L}(d,h^*_{\mathcal{L}},M_n)\geq 0$, which can be achieved by selecting $d$ such that $C_d(s)=0\quad \forall s \in \mathcal{S}$ that is by selecting some valid occupancy measure.\\ \\
Combining the above we get
$\max\limits_{d \in \mathcal{D}}\mathcal{L}(d,h^*_{\mathcal{L}},M_n)- \min\limits_{h \in \mathcal{H}}\mathcal{L}(d^*_{\mathcal{L}},h,M_n)> 0$ which contradicts the optimality of $(d^*_{\mathcal{L}},h^*_{\mathcal{L}})$. \\ \\ 
\item \textit{Approximate objective optimality:} Since $(d^*_{\mathcal{L}},h^*_{\mathcal{L}})$ is an optimal solution to problem \ref{eq.perf_rl.langrangian} with bounded domain $\mathcal D$ and $\mathcal H$,
we have: \[ \mathcal{L}(d,h^*_{\mathcal{L}},M_n)- \min\limits_{h \in \mathcal{H}}\mathcal{L}(d^*_{\mathcal{L}},h,M_n)\leq 0 \quad \forall d \in \mathcal{D}\]
We set $d=d^*$, where $d^*$ is the optimal feasible occupancy measure and we yield:
\[ \sum\limits_{s,a} d^*(s,a) r_n(s,a)-\frac{\lambda}{2}d^*(s,a)^2+\sum\limits_{s}h^*_{\mathcal{L}}(s)C_{d^*}(s)- \min\limits_{h \in \mathcal{H}}\mathcal{L}(d^*_{\mathcal{L}},h,M_n)\leq 0\]
Since $d^*$ is a valid occupancy measure we have $C_{d^*}(s)=0\ \forall s \in \mathcal{S}$. Thus we obtain:
\begin{align*}
    & d^*\cdot r_n - \frac{\lambda}{2}\|d^*\|_2^2 \leq \min\limits_{h \in \mathcal{H}}\mathcal{L}(d^*_{\mathcal{L}},h,M_n) \Leftrightarrow \\
 & d^*\cdot r_n  - \frac{\lambda}{2}\|d^*\|_2^2 \leq \min\limits_{h \in \mathcal{H}}\left[\sum\limits_{s,a} d^*_{\mathcal{L}}(s,a) r_n(s,a)-\frac{\lambda}{2}d^*_{\mathcal{L}}(s,a)^2+\sum\limits_{s}h(s)C_{d^*_{\mathcal{L}}}(s)\right] \Leftrightarrow \\
 & d^*\cdot r_n - \frac{\lambda}{2}\|d^*\|_2^2 -\min\limits_{h \in \mathcal{H}}\left[\sum\limits_{s}h(s)C_{d^*_{\mathcal{L}}}(s)\right]\leq d^*_{\mathcal{L}}\cdot r_n - \frac{\lambda}{2}\|d^*_{\mathcal{L}}\|_2^2\Rightarrow \\
 & d^*\cdot r_n- \frac{\lambda}{2}\|d^*\|_2^2 \leq d^*_{\mathcal{L}}\cdot r_n - \frac{\lambda}{2}\|d^*_{\mathcal{L}}\|_2^2
\end{align*}
The last step follows from the fact that:
\[\min\limits_{h \in \mathcal{H}}\left[\sum\limits_{s}h(s)C_{d^*_{\mathcal{L}}}(s)\right]\leq\left[\sum\limits_{s}h(s)C_{d^*_{\mathcal{L}}}(s)\right]_{h=0}=0  \Rightarrow -\min\limits_{h \in \mathcal{H}}\left[\sum\limits_{s}h(s)C_{d^*_{\mathcal{L}}}(s)\right] \geq 0\] 
\end{itemize} 
Next we write a lemma that appears in \cite{garber2019logarithmic} and is a rephrased form of a result originally proved in \cite{hoffman2003approximate}.\\
\begin{lemma}[ Hoffman’s bound]\label{lemma:three}
Let $\mathcal{P}:=\left\{x \in R^d \mid A x \leq \mathbf{b}\right\}$ be a compact and convex polytope and let $C \in R^{m \times d}$. Given a vector $c \in R^m$, define the set $\mathcal{P}(C, c):=\{x \in \mathcal{P} \mid C x=c\}$. If $\mathcal{P}(C, c) \neq \emptyset$, then there exists $\sigma>0$ such that $\forall x \in$ $\mathcal{P}: \operatorname{dist}(x, \mathcal{P}(C, c))^2 \leq \sigma^{-1}\|C x-c\|^2$. Moreover, we have the bound $\sigma \geq \min _{\mathbf{Q} \in \mathcal{M}} \lambda_{\min }(\mathbf{Q Q}^{\top})$, where $\mathcal{M}$ is the set of all $d \times d$ matrices whose rows are linearly independent rows of the matrix $\mathbf{M}:=\left(A^{\top}, C^{\top}\right)^{\top}$, and $\lambda_{\min }(\cdot)$ denotes the smallest non-zero eigenvalue.\\
\end{lemma}
\noindent In our setting, let $\mathcal{P}=\{d \in \mathcal{D} : C_d(s)=0 \ \forall s \in \mathcal{S} \}$ be the set of valid occupancy measures and $\sigma_n$ the Hoffman constant of $\mathcal{P}$.
Due to Lemma \ref{lemma:three} we obtain for $d^*_{\mathcal{L}}$: $\|d^*_{\mathcal{L}}-\hat{d}\|^2=\operatorname{dist}(d^*_{\mathcal{L}}, \mathcal{P})^2 \leq \sigma_n^{-1}\|C_{d^*_{\mathcal{L}}}\|^2\leq \sigma_n^{-1} S \left(\frac{ \|r_n\|}{(1-\gamma)h_{max}}\right)^2$, because from Lemma \ref{lemma:two} we have that $|C_{d^*_{\mathcal{L}}}(s)|\leq \frac{ \|r_n\|}{(1-\gamma)h_{max}} \quad \forall s \in \mathcal{S}$. \\

\subsubsection{Bounding the Distance Between $\hat{d}$ and $d^*$}
% a^2 = (a-b+b)^2 = (a-b)^2 + 2b(a-b) + b^2
\noindent\\ 
From Lemma \ref{lemma:two} we have: 
\begin{align*}
    & d^*\cdot r_n- \frac{\lambda}{2}\|d^*\|^2 \leq d^*_{\mathcal{L}}\cdot r_n - \frac{\lambda}{2}\|d^*_{\mathcal{L}}\|^2\Leftrightarrow \\
    & d^*\cdot r_n - \frac{\lambda}{2}\|d^*\|^2 \leq \hat{d}\cdot r_n +(d^*_{\mathcal{L}}-\hat{d})\cdot r_n- \frac{\lambda}{2}\|\hat{d}\|^2- \frac{\lambda}{2}\|d^*_{\mathcal{L}}-\hat{d}\|^2-\lambda (d^*_{\mathcal{L}}-\hat{d})\cdot\hat{d} \Rightarrow \\
    & d^*\cdot r_n - \frac{\lambda}{2}\|d^*\|^2-\|d^*_{\mathcal{L}}-\hat{d}\|\cdot (\|r_n\|+\lambda\|\hat{d}\|) \leq \hat{d}\cdot r_n - \frac{\lambda}{2}\|\hat{d}\|^2\Rightarrow \\
    & d^*\cdot r_n- \frac{\lambda}{2}\|d^*\|^2 -\frac{ \|r_n\|\sqrt{S}\sigma_n^{-1/2}}{(1-\gamma)h_{max}} \left(\|r_n\|+\frac{\lambda}{1-\gamma}\right) \leq \hat{d}\cdot r_n- \frac{\lambda}{2}\|\hat{d}\|^2
\end{align*}
Moreover, since $d^*$ is the optimal among all valid occupancy measures and $\hat{d}$ is a valid occupancy measure, we have: 
\[\hat{d}\cdot r_n- \frac{\lambda}{2}\|\hat{d}\|^2\leq d^*\cdot r_n- \frac{\lambda}{2}\|d^*\|^2\]
% We derive policy $\hat{\pi}$ from $\bar{d}$ by normalising vector $\bar{d}(s,a)$ over each state $s$. In particular we have $$\hat{\pi}
Thus we have:
\begin{align}
    &\left| \left(d^*\cdot r_n- \frac{\lambda}{2}\|d^*\|^2\right) -\left(\hat{d}\cdot r_n- \frac{\lambda}{2}\|\hat{d}\|^2\right)   \right| \leq \frac{ \|r_n\|\sqrt{S}\sigma_n^{-1/2}}{(1-\gamma)h_{max}} \left(\|r_n\|+\frac{\lambda}{1-\gamma}\right) 
    \label{fdiff}
\end{align}
Function $f(d)=d\cdot r_n- \frac{\lambda}{2}\|d\|^2$ is strongly concave thus we have:
\[f(\hat{d}) - f(d^*) \leq  \nabla f(d^*)(\hat{d}-d^*) -\frac{\lambda}{2}\left\|d^*-\hat{d}\right\|^2 
\]
From the first order optimality condition of $d^*$ we have $\nabla f(d^*)(\hat{d}-d^*)\leq 0$ and thus we obtain:
\begin{align*}
&f(\hat{d}) - f(d^*) \leq  -\frac{\lambda}{2}\left\|d^*-\hat{d}\right\|^2 \Leftrightarrow \\ 
&\frac{\lambda}{2}\left\|d^*-\hat{d}\right\|^2 \leq f(d^*)- f(\hat{d}) \Leftrightarrow \\ 
&\left\|d^*-\hat{d}\right\| \leq \sqrt{\frac{2(f(d^*)- f(\hat{d}))}{\lambda}} \Rightarrow \\
&\left\|d^*-\hat{d}\right\| \leq \sqrt{\frac{ 2\|r_n\|\sqrt{S}\sigma_n^{-1/2}}{(1-\gamma)h_{max}} \left(\frac{\|r_n\|}{\lambda}+\frac{1}{1-\gamma}\right)} 
\end{align*}
In the last step we used inequality \ref{fdiff}.\\ \\
At this point we are ready to bound the distance between the calculated approximate occupancy measure $\bar{d}$ and the true optimal $d^*$. From the triangle inequality we have:
\begin{align*}
    \left\|d^*-\bar{d}\right\| &\leq \left\|\bar{d}-d^*_{\mathcal{L}}\right\|+\left\|d^*_{\mathcal{L}}-\hat{d}\right\|+\left\|\hat{d}-d^*\right\| \\
    & \leq \sqrt{\frac{14C(\delta)}{\lambda}}+ \frac{ \|r_n\|\sqrt{S}\sigma_n^{-1/2}}{(1-\gamma)h_{max}}+\sqrt{\frac{ 2\|r_n\|\sqrt{S}\sigma_n^{-1/2}}{(1-\gamma)h_{max}} \left(\frac{\|r_n\|}{\lambda}+\frac{1}{1-\gamma}\right)} := \alpha
\end{align*}
This holds with probability at least $1-\delta$, since the bound for $\left\|\bar{d}-d^*_{\mathcal{L}}\right\|$ holds with probability at least $1-\delta$ and the other bounds are deterministic.

\subsection{Proof of Theorem \ref{thr:perf_rl}}
\textbf{Statement.} Under Assumption \ref{assume.sensitivity} and Assumption \ref{assume.hoffman}, there exist $\lambda$ and $N$ such that the output of %robust repeated optimization 
Algorithm \ref{alg:rob_rr_prl}
satisfies $\tilde{d}_N \in \{ d \in \mathcal D :\|d-d_S\|_2 \leq \tilde{C}:=4\cdot \bar{C}(\delta/N) \}$ with probability at least $1-\delta$, where $d_S$ is a performatively stable policy and $\bar{C}(\delta):=\bar \alpha(\delta) + c\sqrt{SA}$.\\

\begin{proof}
We set the error probability tolerance of the robust gradient estimators to $\frac{\delta}{NT}$. Then, from Theorem \ref{thm.pl_robust_gda} we have $\|d^*_n-\bar{d}\|_2 \leq \alpha(M_n,\delta/N)$, with probability at least $1-\delta/N$ for the $n$-th  iteration of repeated retraining. We apply a union bound over all $N$ iterations and we conclude that $\|d^*_n-\bar{d}\|_2 \leq \alpha(M_n,\delta/N)$ in all $N$ iterations with probability at least $1-\delta$, where $\bar{d}$ denotes the output of robust OFTRL at the $n$-th  iteration of repeated retraining. Due to Assumption \ref{assume.hoffman} we have $\alpha(M_n,\delta/N)\leq \bar \alpha(\delta/N)$.\\\\
After adding positive constant $c$ to $\bar d$ we get a solution $\tilde{d}_n$ that satisfies $\|\tilde{d}_n - d^*_n \|\leq \|\tilde{d} - d^*_n \| + \|\tilde{d}_n - \bar d\| \leq \bar \alpha(\delta/N) + c\sqrt{SA}$.\\\\
We consider the approximate policy optimization operator $PO(\cdot)$. Let $PO^C(\bar{d}_n)$ denote the estimation of the approximate occupancy measure in the environment $M_n$ induced by $\bar{d}_n$. This procedure includes the calculation of $\bar{d}_{n+1}$ by robust OFTRL, and the addition of positive constant $c$ to $\bar{d}_{n+1}$ (see Algorithm \ref{alg:rob_rr_prl}). Moreover, let $PO(\bar{d}_n)$ denote the estimation of the optimal occupancy measure in the environment $M_n$ induced by $\bar{d}_n$ (the solution to problem \ref{eq.perf_rl.langrangian}). For a performatively stable policy $d_S$ it holds that $PO(d_S)=d_S$. Moreover, $PO(\cdot)$ is a contraction mapping for some contraction constant $\beta$ (see the proof of Theorem 1 in \citep{mandal2023performative}).

\noindent At $n$-th timestep of repeated retraining we have:
\begin{align*}
    &\|PO^C(\bar{d}_n)-d_S\|\leq \|PO^C(\bar{d}_n)-PO(\bar{d}_n)\|+ \|PO(\bar{d}_n)-d_S\| \Rightarrow \\
    &\|PO^C(\bar{d}_n)-d_S\|\leq \bar{C}(\delta/N) +\beta\|\bar{d}_n-d_S\|
\end{align*}
The update rule of the repeated optimisation algorithm is $\bar{d}_{n+1}=PO^C(\bar{d}_n)$. Thus we have
\[\|\bar{d}_{n+1}-d_S\|\leq \bar{C}(\delta/N) +\beta\|\bar{d}_n-d_S\|\]
For $\|\bar{d}_n-d_S\|\geq \frac{\bar{C}(\delta/N)}{\beta}$ we have
\[\|\bar{d}_{n+1}-d_S\|\leq \bar{C}(\delta/N) +\beta\|\bar{d}_n-d_S\|\leq 2\beta\|\bar{d}_n-d_S\|\]
Using a large enough regularisation constant $\lambda$ we can make $\beta$ less than $\frac{1}{2}$, thus $PO^C(\cdot)$ is a contraction mapping for $\|\bar{d}_n-d_S\|\geq \frac{\bar{C}(\delta/N)}{\beta}$.\\
For $\|\bar{d}_n-d_S\|< \frac{\bar{C}(\delta/N)}{\beta}$ we have
$\|\bar{d}_{n+1}-d_S\|\leq 2\bar{C}(\delta/N)<\frac{\bar{C}(\delta/N)}{\beta}$. Thus, after some iterations (suppose this is less than $N$, without loss of generality, since we did not assume anything for $N$) all iterates $\bar{d}_n$ will satisfy $\|\bar{d}_n-d_S\|\leq \frac{\bar{C}(\delta/N)}{\beta}$
, which means that after $N$ iterations repeated application of $PO^C(\cdot)$ gives an approximate performatively stable occupancy measure $\hat{d}_{S}$, such that $\|\hat{d}_S-d_S\|< \frac{\bar{C}(\delta/N)}{\beta}$. For $\beta=\frac{1}{4}$ we obtain $\|\hat{d}_S-d_S\|< 4\bar{C}(\delta/N)$.
\end{proof}
\subsection{Return Suboptimality of the Approximately Stable Policy}
We are interested in comparing the Return of the approximately stable policy parameterised by $\hat{d}_{S}$ in the MDP induced by $\hat{d}_{S}$ with the Return of performatively stable occupancy measure ${d}_{S}$. To do this we will first bound the distance between the approximately stable policy and the performatively stable policy. Then we will bound the distance between the occupancy measures induced by these policies in their corresponding MDPs. Finally, we will use this bound to bound the difference of Returns, since the Return in an MDP is equal to the inner product between the occupancy measure and the reward function.\\\\
\noindent Since $\hat{d}_{S}$ might violate some equality constraints it is not necessarily a valid occupancy measure. To  overcome this issue, we first turn $\hat{d}_{S}$ into a policy and then deploy this policy in the environment to get a valid occupancy measure. In particular, we deploy policy $\hat{\pi}_{S}(a|s)=\frac{\hat{d}_{S}(s,a)}{\sum_{a' \in \mathcal{A}}\hat{d}_{S}(s,a')}$. The performatively stable policy $\pi_{S}$ can be written as $\pi_{S}(a|s)=\frac{d_{S}(s,a)}{\sum_{a' \in \mathcal{A}}d_{S}(s,a')}$.  We will try to bound the distance between  $\pi_{S}$ and $\hat{\pi}_{S}$.

% \cite{puterman2014markov}

\begin{align*}
    \|\pi_{S}-\hat{\pi}_{S}\|_2^2&=\sum_{s \in \mathcal{S}}\sum_{a \in \mathcal{A}}(\pi_{S}(a|s)-\hat{\pi}_{S}(a|s))^2 \\ 
    &=\sum_{s \in \mathcal{S}}\sum_{a \in \mathcal{A}}\left(\frac{d_{S}(s,a)}{\sum_{a' \in \mathcal{A}}d_{S}(s,a')}-\frac{\hat{d}_{S}(s,a)}{\sum_{a' \in \mathcal{A}}\hat{d}_{S}(s,a')}\right)^2 \\ 
    &=\sum_{s \in \mathcal{S}}\sum_{a \in \mathcal{A}}\left(\frac{d_{S}(s,a)}{\sum_{a' \in \mathcal{A}}d_{S}(s,a')}-\frac{\hat{d}_{S}(s,a)}{\sum_{a' \in \mathcal{A}}d_{S}(s,a')}+\frac{\hat{d}_{S}(s,a)}{\sum_{a' \in \mathcal{A}}d_{S}(s,a')}-\frac{\hat{d}_{S}(s,a)}{\sum_{a' \in \mathcal{A}}\hat{d}_{S}(s,a')}\right)^2 \\ 
    &\leq2\sum_{s \in \mathcal{S}}\sum_{a \in \mathcal{A}}\left(\frac{d_{S}(s,a)}{\sum_{a' \in \mathcal{A}}d_{S}(s,a')}-\frac{\hat{d}_{S}(s,a)}{\sum_{a' \in \mathcal{A}}d_{S}(s,a')}\right)^2+2\sum_{s \in \mathcal{S}}\sum_{a \in \mathcal{A}}\left(\frac{\hat{d}_{S}(s,a)}{\sum_{a' \in \mathcal{A}}d_{S}(s,a')}-\frac{\hat{d}_{S}(s,a)}{\sum_{a' \in \mathcal{A}}\hat{d}_{S}(s,a')}\right)^2 \\ 
    &=\underbrace{2\sum_{s \in \mathcal{S}}\frac{1}{\left(\sum_{a' \in \mathcal{A}}d_{S}(s,a')\right)^2}\sum_{a \in \mathcal{A}}\left(d_{S}(s,a)-\hat{d}_{S}(s,a)\right)^2}_{A_1}+ \underbrace{2\sum_{s \in \mathcal{S}}\sum_{a \in \mathcal{A}}\left(\hat{d}_{S}(s,a)\right)^2\left(\frac{1}{\sum_{a' \in \mathcal{A}}d_{S}(s,a')}-\frac{1}{\sum_{a' \in \mathcal{A}}\hat{d}_{S}(s,a')}\right)^2}_{A_2} \\ 
\end{align*}
\begin{align*}
    A_2&=2\sum_{s \in \mathcal{S}}\sum_{a \in \mathcal{A}}\left(\hat{d}_{S}(s,a)\right)^2\left(\frac{1}{\sum_{a' \in \mathcal{A}}d_{S}(s,a')}-\frac{1}{\sum_{a' \in \mathcal{A}}\hat{d}_{S}(s,a')}\right)^2\\
    &=2\sum_{s \in \mathcal{S}}\sum_{a \in \mathcal{A}}\left(\frac{\hat{d}_{S}(s,a)}{\sum_{a' \in \mathcal{A}}d_{S}(s,a')\sum_{a' \in \mathcal{A}}\hat{d}_{S}(s,a')}\right)^2\left(\sum_{a' \in \mathcal{A}}\hat{d}_{S}(s,a')-d_{S}(s,a')\right)^2 \\
    &\leq 2\sum_{s \in \mathcal{S}}\sum_{a \in \mathcal{A}}\left(\frac{\hat{d}_{S}(s,a)}{\sum_{a' \in \mathcal{A}}d_{S}(s,a')\sum_{a' \in \mathcal{A}}\hat{d}_{S}(s,a')}\right)^2A\sum_{a' \in \mathcal{A}}\left(\hat{d}_{S}(s,a')-d_{S}(s,a')\right)^2 \\
    % &\leq 2\sum_{s \in \mathcal{S}}\sum_{a \in \mathcal{A}}\left(\frac{1}{\sum_{a' \in \mathcal{A}}d_{S}(s,a')}\right)^2A\sum_{a' \in \mathcal{A}}\left(\hat{d}_{S}(s,a')-d_{S}(s,a')\right)^2 \\
     &= 2\sum_{s \in \mathcal{S}}\left(\frac{1}{\sum_{a' \in \mathcal{A}}d_{S}(s,a')\sum_{a' \in \mathcal{A}}\hat{d}_{S}(s,a')}\right)^2\sum_{a \in \mathcal{A}}\left(\hat{d}_{S}(s,a)\right)^2A\sum_{a' \in \mathcal{A}}\left(\hat{d}_{S}(s,a')-d_{S}(s,a')\right)^2 \\
     &\leq 2\sum_{s \in \mathcal{S}}\left(\frac{1}{\sum_{a' \in \mathcal{A}}d_{S}(s,a')\sum_{a' \in \mathcal{A}}\hat{d}_{S}(s,a')}\right)^2\left(\sum_{a \in \mathcal{A}}\hat{d}_{S}(s,a)\right)^2A\sum_{a' \in \mathcal{A}}\left(\hat{d}_{S}(s,a')-d_{S}(s,a')\right)^2 \\
     &= 2\sum_{s \in \mathcal{S}}\left(\frac{1}{\sum_{a' \in \mathcal{A}}d_{S}(s,a')}\right)^2A\sum_{a' \in \mathcal{A}}\left(\hat{d}_{S}(s,a')-d_{S}(s,a')\right)^2 \\
     &= A\cdot A_1 \\
\end{align*}
We make the assumption that for each state $s$ the initial state distribution $\rho$ has measure $\rho(s)\geq p_1$. Then, after deploying $\pi_{S}$ for infinite time, the probability of occupation of each state $s$ is at least
$(1-\gamma)p_1$, because the probability of restarting at $s$ is at least $(1-\gamma)p_1$. Occupancy measure $d_{S}$ has sum $1/(1-\gamma)$, so it needs to be multiplied with $1-\gamma$ to become an occupancy distribution over state-actions. $(1-\gamma)\sum_{a' \in \mathcal{A}}d_{S}(s,a')$ is the probability of occupation of state $s$ after deploying $\pi_{S}$ for infinite time and as we saw right before it is at least $(1-\gamma)p_1$. Thus we conclude  that $\sum_{a' \in \mathcal{A}}d_{S}(s,a')\geq p_1$ for all $s\in\mathcal{S}$. Using this fact we can bound term $A_1$ as follows: \[A_1=2\sum_{s \in \mathcal{S}}\frac{1}{\left(\sum_{a' \in \mathcal{A}}d_{S}(s,a')\right)^2}\sum_{a \in \mathcal{A}}\left(d_{S}(s,a)-\hat{d}_{S}(s,a)\right)^2\leq 2\sum_{s \in \mathcal{S}}\frac{1}{{p_1}^2}\sum_{a \in \mathcal{A}}\left(d_{S}(s,a)-\hat{d}_{S}(s,a)\right)^2=\frac{2}{{p_1}^2}\|d_{S}-\hat{d}_{S}\|_2 ^2\]
and the policy distance as:
\[\|\pi_{S}-\hat{\pi}_{S}\|_2^2\leq (A+1)A_1 \leq \frac{2(A+1)}{{p_1}^2}\|d_{S}-\hat{d}_{S}\|_2 ^2\]
Taking the square root of both sides we obtain:
\[\|\pi_{S}-\hat{\pi}_{S}\|_2\leq (A+1)A_1 \leq \frac{2\sqrt{A+1}}{p_1}\|d_{S}-\hat{d}_{S}\|_2 \leq \frac{2\tilde{C}\sqrt{A+1}}{p_1}\]
% For the $\ell_1$ norm we have:
% \[\|\pi_{S}-\hat{\pi}_{S}\|_2\leq (A+1)A_1 \leq \frac{2\sqrt{A+1}}{p_1}\|d_{S}-\bar{d}\|_2 \leq \frac{2\alpha\sqrt{A+1}}{p_1\beta}\]
Then we follow a similar approach as Lemma 14.1 of \cite{agarwal2019reinforcement} to bound the distance between $d_{S}$ and the occupancy measure induced by $\hat{\pi}_{S}$. In our case, not only the two polices, but also the environments they are deployed on differ.\\ \\
We denote $\mathbb{P}_h^\pi$ as the state distribution resulting from $\pi$ at time step $h$ with $\rho$ as the initial state distribution. We consider bounding $\left\|\mathbb{P}_h^{\pi_{S}}-\mathbb{P}_h^{\hat{\pi}_{S}}\right\|_1$ with $h \geq 1$.
$$
\begin{aligned}
\mathbb{P}_h^{\pi_{S}}\left(s^{\prime}\right)-\mathbb{P}_h^{\hat{\pi}_{S}}\left(s^{\prime}\right) &= \sum_{s, a}\mathbb{P}_{h-1}^{\pi_{S}}(s) \pi_{S}(a \mid s)P_{d_{S}}\left(s^{\prime} \mid s, a\right)-\sum_{s, a}\mathbb{P}_{h-1}^{\hat{\pi}_{S}}(s) \hat{\pi}_{S}(a \mid s) P_{\hat{d}_{S}}\left(s^{\prime} \mid s, a\right) \\
&= \sum_{s, a}\left(\mathbb{P}_{h-1}^{\pi_{S}}(s) \pi_{S}(a \mid s)-\mathbb{P}_{h-1}^{\hat{\pi}_{S}}(s) \hat{\pi}_{S}(a \mid s)\right) P_{d_{S}}\left(s^{\prime} \mid s, a\right)  \\
&\quad +\sum_{s, a}\mathbb{P}_{h-1}^{\hat{\pi}_{S}}(s) \hat{\pi}_{S}(a \mid s) \left(P_{d_{S}}(s^{\prime} \mid s, a)-P_{\hat{d}_{S}}(s^{\prime} \mid s, a)\right) \\
&= \sum_{s, a}\left(\mathbb{P}_{h-1}^{\pi_{S}}(s) \pi_{S}(a \mid s)-\mathbb{P}_{h-1}^{\pi_{S}}(s) \hat{\pi}_{S}(a \mid s)+\mathbb{P}_{h-1}^{\pi_{S}}(s) \hat{\pi}_{S}(a \mid s)-\mathbb{P}_{h-1}^{\hat{\pi}_{S}}(s) \hat{\pi}_{S}(a \mid s)\right) P_{d_{S}}\left(s^{\prime} \mid s, a\right) \\
& \quad+ \sum_{s, a}\mathbb{P}_{h-1}^{\hat{\pi}_{S}}(s) \hat{\pi}_{S}(a \mid s) \left(P_{d_{S}}(s^{\prime} \mid s, a)-P_{\hat{d}_{S}}(s^{\prime} \mid s, a)\right) \\
&= \sum_s \mathbb{P}_{h-1}^{\pi_{S}}(s) \sum_a\left(\pi_{S}(a \mid s)-\hat{\pi}_{S}(a \mid s)\right) P_{d_{S}}\left(s^{\prime} \mid s, a\right) 
 +\sum_s\left(\mathbb{P}_{h-1}^{\pi_{S}}(s)-\mathbb{P}_{h-1}^{\hat{\pi}_{S}}(s)\right) \sum_a \hat{\pi}_{S}(a \mid s) P_{d_{S}}\left(s^{\prime} \mid s, a\right) \\
 & \quad+ \sum_{s, a}\mathbb{P}_{h-1}^{\hat{\pi}_{S}}(s) \hat{\pi}_{S}(a \mid s) \left(P_{d_{S}}(s^{\prime} \mid s, a)-P_{\hat{d}_{S}}(s^{\prime} \mid s, a)\right) \\
\end{aligned}
$$
Taking absolute value on both sides, we get:
$$
\begin{aligned}
\sum_{s^{\prime}}\left|\mathbb{P}_h^{\pi_{S}}\left(s^{\prime}\right)-\mathbb{P}_h^{\hat{\pi}_{S}}\left(s^{\prime}\right)\right| \leq & \sum_s \mathbb{P}_{h-1}^{\pi_{S}}(s) \sum_a\left|\pi_{S}(a \mid s)-\hat{\pi}_{S}(a \mid s)\right| \sum_{s^{\prime}} P_{d_{S}}\left(s^{\prime} \mid s, a\right) \\
& +\sum_s\left|\mathbb{P}_{h-1}^{\pi_{S}}(s)-\mathbb{P}_{h-1}^{\hat{\pi}_{S}}(s)\right| \sum_{s^{\prime}} \sum_a \hat{\pi}_{S}(a \mid s) P_{d_{S}}\left(s^{\prime} \mid s, a\right) \\
&+ \sum_{s, a}\mathbb{P}_{h-1}^{\hat{\pi}_{S}}(s) \hat{\pi}_{S}(a \mid s) \left|P_{d_{S}}(s^{\prime} \mid s, a)-P_{\hat{d}_{S}}(s^{\prime} \mid s, a)\right| \\
\leq & \frac{2\tilde{C}(A+1)}{p_1}+\frac{2\tilde{C}{\tilde \epsilon_p}(A+1)S}{p_1}+\left\|\mathbb{P}_{h-1}^{\pi_{S}}-\mathbb{P}_{h-1}^{\hat{\pi}_{S}}\right\|_1\\
\leq & \frac{4\tilde{C}({\tilde \epsilon_p} S+1)(A+1)}{p_1}+\left\|\mathbb{P}_{h-2}^{\pi_{S}}-\mathbb{P}_{h-2}^{\hat{\pi}_{S}}\right\|_1\leq\cdots\leq\frac{2h\tilde{C}({\tilde \epsilon_p} S+1)(A+1)}{p_1} .
\end{aligned}
$$
where in the second step we used the fact that:
\[\sum_a\left|\pi_{S}(a \mid s)-\hat{\pi}_{S}(a \mid s)\right|=\|\pi_{S}(\cdot \mid s)-\hat{\pi}_{S}(\cdot \mid s)\|_1\leq \sqrt{A} \|\pi_{S}(\cdot \mid s)-\hat{\pi}_{S}(\cdot \mid s)\|_2\leq \sqrt{A} \|\pi_{S}-\hat{\pi}_{S}\|_2\leq \frac{2\tilde{C}(A+1)}{p_1}\] and 
\begin{align*}
    \sum_{s, a}\mathbb{P}_{h-1}^{\hat{\pi}_{S}}(s) \hat{\pi}_{S}(a \mid s) \left|P_{d_{S}}(s^{\prime} \mid s, a)-P_{\hat{d}_{S}}(s^{\prime} \mid s, a)\right| &\leq \sum_{s, a}\mathbb{P}_{h-1}^{\hat{\pi}_{S}}(s) \hat{\pi}_{S}(a \mid s) \sum_{s, a}\left|P_{d_{S}}(s^{\prime} \mid s, a)-P_{\hat{d}_{S}}(s^{\prime} \mid s, a)\right|\\
    &\leq \left\|P_{d_{S}}-P_{\hat{d}_{S}}\right\|_1\leq S\sqrt{A}\left\|P_{d_{S}}-P_{\hat{d}_{S}}\right\|_2 \leq S\sqrt{A}{\tilde \epsilon_p}\left\|d_{S}-\hat{d}_{S}\right\|_2\\
    &\leq \frac{2\tilde{C}{\tilde \epsilon_p}(A+1)S}{p_1}
\end{align*}
Now using the definition of occupancy measure $d_\rho^\pi$, we have:
$$
d_\rho^{\pi_{S}}-d_\rho^{\hat{\pi}_{S}}= \sum_{h=0}^{\infty} \gamma^h\left(\mathbb{P}_h^{\pi_{S}}-\mathbb{P}_h^{\hat{\pi}_{S}}\right) .
$$
We take the $\ell_1$ norm of both sides and we obtain:
$$
\left\|d_\rho^{\pi_{S}}-d_\rho^{\hat{\pi}_{S}}\right\|_1 \leq\sum_{h=0}^{\infty} \gamma^h \frac{2h\tilde{C}({\tilde \epsilon_p} S+1)(A+1)}{p_1}
$$
Using the identity $\sum_{h=0}^{\infty} \gamma^h h=\frac{\gamma}{(1-\gamma)^2}$ we conclude that:
$$
\left\|d_\rho^{\pi_{S}}-d_\rho^{\hat{\pi}_{S}}\right\|_1 \leq \frac{2\tilde{C}({\tilde \epsilon_p} S+1)(A+1)\gamma}{p_1(1-\gamma)^2}
$$
The difference between the return of the estimated policy $\hat{\pi}_{S}$ and the return of the performatively stable policy $\pi_{S}$ is bounded as follows:
\begin{align*}
    r_{\hat{\pi}_{S}}\cdot d_\rho^{\hat{\pi}_{S}} - r_{{\pi_{S}}}\cdot d_\rho^{\pi_{S}} &= r_{\hat{\pi}_{S}}\cdot d_\rho^{\hat{\pi}_{S}} -r_{\hat{\pi}_{S}}\cdot d_\rho^{\pi_{S}}+r_{\hat{\pi}_{S}}\cdot d_\rho^{\pi_{S}}- r_{{\pi_{S}}}\cdot d_\rho^{\pi_{S}} \\
    & = r_{\hat{\pi}_{S}}\cdot (d_\rho^{\hat{\pi}_{S}} -d_\rho^{\pi_{S}})+(r_{\hat{\pi}_{S}}- r_{{\pi_{S}}})\cdot d_\rho^{\pi_{S}} \\
    & \geq -R \left\|d_\rho^{\pi_{S}}-d_\rho^{\hat{\pi}_{S}}\right\|_1-\|r_{\hat{\pi}_{S}}- r_{{\pi_{S}}}\|_2 \|d_\rho^{\pi_{S}}\|_2 \\ 
    & \geq -R \left\|d_\rho^{\pi_{S}}-d_\rho^{\hat{\pi}_{S}}\right\|_1-{\tilde \epsilon_r}\left\|d_\rho^{\pi_{S}}-d_\rho^{\hat{\pi}_{S}}\right\|_2 \|d_\rho^{\pi_{S}}\|_2 \\
    & \geq -\left\|d_\rho^{\pi_{S}}-d_\rho^{\hat{\pi}_{S}}\right\|_1\left(R+\frac{{\tilde \epsilon_r}}{1-\gamma}\right) \\
    & \geq -\frac{2\gamma({\tilde \epsilon_p} S+1)(A+1)((1-\gamma)R+{\tilde \epsilon_r})\tilde{C}}{p_1(1-\gamma)^3}
\end{align*}

\subsection{Implementation Details for the Experimental Analysis}
In the simulations of Robust Performative RL we tuned the parameters $c_p$ and $\lambda$ to make repeated retraining converge in a number of iterations close to $25$, with the logic of avoiding too sharp and too slow convergence. We found that $c_p=1$ and $\lambda=0.001$ achieved this purpose. The Robust Performative RL experiments were conducted in a node with 200 CPUs and 20GB RAM and took approximately 24 hours to complete. 
% \end{document}